\documentclass[twocolumn]{article}
\onecolumn

\PassOptionsToPackage{numbers, sort,compress}{natbib}
\usepackage[final]{neurips_2022}

\usepackage[utf8]{inputenc} % allow utf-8 input
\usepackage[T1]{fontenc}    % use 8-bit T1 fonts
\usepackage{hyperref}       % hyperlinks
\usepackage{url}            % simple URL typesetting
\usepackage{booktabs}       % professional-quality tables
\usepackage{amsfonts}       % blackboard math symbols
\usepackage{nicefrac}       % compact symbols for 1/2, etc.
\usepackage{microtype}      % microtypography
\usepackage{xcolor}         % colors
\usepackage{svg}
\usepackage{xspace}
\usepackage[export]{adjustbox}
\usepackage{printlen}
\uselengthunit{cm}
\usepackage{comment}
\usepackage{amsmath}
\usepackage{amsthm}
\usepackage{graphicx}
\usepackage{rotating}
\usepackage{float}
\usepackage{mathrsfs}
\usepackage{amssymb}
\usepackage{autobreak}
\usepackage{mathtools}
\usepackage{wrapfig}
\usepackage{bbm}
\usepackage{algorithm,algcompatible}
\usepackage[permil]{overpic}
\usepackage[noend]{algpseudocode}

\usepackage{subcaption}
\usepackage{lipsum}
\usepackage{outlines}
\usepackage[capitalise]{cleveref}
\usepackage{tikz}
\usetikzlibrary{shapes,backgrounds,patterns}
\usepackage[absolute]{textpos}
\usepackage{thmtools} 
\usepackage{thm-restate}

\usepackage[toc,page,header]{appendix}
\usepackage{minitoc}
\usepackage{longtable,xtab,booktabs}
\newsavebox{\algleft}
\newsavebox{\algright}

\algdef{SE}[DOWHILE]{Do}{doWhile}{\algorithmicdo}[1]{\algorithmicwhile\ #1}%

\DeclareMathOperator*{\argmax}{arg\,max}

\newcommand\numberthis{\addtocounter{equation}{1}\tag{\theequation}}

\newcommand{\GP}{\textsc{\small{GP}}\xspace}
\newcommand{\gpucb}{\textsc{\small{GP-UCB}}\xspace}
\newcommand{\goose}{\textsc{\small{GoOSE}}\xspace}
\newcommand{\greedy}{\textsc{\small{Greedy}}\xspace}
\newcommand{\SE}{\textsc{\small{SE}}\xspace}
\newcommand{\macopt}{\textsc{\small{MaCOpt}}\xspace}
\newcommand{\safemac}{\textsc{\small{SafeMaC}}\xspace}
\newcommand{\passivemac}{\textsc{\small{PassiveMaC}}\xspace}

\newcommand{\mac}{\textsc{\small{MaC}}\xspace}
\newcommand{\bo}{\textsc{\small{BO}}\xspace}
\newcommand{\ucb}{\textsc{\small{UCB}}\xspace}

\newcommand{\mypar}[1]{\textbf{#1}.}

\newcommand*{\eq}{Eq.}

\newcommand*{\lem}{Lem.}

\newcommand*{\alg}{Alg.}

\newcommand{\trans}{\top}

\newcommand{\T}{\mathcal T}

\newcommand{\G}{\mathcal G}

\newcommand{\N}{\mathcal N}

\newcommand{\R}{\mathbb{R}}

\newtheorem{lemma}{Lemma}

\newtheorem{theorem*}{Theorem}

\newtheorem{proposition}{Proposition}

\newcommand{\LocAgents}{X}
\newcommand{\LocAgent}[2][]{x_{#2}^{#1}}
\newcommand{\Discat}[1][]{D^{#1}}
\newcommand{\tlDiscat}[1][]{\tilde{D}^{#1}}
\newcommand{\Obj}{F}
\newcommand{\Objfunc}[3]{\Obj(#1;#2,#3)} % If there is optional argument then []
\newcommand{\delgain}[3]{\Delta(#1;#2,#3)}
\newcommand{\Domain}{V}
\newcommand{\PtInDomain}{v}

\newcommand{\constrain}{q}
\newcommand{\density}{\rho}
\newcommand{\noise}{\eta}

\newcommand{\LipConst}{L_\constrain}

\newcommand{\Rbar}{\bar{R}_{\epsilon_\constrain}(\LocAgents_0)}
\newcommand{\RbarO}[1]{\bar{R}_{0}(#1)}
\newcommand{\Roperator}[3]{R^{\text{#1}}_{#2}({#3})}
\newcommand{\Rtilde}[2]{\tilde{R}^{#1}_{\epsilon_\constrain}(#2)}
\newcommand{\Rbareps}[2]{\bar{R}^{#1}_{\epsilon_\constrain}(#2)}
\newcommand{\Rtiloperator}[3]{\tilde{R}^{\text{#1}}_{#2}({#3})}
\newcommand{\OptiOperReach}[2]{O^{#1}_t(#2)}
\newcommand{\optiOper}[2]{o^{#1}_t(#2)}
\newcommand{\tilOptiOper}[2]{\tilde{O}^{#1}_t(#2)}

\newcommand{\PessiOperReach}[2]{P^{#1}_t(#2)}
\newcommand{\pessiOper}[2]{p^{#1}_t(#2)}
\newcommand{\tilPessiOper}[2]{\tilde{P}^{#1}_t(#2)}

\newcommand{\updensity}[1][]{u^{\density}_{#1}}
\newcommand{\lbdensity}[1][]{l^{\density}_{#1}}
\newcommand{\mudensity}[1][]{\mu^{\density}_{#1}}
\newcommand{\sigdensity}[1][]{\sigma^{\density}_{#1}}

\newcommand{\ubconst}[1][]{u^{\constrain}_{#1}}
\newcommand{\lbconst}[1][]{l^{\constrain}_{#1}}
\newcommand{\muconst}[1][]{\mu^{\constrain}_{#1}}
\newcommand{\sigconst}[1][]{\sigma^{\constrain}_{#1}}

\newcommand{\zDeci}{z}

\newcommand{\edges}[1]{\mathcal{E}_{#1}}
\newcommand{\BatchColl}[1]{\mathcal{B}_{#1}}
\newcommand{\batch}{B}

\newcommand{\epsconst}{\epsilon_{\constrain}}
\newcommand{\epsdensity}{\epsilon_{\density}}
\newcommand{\noiseconst}{\sigma^{-2}_{\constrain}}
\newcommand{\noisedensity}{\sigma^{-2}_{\density}}
\newcommand{\betaconst}[1][]{\beta^{\constrain}_{#1}}
\newcommand{\betadensity}[1][]{\beta^{\density}_{#1}}
\newcommand{\gammaconst}[1]{\gamma^{\constrain}_{#1}}
\newcommand{\gammadensity}[1]{\gamma^{\density}_{#1}}
\newcommand{\deltime}[1]{{\delta t^{\star}_{\density}}^{#1}}
\newcommand{\smdeltime}[1]{{\delta t}^{#1}_{\density}}

\newcommand{\Tdensity}{T_{\density}}
\newcommand{\tdensity}{t_{\density}}
\newcommand{\Tconst}{T_{\constrain}}
\newcommand{\tconst}{t_{\constrain}}

\newcommand{\pessiSet}[2][]{S_{#2}^{ p #1}}
\newcommand{\uniSet}[2][]{S_{#2}^{u #1}}
\newcommand{\optiSet}[2][]{S_{#2}^{ o, \epsconst #1}}
\newcommand{\uncertaindisk}{U}
\newcommand{\sumMaxWidth}[2][]{w^{#1}_{#2}}
\newcommand{\sumMaxwidth}[2][]{w^{#1}_{#2}}
\newcommand{\widthconst}[2]{\omega^{#1}_{#2}}
\newcommand{\numOfAgents}{N}
\newcommand{\diskradius}{r}

\newcommand{\simActReg}{r^{act}_{t}}
\newcommand{\simReg}[1]{r_{t}^{#1}} %{r_{act}(#1)}
\newcommand{\simOptiReg}[1]{r^{\textsc{\small{o}}\xspace}_{#1}}
\newcommand{\simLocReg}[1]{r^i_{\batch}(#1)}
\newcommand{\simOPTloc}{OPT^i_t}
\newcommand{\simOPT}{OPT_t}
\newcommand{\cumOPTloc}{OPT^i_l}
\newcommand{\cumOPT}{OPT}

\newcommand{\gap}{\delta}

\newcommand{\densitykernel}{K^\density}
\newcommand{\constkernel}{K^\constrain}

\newcommand{\kernelfunc}{k}

\newcommand{\trace}{Tr}

\newcommand{\factor}{|\Domain|}

\newcommand{\actualRegret}{Reg_{act}}
\newcommand{\optiRegret}{Reg^{\textsc{\small{o}}\xspace}_{act}}
\newcommand{\optiLocReg}{{Reg^{\textsc{\small{o}}\xspace}_{l}}}
\newcommand{\LocReg}{{Reg^{i}}}

\newcommand{\DiskCoverageRatio}{C_{D}}

\newcommand{\revcom}[1]{\textcolor{black}{#1}}

\allowdisplaybreaks
\makeatletter
\def\thanks#1{\protected@xdef\@thanks{\@thanks
        \protect\footnotetext{#1}}}
\makeatother

\title{%SafeMaC: 
Near-Optimal Multi-Agent Learning for \\Safe Coverage Control
} %SMART, SafeMAC, SafeMACOpt
\author{%
  Manish Prajapat\thanks{$\dagger$ Joint supervision. Code available at \url{https://github.com/manish-pra/SafeMaC}}  \\
  ETH Zurich\\
  \texttt{manishp@ai.ethz.ch} \\
   \And
   Matteo Turchetta \\
   ETH Zurich \\
   \texttt{matteotu@inf.ethz.ch} \\
   \And
   Melanie N. Zeilinger\textsuperscript{$\dagger$} \\
   ETH Zurich \\
   \texttt{mzeilinger@ethz.ch} \\
   \And
   Andreas Krause\textsuperscript{$\dagger$} \\
   ETH Zurich \\
   \texttt{krausea@ethz.ch} \\
}

\begin{document}
\doparttoc % Tell to minitoc to generate a toc for the parts
\faketableofcontents % Run a fake tableofcontents command for the partocs

% \part{} % Start the document part
% \parttoc % Insert the document TOC

\maketitle

\begin{abstract}
\looseness=-1
In multi-agent coverage control problems, agents navigate their environment to reach locations that maximize the coverage of some density. In practice, the density is rarely known \emph{a priori}, further complicating the original NP-hard problem. Moreover, in many applications, agents cannot visit arbitrary locations due to \emph{a priori} unknown safety constraints. In this paper, we aim to efficiently learn the density to approximately solve the coverage problem while preserving the agents' safety. We first propose a conditionally linear submodular coverage function that facilitates theoretical analysis. Utilizing this structure, we develop \macopt, a novel algorithm that efficiently trades off the exploration-exploitation dilemma due to partial observability, and show that it achieves sublinear regret. Next, we extend results on single-agent safe exploration to our multi-agent setting and propose \safemac for safe coverage and exploration. We analyze \safemac and give first of its kind results: near optimal coverage in finite time while provably guaranteeing safety. We extensively evaluate our algorithms on synthetic and real problems, including a biodiversity monitoring task under safety constraints, where \safemac outperforms competing methods.
\end{abstract}

\section{Introduction}
\label{sec:introduction}
% Safe coverage control
\looseness -1 In multi-agent coverage control (\mac) problems, multiple agents coordinate to maximize coverage over some spatially distributed events. Their applications abound, from collaborative mapping~\citep{Swarm-SLAM},  environmental monitoring \citep{habitat-monitor}, inspection robotics \citep{inspection-app} to sensor networks~\citep{sensor-networks}. In addition, the coverage formulation can address core challenges in cooperative multi-agent RL~\citep{maddpg,prajapat2021competitive}, e.g., \emph{exploration} \citep{liu2021cooperative-exploration-challenge}, by providing high-level goals. In these  applications, agents often encounter safety constraints that may lead to critical accidents when ignored, e.g., obstacles \citep{aude-obstacle} or extreme weather conditions \citep{weather-hazard-keiv, gao2021weather}. 

Deploying coverage control solutions in the  real world presents many challenges: (\textit{i}) for a given density of relevant events, this is an \textit{NP hard problem} \citep{submodularity-andreas}; (\textit{ii}) such \textit{density} is \textit{rarely known} in practice \citep{habitat-monitor} and must be learned from data, which presents a complex active learning problem as the quantity we measure (the density) differs from the one we want to optimize (its coverage); (\textit{iii}) agents often operate under \textit{safety-critical} conditions, \citep{aude-obstacle, weather-hazard-keiv, gao2021weather}, that may be \textit{unknown a priori}. This requires cautious exploration of the environment to prevent catastrophic outcomes. While  prior work addresses subsets of these challenges (see \cref{sec:related_work}), we are not aware of methods that address them jointly.

\looseness=-1
This work makes the following contributions toward efficiently solving safe coverage control with \emph{a-priori} unknown objectives and constraints. 
\textbf{Firstly}, we model this multi-agent learning task as a \emph{conditionally linear} coverage function.
%such that for every agent, the resulting objective is a linear density function given the locations of the other agents.
We use the \emph{monotonocity} and the \emph{submodularity} of this function to propose \macopt, a new algorithm for the unconstrained setting that enjoys sublinear cumulative regret  and efficiently recommends a near-optimal solution.
%even under the challenging exploration-exploitation dilemma due to partial observability. 
\textbf{Secondly}, we extend \goose \cite{turchetta2019safe}, an algorithm for single agent safe exploration, to the multi-agent case. Combining our extension of \goose with \macopt, we propose \safemac, a novel algorithm for safe multi-agent coverage control. We analyze it and show it attains a near-optimal solution in a finite time.
% We theoretically analyze \safemac and give first of its kind results: even under \emph{a-priori} unknown constraints and density, it achieves a near-optimal solution up to arbitrary precision in a finite time. \\
\textbf{Finally}, we demonstrate our algorithms on a synthetic and two real world applications: safe biodiversity monitoring and obstacle avoidance. We show \safemac finds better solutions than algorithms that do not actively explore the feasible region and is more sample efficient than competing near-optimal safe algorithms.

\section{Problem Statement}
\label{sec:problem_statement}

% \begin{itemize}
%     \item General coverage control problem
%     \item Our specific objective
%     \item Constraint
%     \item Unknown density and constraint
% \end{itemize}

We present the safety-constrained multi-agent coverage control problem (\cref{fig: coverage-maximize}) that we aim to solve.

\looseness=-1
\mypar{Coverage control}
Coverage control models situations where we want deploy a swarm of dynamic agents to maximize the coverage of a quantity of interest, see \cref{fig: coverage-maximize}. 
% For example, in shared mobility, we aim to dispatch a fleet to maximize the coverage of customer-dense areas or 
% For example, in biodiversity monitoring, we want to dispatch our mobile sensors to monitor densely populated areas. 
Formally, given a finite\footnote{Continuous domains can be handled via discretization} set of possible locations $\Domain$, the goal of coverage control is to maximize a function $\Obj\colon 2^{\Domain}\to \R$ that assigns to each subset, $\LocAgents \subseteq \Domain$, the corresponding coverage value. For $\numOfAgents$ agents, the resulting problem is $\argmax_{\LocAgents\colon |\LocAgents|\leq \numOfAgents} \Obj(\LocAgents)$. The discrete domain $\Domain$ can be represented by a graph, where nodes represent locations in the domain, and an edge connects node $\PtInDomain$ to $\PtInDomain'$ if the agent can go from $\PtInDomain$ to $\PtInDomain'$. This corresponds to a deterministic MDP where locations are states and edges represent transitions.

\mypar{Sensing region}
\looseness=-1
Depending on the application, we may use different definitions of $\Obj$. Here, we  model cases where agent $i$ at location  $\LocAgent[i]{}$ covers a limited sensing region around it, $\Discat[i]$. While $\Discat[i]$ can be any connected subset of $\Domain$, in practice it is often a ball centered at $\LocAgent[i]{}$. 
% $\Discat[i]=\{\PtInDomain \in \Domain\colon \|\PtInDomain-\LocAgent[i]{}\|_2\leq r\}$. 
Given a function $\density\colon\Domain\to \R$ denoting the density of a quantity of interest at each $\PtInDomain \in \Domain$, our coverage objective is
\vspace{-1mm}
\begin{equation}
    \Objfunc{\LocAgents}{\density}{\Domain} = \sum_{\LocAgent[i]{} \in \LocAgents} \sum_{\PtInDomain \in \Discat[i-]} \density(\PtInDomain) / \factor, \numberthis \label{eqn: disk-coverage}
% \vspace{-1mm}
\end{equation}
where $\Discat[i-]\coloneqq\Discat[i]\setminus \Discat[1\colon i-1]$ indicates the elements in $\Domain$ covered by agent $i$ but not agents $1 \colon i-1$, $\Discat[1\colon i-1] = \cup_{j=1}^{i-1}\Discat[j]$ and $\factor$ denotes cardinality of the domain $\Domain$. 
%$N=\max_{x\in\Domain} |\Discat[x]|$. 
% When the set $\Domain$ results from the discretization of a continuous domain, normalizing by $N$ makes the objective invariant to the discretization step.
% Normalizing by $N$ makes the objective invariant to the size of the sensing region.
%
\begin{figure}
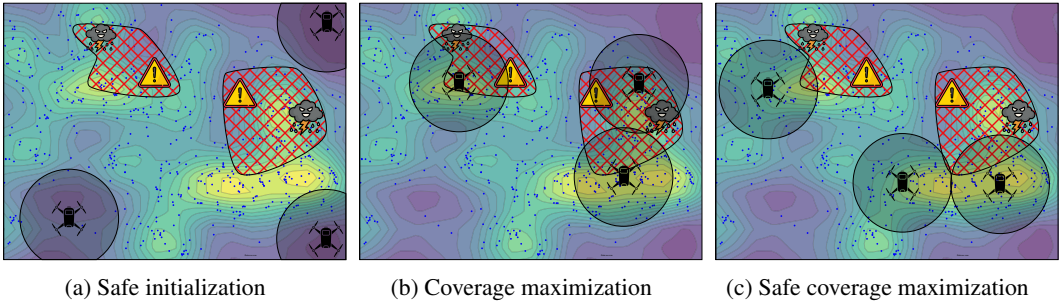

% \hspace{-40.00mm}
% \centering
\begin{subfigure}[t]{0.32\columnwidth}
\centering
\scalebox{0.147}{\input{images/drone-init-ready}}
% \frame{
% \includegraphics[scale=0.147]{images/drone-init-ready.pdf}}
% \frame{\begin{overpic}[scale=0.147]{images/drone-init-ready.pdf}\put(1,1){\fontsize{1.2}{4}\selectfont flaticon.com}\end{overpic}}
\caption{Safe initialization} \label{fig: TD-init}
\end{subfigure}
% \hspace{0.05mm}
~
\begin{subfigure}[t]{0.32\columnwidth}
\centering
\scalebox{0.147}{\input{images/unsafe-drone-ready}}
% \frame{\begin{overpic}[scale=0.147]{images/unsafe-drone-ready.pdf}\put(1,1){\fontsize{1.2}{4}\selectfont flaticon.com}\end{overpic}}
% \frame{
% \includegraphics[scale=0.147]{images/unsafe-drone-ready.pdf}}
\caption{Coverage maximization} \label{fig: TD-unsafe}
\end{subfigure}
% \hspace{20.00mm}
~
\begin{subfigure}[t]{0.32\columnwidth}
\centering
\scalebox{0.147}{\input{images/safe-drone-ready}}
% \frame{\begin{overpic}[scale=0.147]{images/safe-drone-ready.pdf}\put(1,1){\fontsize{1.2}{4}\selectfont flaticon.com}\end{overpic}}
% \frame{
% \includegraphics[scale=0.147]{images/safe-drone-ready.pdf}}
\caption{Safe coverage maximization} \label{fig: TD-safe}
\end{subfigure}
% \hspace{20.00mm}
% \hspace{-6.00mm}
\caption{\revcom{The three drones aim to maximize the gorilla nests' coverage (grey shaded circle) while avoiding unsafe extreme weather zones (red cross pattern). The contours (yellow is high and purple is low) represent the density of gorilla nests (blue dots). The density and the constraint are a-prior unknown. a) To be safe, drones apply a conservative strategy and do not explore, which results in poor coverage. In b), the drones maximize coverage but get destroyed in extreme weather. c) shows \safemac solution. The drones strike a balance, trading off between learning the density and the constraints, and thus achieve near-optimal coverage while always being safe.}}
% \caption{\revcom{Task description for safe coverage control a) Agents start from a safe location. A safe but conservative strategy is to stay at the initial seed, which may result in unknown poor coverage. For this, b) the agents iteratively collect density measurements and attempt to maximize coverage but may result in accidents due to unknown constraints. Hence, c) the agents must plan a safe path to learn about maximizing coverage. This requires agents to balance a trade-off of learning the density vs the constraints while being safe always.}}
%can safely expand the pessimistic set by evaluating decisions on the boundary (dark green shaded), but and evaluate in goal directed manner.
\vspace{-4.00mm}\label{fig: coverage-maximize}
\end{figure}

\mypar{Safety}
\looseness=-1
In many real-world problems, agents cannot go to arbitrary locations due to safety concerns. 
% For example, in bio-diversity monitoring, we should not expose our agents to extreme weather conditions to prevent costly damage. 
To model this, we introduce a constraint function $\constrain\colon\Domain\to\R$ and  we consider safe all locations $\PtInDomain$ satisfying $\constrain(\PtInDomain)\geq 0$. Such constraint restricts the space of possible solutions of our problem in two ways. First, it prevents agents from monitoring from unsafe locations.
% , i.e., any $\LocAgents=\{ \LocAgent[1]{}, . . . ,\LocAgent[\numOfAgents]{} \}$ such that $\exists i\colon q(\LocAgent[i]{})<0$ .
Second, depending on its dynamics, agent $i$ may be unable  to safely reach a disconnected safe area starting from $\LocAgent[i]{0}$, see \cref{fig: disk}. We denote with $\Rbareps{}{\{ \LocAgent[i]{0} \}}$ the largest safely reachable region starting from $\LocAgent[i]{0}$ and with $\BatchColl{}$ a collection of batches of agents such that all agents in the same batch $\batch$ share the same safely reachable set, $\forall i,j \in B\colon \Rbareps{}{\{ \LocAgent[i]{0} \}} \cap \Rbareps{}{\{ \LocAgent[j]{0} \}}\neq \emptyset$, see \cref{Apx: definitions} for formal definitions. Based on this, we define the safely reachable control problem
\vspace{-1mm}
\begin{equation}
    \sum_{\batch \in \BatchColl{}} \max_{\LocAgents^{\batch} \in  \Rbareps{}{\LocAgents^{\batch}_0}} \Objfunc{\LocAgents^{\batch}}{\density}{\Rbareps{}{\LocAgents^{\batch}_0}} \label{eq:main_problem}, 
\end{equation}
% \vspace{-1mm}
where $\LocAgents^{\batch}_0=\{\LocAgent[i]{0}\}_{i\in\batch}$ are the starting locations of all agents in $\batch$ and $\Rbareps{}{\LocAgents^{\batch}_0}=\bigcup_{i\in\batch}\Rbareps{}{\{ \LocAgent[i]{0} \}}$ indicates the largest safely reachable region from any point $\LocAgent[i]{0}$ for all $i$ in $\batch$ \revcom{(since the agents have the same dynamics, $\Rbareps{}{\LocAgents^{\batch}_0}=\Rbareps{}{\{ \LocAgent[i]{0} \}}, \forall i\in \batch$). In safety-critical monitoring, there may be unreachable safe regions. However, since agents should be able to collect measurements if required, we focus only on covering the safely reachable region.}
%, which, by definition of $\batch$ is equal to $\Rbareps{}{\{ \LocAgent[i]{0} \}}$ for any $i \in \batch$.

\mypar{Unknown density and constraint}
\looseness=-1
In practice, the density $\density$ and the constraint $\constrain$ are often unknown \textit{a priori}. However, the agents can iteratively obtain noisy measurements of their values at target locations. 
% In the following, we denote with $t_\density$ and $t_\constrain$ the iterations when agents measure the density and the constraint, respectively, and with $t=t_\density+t_\constrain$ their sum. 
We consider synchronous measurements, i.e., we wait until all agents have collected the desired measurement for the current iteration before moving to the next one. Here, we focus on the high-level problem of choosing informative locations, rather than  the design of low-level motion planning
\footnote{\revcom{Agents can use their transition graph to find a path between two goals. In a continuous domain, the path can be tracked with a controller (e.g., MPC)}}
. Therefore, our goal is to find an approximate solution to the problem in \cref{eq:main_problem} preserving safety throughout exploration, i.e., at every location visited by the agents, while taking as few measurements as possible in case the dynamics of the agents are deterministic and known as in \cite{turchetta2019safe}.
%We focus on fixed-confidence, where we minimize rounds for achieving fixed confidence of the returned solution.

% Note that, between $t$ and $t+1$, there may be multiple time steps of the dynamical system as agents take time to navigate to the desired location. Focusing on iteration efficiency rather than time step efficiency matters when travelling costs are negligible. Achieving safety during exploration means that all agents must fulfil the constraint at all times, including those in between iterations.

% \begin{itemize}
%     \item General coverage control problem: Done
%     \item Our specific objective: Done
%     \item Additions of constraints: Done
%     \item Problem: Done
%     \item Challenges related to unknown density and constraints. Mention known dynamics.
% \end{itemize}

% \vspace{-3mm}
\section{Background} 
% \vspace{-2.5mm}
\label{sec:background}
% \begin{itemize}
%     \item Submodularity: Done
%     \item \goose: Done
%     \item Assumptions - Seed and regularity: Done
%     \item Assumptions in our case: safe seed, constraint regularity, density regularity: Done
%     \item Confidence intervals: Done
%     \item safety operators: Done
%     \item ergodicity operators: Done
% \end{itemize}
%

This section presents foundational ideas that our method builds on. In particular, it discusses (\textit{i}) monotone submodular functions
%  a class of functions that we can approximately optimize efficiently, 
and (\textit{ii}) previous work on single-agent safe exploration.

\mypar{Submodularity} 
\looseness=-1
Optimizing a function defined over the power set of a finite domain, $\Domain$,
%, like the one in \cref{sec:problem_statement}, 
scales combinatorially with the size of $\Domain$ in general. In special cases, we can exploit the structure of the objective to find approximate solutions efficiently. Monotone submodular functions are one example of this.

\looseness=-1
A set function $\Obj: 2^{\Domain} \to \R$ is \textit{monotone} if for all $A \subseteq B \subset \Domain$ we have $\Obj(A) \leq \Obj(B)$. It is \textit{submodular} if $ \forall A\subseteq B\subseteq \Domain,\PtInDomain \in V \setminus B $, we have, $\Obj(A \cup \{ \PtInDomain \}) - \Obj(A) \geq  \Obj(B \cup \{\PtInDomain \}) - \Obj(B)$. In  coverage control, this means adding  $\PtInDomain$ to $A$ yields at least as much increase in coverage than adding $\PtInDomain$ to $B$, if $A\subseteq B$.
%, a property often arising in real-world problems also known as diminishing returns. 
Crucially, \citep{Nemhauser-1minus1by-e} guarantees that the greedy algorithm produces a solution within a factor of $(1-1/e)$ of the optimal solution for problems of the type $\argmax_{\LocAgents: |\LocAgents|\leq \numOfAgents} \Objfunc{\LocAgents}{\density}{\Domain}$, when $\Obj$ is monotone and submodular.
In practice, the greedy algorithm often outperforms this worst-case guarantee \citep{better-than-worse-bound-Andreas} \revcom{and guaranteeing a solution better than $(1-1/e)$ factor is NP hard \citep{Feige-no-other-effi-algo}.}

\looseness=-1
The coverage function in \cref{eqn: disk-coverage} is a conditionally linear, monotone and submodular function (proof in \cref{Apx: disk-submodular}), which lets us use the results above to design our algorithm for safe coverage control.

\mypar{Goal-oriented safe exploration}
\goose 
\looseness -1 \citep{turchetta2019safe} is a single-agent safe exploration algorithm that extends unconstrained methods to safety-critical cases.
% by separating the exploration of the constraint and the objective. 
Concretely, it maintains under- and over-approximations of the feasible set, called pessimistic and optimistic safe sets. It preserves safety by restricting the agent to the pessimistic safe set. It efficiently explores the objective by letting the original unconstrained algorithm recommend locations within the optimistic safe set. If such recommendations are provably safe, the agent evaluates the objective there. Otherwise, it evaluates the constraint at a sequence of safe locations to  prove that such recommendation is either safe, which allows it to evaluate the objective, or unsafe, which triggers the unconstrained algorithm to provide a new recommendation.

\mypar{Assumptions}
\looseness=-1
To guarantee safety, \goose makes two main assumptions. First, it assumes there is an initial set of safe locations, $\LocAgents_0$, from where the agent can start exploring. Second, it assumes the constraint is sufficiently well-behaved, so that we can use data to infer the safety of unvisited locations. Formally, it assumes the domain $\Domain$ is endowed with a positive definite kernel $\kernelfunc^{\constrain}(\cdot ,\cdot)$, and that the constraint's norm in the associated \textit{Reproducing Kernel Hilbert Space} \cite{kernel-Schlkopf} is bounded, $\|\constrain\|_{\kernelfunc^{\constrain}} \leq B_{\constrain}$. 
% In the following, it should be clear from the context if a quantity refers to the constraint or the density. Thus, we omit the corresponding symbols unless necessary to declutter the notation. 
This lets us use Gaussian Processes (GPs) \citep{gp-Rasmussen}to construct high-probability confidence intervals for $\constrain$. 
We specify the \GP prior over $\constrain$ through a mean function, which we assume to be zero everywhere w.l.o.g., $\mu(\PtInDomain) = 0, \forall \PtInDomain \in \Domain$, and a kernel function, $\kernelfunc$, that captures the covariance between different locations. If we have access to $T$ measurements, at $\Domain_T = \{ \PtInDomain_t \}^{T}_{t=1}$ perturbed by i.i.d. Gaussian noise, $y_{T} = \{ q(\PtInDomain_t)+\noise_t \}_{t=1}^T$ with $\noise_t \sim \N(0, \sigma^{2})$, we can compute the posterior mean and covariance over the constraint at unseen locations $\PtInDomain$, $\PtInDomain^{\prime}$ as $\mu_T(\PtInDomain) = k^{\top}_T(\PtInDomain) (K_T + \sigma^2 I)^{-1}y_T$ and $k_t(\PtInDomain,\PtInDomain^{\prime}) = k(\PtInDomain,\PtInDomain^{\prime}) - k^{\top}_T(\PtInDomain)(K_T + \sigma^2 I)^{-1}k_T(\PtInDomain^{\prime})$,
% %
% \begin{align}
%     \mu_T(\PtInDomain) = k^{\top}_T(\PtInDomain) (K_T + \sigma^2 I)^{-1}y_T, ~~~k_t(\PtInDomain,\PtInDomain^{\prime}) = k(\PtInDomain,\PtInDomain^{\prime}) - k^{\top}_T(\PtInDomain)(K_T + \sigma^2 I)^{-1}k_T(\PtInDomain^{\prime}), \label{eq:GPposterior}
% \end{align}
% %
where $k_T(\PtInDomain) = [k(\PtInDomain_1,\PtInDomain), . . . , k(\PtInDomain_{T}, \PtInDomain)]^{\top}, K_T$ is the positive definite kernel matrix $[k(\PtInDomain,\PtInDomain^{\prime})]_{\PtInDomain,\PtInDomain^{\prime} \in \Domain_T}$ and $I \in \R^{T \times T}$ denotes the identity matrix.

\looseness=-1
In this work, we make the same assumptions about the safe seed and the regularity of $\constrain$ and $\density$.

\mypar{Approximations of the feasible set}
\looseness=-1
Based on the GP posterior above, \goose builds monotonic confidence intervals for the constraint at each iteration $t$ as  $\lbconst[t](\PtInDomain) \coloneqq \max \{ \lbconst[t-1](\PtInDomain), \muconst[t-1](\PtInDomain) - \betaconst[t] \sigconst[t-1](\PtInDomain) \}$ and $\ubconst[t](\PtInDomain) \coloneqq \min \{ \ubconst[t-1](\PtInDomain), \muconst[t-1](\PtInDomain) + \betaconst[t] \sigconst[t-1](\PtInDomain) \}$, which contain the true constraint function for every $\PtInDomain \in \Domain$ and $t\geq 1$, with high probability if $\betaconst[t]$ is selected as in \citep{beta-chowdhury17a} or \cref{sec:analysis}. \goose uses these confidence intervals within a set $S\subseteq\Domain$ together with the \revcom{$\LipConst$}-Lipschitz continuity of $\constrain$,  to define operators that determine which locations are safe in plausible worst- and best-case scenarios,
\begin{align}
    &\pessiOper{}{S} = \{ \PtInDomain \in \Domain, | \exists \zDeci \in S: \lbconst[t](\zDeci) - \LipConst d(\PtInDomain,\zDeci) \geq 0 \},\label{eq:PessOp}\\
    &\optiOper{\epsconst}{S} = \{ \PtInDomain \in \Domain, | \exists \zDeci \in S: \ubconst[t](\zDeci) - \epsconst - \LipConst d(\PtInDomain,\zDeci) \geq 0 \}\label{eq:OptOp}.
\end{align}
Notice that the pessimistic operator relies on the lower bound, $\lbconst$, while the optimistic one on the upper bound, $\ubconst$.
% These operators differ in that the pessimistic set relies on the lower bound of the constraint while the optimistic set uses the upper bound.
Moreover, the optimistic one uses a margin $\epsconst$ to exclude  "barely" safe locations as the agent might get stuck learning about them.
Finally, to disregard locations the agent could not safely reach or from where it could not safely return, \goose introduces the $\Roperator{ergodic}{}{\cdot,\cdot}$ operator.  $\Roperator{ergodic}{}{\pessiOper{}{S},S}$ indicates locations in $S$ or locations in $\pessiOper{}{S}$ reachable from $S$ and from where the agent can return to $S$ along a path contained in $\pessiOper{}{S}$. Combining $\pessiOper{}{S}$ and $\Roperator{ergodic}{}{\cdot,\cdot}$, \goose defines the pessimistic and ergodic operator $\tilPessiOper{}{\cdot}$, which it uses to update the pessimistic safe set. Similarly, it defines $\tilOptiOper{}{\cdot}$ using $\optiOper{\epsconst}{\cdot}$ to compute the optimistic safe set.

\savebox{\algleft}{
\begin{minipage}[t]{.44\textwidth}
%
%%%%%%%%%%% PSEUDOCODE GREEDY UCB %%%%%%%%%%%%%%%%%%
%
\begin{algorithm}[H]
\caption{Greedy UCB (\greedy)} \label{alg:greedyUCB}
% \SetAlgoLined
\begin{algorithmic}[1]
\State \textbf{Inputs} $\updensity[t-1], \lbdensity[t-1], \batch, \uniSet[]{t} $
\For{$i = 1 ,2 , . . . ,|\batch|$}
\State \!\!\!\! $\LocAgent[i]{t}\! \leftarrow\! \argmax\limits_{\LocAgent[i]{}} \!\!\!\!\! \sum\limits_{\PtInDomain \in \Discat[i] \backslash \Discat[1:i-1]_t  \cap \uniSet[]{t}} \!\!\!\! \!\!\!\!\! \updensity[t-1](\PtInDomain)$ \label{alg:greedyUCB:selection}
\State \!\!\!\! $\LocAgent[g,i]{t}\!\! \leftarrow\!\!\!\!\!\!\!\! \argmax\limits_{\PtInDomain \in \Discat[i] \backslash \Discat[1:i-1]_t \cap \uniSet[]{t}} \!\!\!\!\!\!\!\! \updensity[t-1](\PtInDomain)  -  \lbdensity[t-1](\PtInDomain)$ 
\EndFor
\State $\sumMaxWidth{t} \leftarrow \sum_{i=1}^{|\batch|}  \updensity[t-1](\LocAgent[g,i]{t})  -  \lbdensity[t-1](\LocAgent[g,i]{t})$ 
\State \textbf{Return} $\LocAgents^{\batch}_{t}, \sumMaxWidth{t}$
\end{algorithmic}
\end{algorithm}
\vspace{-1.65em}
%
%%%%%%%%%%% PSEUDOCODE MACOPT %%%%%%%%%%%%%%%%%%
%
\begin{algorithm}[H]
\caption{\macopt} \label{alg:macopt}
% \SetAlgoLined
\begin{algorithmic}[1]
\State \textbf{Inputs} $\LocAgents_0$, $\epsdensity$, $\Domain$, $GP_{\density}, t \leftarrow \revcom{1}$
\State $\LocAgents_{1}, \sumMaxWidth{1} \leftarrow \greedy(\updensity[0], \lbdensity[0],[\numOfAgents],\Domain)$
\While{$\sumMaxWidth{t} > \epsdensity$} \label{alg:macopt:while_condition}
\State \!\!\!\!$\forall i, \LocAgent[g, i]{t} \!\! \leftarrow \! \argmax\limits_{\PtInDomain \in \Discat[i-]_t} \updensity[t-1](\PtInDomain)  -  \lbdensity[t-1](\PtInDomain) $ \label{alg:macopt:uncertainty_sampling}
\State \!\!\!\!$\forall i$, $y^i_{\density_t} = \density(\LocAgent[g,i]{t}) + \noise_{\density}$, Update \GP
\State \!\!\!\!$t \leftarrow t + 1$
\State \!\!\!\!\!$\LocAgents_{t},\! \sumMaxWidth{t} \!\!\leftarrow \!\greedy(\updensity[t-1], \lbdensity[t-1],\![\numOfAgents],\!\Domain)$
\EndWhile
\State \textbf{Recommend} $\LocAgents_t$
\end{algorithmic}
\end{algorithm}
\vspace{-1.65em}
%
%%%%%%%%%%% PSEUDOCODE SAFE EXPANSION %%%%%%%%%%%%%%%%%%
%
\begin{algorithm}[H]
\caption{Safe Expansion (\SE)}\label{alg:se}
%, \citep{turchetta2019safe}}
% \SetAlgoLined
\begin{algorithmic}[1]
\State \textbf{Inputs} $\optiSet[]{t}, \pessiSet[]{t}, \LocAgent[g]{t} $
\State $A_t(p) \! \leftarrow \! \{ \PtInDomain \! \in \! \optiSet[]{t} \backslash \pessiOper{}{\pessiSet[]{t}} | h(\PtInDomain) = p \}$ \label{alg: se-priority}
\State $W^{\epsconst}_t \leftarrow \{ \PtInDomain \in \pessiSet[]{t} | \ubconst[t](\PtInDomain) - \lbconst[t](\PtInDomain) > \epsconst\}$ \label{alg: se-eps-uncertain-set}
\State $\alpha^{\star} \leftarrow \max {\alpha}~s.t.~ |G^{\epsconst}_t(\alpha)| > 0 $ \label{alg: se-pick-expander}
\If{Optimization problem feasible}
\State \!\!\!\!\!\! $\PtInDomain_t \! \leftarrow \! \argmax\nolimits_{\PtInDomain \in G^{\epsconst}_t(\alpha^{\star})} \! \ubconst[t](\PtInDomain) \! - \! \lbconst[t](\PtInDomain)$ \label{alg: se-max-pick}
\State \!\!\!\!\!\! Update GP with $y_t = \constrain(\PtInDomain_t) + \noise_{\constrain}$ \label{alg: se-measure}
\EndIf
\end{algorithmic}
\end{algorithm}
\end{minipage}}
\savebox{\algright}{%
\begin{minipage}[t]{.52\textwidth}
%
%%%%%%%%%%% PSEUDOCODE SAFEMAC %%%%%%%%%%%%%%%%%%
%
\begin{algorithm}[H]
\caption{\safemac}
% \SetAlgoLined
\begin{algorithmic}[1]
\State \textbf{Inputs} $\LocAgents_0$, $\LipConst$, $\epsdensity$, $\Domain$, $GP_{\density}, GP_{\constrain}$ 
\State $\forall i$, $\pessiSet[,i]{0} \xleftarrow{} \LocAgents_0$, $\optiSet[,i]{0} \!\! \xleftarrow{} \Domain$, $t \leftarrow \revcom{1} $
\State $\LocAgents_{1}, \sumMaxWidth{1} \leftarrow \greedy(\updensity[0], \lbdensity[0],[\numOfAgents],\Domain)$ \label{alg: init-greedy}
\While{$\forall i, (\optiSet[,i]{t-1} \backslash \pessiSet[,i]{t-1}) \cap  \Discat[i]_t \! \neq \! \emptyset$ \! or \! $\sumMaxWidth{t} > \epsilon_{\density}$} \label{alg: termination-condi}
\If{$\sumMaxWidth{t} > \epsdensity$}
\State \!\!\!\!\! $\forall i, \LocAgent[g, i]{t} \!\! \leftarrow \! \argmax\limits_{\PtInDomain \in \Discat[i-]_t} \updensity[t-1](\PtInDomain)  -  \lbdensity[t-1](\PtInDomain) $ \label{alg: coverage-phase}
\Else
\State \!\!\!\!\! $\forall i, \LocAgent[g, i]{t} \!\! \leftarrow \!\!\!\! \!\!\!\!\!\!\! \argmax\limits_{\PtInDomain \in (\optiSet[,i]{t-1} \backslash \pessiSet[,i]{t-1}) \cap  \Discat[i]_t} \!\!\!\!\! \!\!\!\!\! \ubconst[t-1](\PtInDomain)  -  \lbconst[t-1](\PtInDomain) $ \label{alg: exploration-phase}
\EndIf
\If{$ \exists ~i\in [\numOfAgents], \LocAgent[g, i]{t} \not\in \pessiSet[,i]{t}$} \label{alg: SE-module-st}
\State \!\!\!$\SE(\optiSet[,i]{t-1}\!\!,\pessiSet[,i]{t-1},\LocAgent[g, i]{t}), \forall i :\! \LocAgent[g, i]{t} \not\in \pessiSet[,i]{t}$  
\State \!\!\!$\pessiSet[,i]{t} \!\!\! \leftarrow \! \tilPessiOper{}{\pessiSet[,i]{t-1}}, \optiSet[,i]{t} \!\! \leftarrow \! \tilOptiOper{\epsconst}{\pessiSet[,i]{t-1}}, \forall i$ \label{alg: goose-operation}
\State \!\!\!$t \leftarrow t + 1$
\EndIf \label{alg: SE-module-end}

%
% \State Compute $\BatchColl{t}$ using \eq~\ref{eqn: uni-batching} \label{alg: batching}
\State $\forall i, \; \BatchColl{t}'(i) = \{ j  \in [\numOfAgents] | \uniSet[,i]{t}  \cap  \uniSet[,j]{t}  \neq  \emptyset \} \label{alg: safemac-batching1}$
\State $\BatchColl{t}  =  \bigcup\nolimits_{i \in [\numOfAgents]}  \BatchColl{t}'(i) \label{alg: safemac-batching2}$
% \State \!\!\!\!$\BatchColl{t} \!\! = \!\!\! \bigcup\limits_{i \in [\numOfAgents]} \!\!\!\! \BatchColl{t}'(i), \BatchColl{t}'(i) \!= \! \{ j \! \in \![\numOfAgents] | \uniSet[,i]{t} \! \cap \! \uniSet[,j]{t} \!\! \neq \! \emptyset \}$
\If{$\textit{for any}~\batch \in \BatchColl{t},\uniSet[,\batch]{t} \neq \uniSet[,\batch]{t-1}$} \label{alg: topo-change-st}
\State \!\!\!\!$\LocAgents^{}_{t},\sumMaxWidth{t} \!\leftarrow \! \greedy (\updensity[t-1],\lbdensity[t-1], \batch, \uniSet[,\batch]{t}$) 
\State \!\!\!\!\! $\forall i, \LocAgent[g, i]{t} \!\! \leftarrow \! \argmax\limits_{\PtInDomain \in \Discat[i-]_t} \updensity[t-1](\PtInDomain)  -  \lbdensity[t-1](\PtInDomain) $ \label{alg: max-density-width-under-disk}
\EndIf \label{alg: topo-change-ed}
\If{$\forall i, \LocAgent[g, i]{t} \in \pessiSet[,i]{t}$ and $\sumMaxWidth{t} > \epsilon_{\density}$}
\State $\forall i, y^i_{\density_t} = \density(\LocAgent[g,i]{t}) + \noise_{\density}$ \label{alg: density-meas}
\State Update GP i.e, compute $\updensity[t],\lbdensity[t]$ \label{alg: density-gp-update}
\State $t \leftarrow t + 1$ 
\State \!$\LocAgents^{}_{t},\!\sumMaxWidth{t} \!\leftarrow  \greedy(\updensity[t-1],\lbdensity[t-1],\!\batch, \uniSet[,\batch]{t-1})$\label{alg: new-greedy}
\EndIf
\EndWhile
\State \textbf{Recommend} $\LocAgents_t$
\end{algorithmic}
\label{alg: safemac}
\end{algorithm}
\end{minipage}}
\section{\macopt and \safemac}
\vspace{-1mm}
\label{sec:algorithms}

This section presents \macopt and \safemac, our algorithms for unconstrained and safety-constrained multi-agent coverage control, which we then formally analyze in \cref{sec:analysis}.
\subsection{\macopt: unconstrained multi-agent coverage control}
\label{sec:macopt}
\looseness=-1
% We start by studying the coverage control problem in the unconstrained case. 
%This is a problem of independent interest whose solution is a crucial component of the algorithm for the constrained problem.

\mypar{Greedy sensing regions}
\looseness=-1
In sequential optimization, it is crucial to balance  exploration and exploitation. 
% The rule to learn about the utility function by repeatedly observing $\argmax_{\PtInDomain\in \Domain} \sigdensity[t-1](\PtInDomain)$ can be quite useless since it reduces uncertainty everywhere in the domain. 
% On the other hand, maximizing the expected rewards can be stuck in the local optimum. 
\gpucb \citep{beta-srinivas} is a theoretically sound strategy to strike such a trade-off that works well in practice. Agents evaluate the objective at locations that maximize an upper confidence bound over the objective given by the GP model such that locations with either a high posterior mean (exploitation) or standard deviation (exploration) are visited. We  construct a valid upper confidence  bound for the coverage $\Obj(\LocAgents)$ starting from our confidence intervals on $\density$, by replacing the true density $\density$ with its upper bound $\updensity[t]$ in \cref{eqn: disk-coverage}. Next, we apply the greedy algorithm to this upper bound (Line \ref{alg:greedyUCB:selection} of \cref{alg:greedyUCB}) to select $\numOfAgents$ candidate locations for evaluating the density.
% where observations are obtained at $\argmax_{\PtInDomain\in \Domain} \mudensity[t-1](\PtInDomain) + \betadensity[t] \sigdensity[t-1](\PtInDomain)$, implicitly trades off the exploration and exploitation dilemma. 
% \mypar{Greedy pick} A natural extension of \gpucb to the setting of coverage maximization involving multiple agents is given by, 
% \begin{align*}
%     \LocAgent[i]{t} = \argmax_{\LocAgent[i]{}} \sum_{\PtInDomain \in \Discat[i-] } \mudensity[t-1](\PtInDomain) + \sqrt{\betadensity[t]} \sigdensity[t-1](\PtInDomain)   \label{eqn: greedy-pick-strategy} \numberthis
% \end{align*}
% where $\betadensity[t]$ are appropriate constants. The above rule is a myopic greedy selection as per upper confidence bound for picking the location of agent $i$ conditioned on the previous $1:i-1$ agents. 
%
% Here, we exploit that the coverage function is locally linear, that is, the incremental gain of coverage by adding $\LocAgent[i]{t}$ to $\LocAgents^{1:i-1}_t$ is linear. 
However, this simple exploration strategy may perform poorly, due to the fact that in order to reduce the uncertainty over the coverage $\Obj$ at $\LocAgents$, we must learn the density $\density$ at all locations inside the sensing region, $\bigcup_{i=1}^\numOfAgents \Discat[i]$,  rather than simply at $\LocAgents$. It is a form of partial monitoring \citep{kirschner2020information}, where the objective $\Obj$ differs from the quantity we measure, i.e., the density $\density$. Next, we explain how to choose locations where to observe the density for a given $\LocAgents$.
%
% Since the coverage maximization is an instance of partial monitoring, we do not directly observe the coverage but only receive noisy partial feedback in terms of the utility function. For such a setting, the observation rule in \cref{eqn: greedy-pick-strategy} is still too greedy and generally stuck in shallow local optimum. 
% (\todo{May be we refer to experiment section plot?})

\begin{figure}
% \hspace{-40.00mm}
% \centering
\begin{subfigure}[t]{0.45\columnwidth}
\centering
\scalebox{0.147}{\input{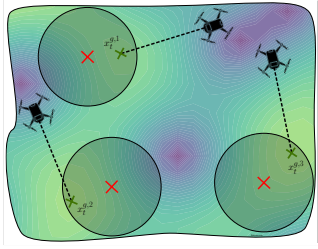}}
% \begin{overpic}[scale=0.160]{images/macopt-algo-domain-ready.pdf}\put(285,575){\fontsize{6}{4}\selectfont$\LocAgent[g,1]{t}$}\put(260,65){\fontsize{6}{4}\selectfont$\LocAgent[g,2]{t}$}\put(875,160){\fontsize{6}{4}\selectfont$\LocAgent[g,3]{t}$}\put(30,30){\fontsize{1.2}{4}\selectfont flaticon.com}\end{overpic}
% \includegraphics[scale=0.160]{images/macopt-algo-domain-ready.pdf}
\caption{Uncertainty sampling in \macopt} \label{fig: macopt-algo}
\end{subfigure}
% \hspace{20.00mm}
~
\begin{subfigure}[t]{0.45\columnwidth}
\input{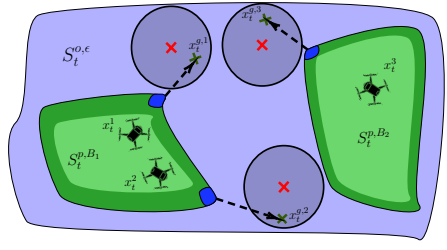}
\caption{Illustration of multi-agent \goose} \label{fig: multi-goose}
\end{subfigure}
% \hspace{-6.00mm}
\caption{\revcom{a) The contours represent the density uncertainty, and the red $\times$'s correspond to the maximum coverage locations evaluated by the \greedy \cref{alg:greedyUCB}. While these locations maximize coverage, they may not be informative about the coverage since the uncertainty can be low. Therefore, the agents collect measurements at the maximum uncertainty of the density in a disc (green $\times$'s, $\LocAgent[g, i]{t}$), also known as uncertainty sampling. b) In a constrained environment, \safemac evaluates $\LocAgent[g,i]{t}$ for all agents in the optimistic set $\optiSet[]{t}$ (violet) and set it as a next goal. It forms an expander region (dark blue) to safely expand the pessimistic safe set $\pessiSet[]{t}$ (green) toward the goal.}
%can safely expand the pessimistic set by evaluating decisions on the boundary (dark green shaded), but and evaluate in goal directed manner.
} 
\vspace{-4.00mm}
\end{figure}

\mypar{Uncertainty sampling}
\looseness=-1
Given location assignments $\LocAgents$ for the agents,  we measure the density to efficiently learn  the function $\Obj(\LocAgents)$. Intuitively, agent $i$ observes the density where it's most uncertain within the area it covers that is not covered by agents $\{1,\ldots i-1\}$, i.e., $\Discat[i-]_t$ (Line \ref{alg:macopt:uncertainty_sampling} of \alg~\ref{alg:macopt}, \cref{fig: macopt-algo}). 

% Next, \macopt forms a hypothetical disks around the locations picked as per greedy strategy \cref{eqn: greedy-pick-strategy}, and for each agent $i$ sets a goal to observe max uncertain location under its disk $i$ conditioned on previous disks $1:i-1$. Formally, the observation locations are given by:
% %
% \begin{align}
%   \LocAgent[g, i]{t} \leftarrow \argmax_{x \in \Discat[i-]_t} \sigdensity[t-1](x) \numberthis \label{eqn: max-uncertain-loc}
% \end{align}
%

\mypar{Stopping criterion}
\looseness=-1
The algorithm terminates when a near-optimal solution is achieved. Intuitively, this occurs when the uncertainty about the coverage value of the greedy recommendation is low. Formally, we require the sum of the uncertainties over the sampling targets to be below a threshold, i.e. , $\sumMaxWidth{t} = \sum_{i=1}^{\numOfAgents} u^\density_{t-1}(\LocAgent[g, i]{t})- l^\density_{t-1}(\LocAgent[g, i]{t}) \leq \epsilon_\density$ (Line \ref{alg:macopt:while_condition} of \cref{alg:macopt}). Importantly, this stopping criterion requires the confidence intervals to shrink only at regions that potentially maximize the coverage.

\mypar{\macopt}
\looseness=-1
Now, we introduce \macopt in \cref{alg:macopt}. At round $t$, we select the sensing locations for the agents, $\LocAgents_t$, by greedily optimizing the upper confidence bound of the coverage. Then, each agent $i$ collects noisy density measurements at the points of highest uncertainty within $\Discat[i-]_t$. Finally, we update our \GP over the density and, if the sum of maximum uncertainties within each sensing region is small, we stop the algorithm.
\vspace{-1mm}
\subsection{\safemac: safety-constrained multi-agent coverage control}
\label{sec:safemac}
\looseness=-1
%
% This section presents \safemac (\cref{alg: safemac}), our algorithm for safety-constrained coverage control problems. We start with a high-level explanation and, later, discuss the details of its components.

\mypar{Intuition} 
\looseness=-1
We adopt a perspective similar to \goose as we separate the exploration of the safe set from the maximization of the coverage. Given an over and under approximation of the safe set (whose computation is discussed later), we want to explore optimistically optimal goals for each agent, similar to \macopt. To this end, we find the maximizers of the density upper bound in the optimistic safe set with the \greedy algorithm. Then, we define sampling goals to learn the coverage at those locations.

% things that we talked about here. goal setting, 

\mypar{Phases of \safemac} 
\looseness=-1
Coverage values depend both on the density and the feasible region (\cref{eq:main_problem}). Thus, there are two sensible sampling goals given a disk assignment: i) \textit{optimistic coverage}: if we are uncertain about the density within the disks, we target locations with the highest density uncertainty  (\cref{alg: coverage-phase} of \cref{alg: safemac}); ii) \textit{optimistic exploration}: if we know the density within the disk but there are locations under it that we cannot classify as either safe (in $\pessiSet[]{}$) or unsafe (in $\Domain \setminus \optiSet[]{}$), we target those with the highest constraint uncertainty among them (\cref{alg: exploration-phase}). If all the goal locations are safe with high probability, which can only happen during \textit{optimistic coverage}, we safely evaluate the density there (\cref{alg: density-meas}). Otherwise, we explore the constraint with a goal directed strategy that aims at classifying them as either safe or unsafe similar to \goose (\cref{alg: SE-module-st}-\ref{alg: SE-module-end}). In case this changes the topological connection of the optimistic feasible set, we recompute the disks as this may change \greedy's output (\cref{alg: topo-change-st}-\ref{alg: topo-change-ed}). We repeat this loop until we know the feasibility of all the points under the disks recommended by \greedy and their density uncertainty is low (\cref{alg: termination-condi}). Next, we explain how the multiple agents coordinate their individual safe regions to evaluate a goal (\emph{\macopt in batches}), how the agents progress toward their goals (\emph{safe expansion}) and finally we describe \emph{\safemac convergence}. 

%
% \vspace{-1em}
% \vfill
\begin{figure*}
\noindent\usebox{\algleft}\hfill\usebox{\algright}
\vspace{-6mm}
\end{figure*}
% %
% 

\mypar{\macopt in batches}
\looseness=-1
In the multi-agent setting of \goose (see \cref{fig: multi-goose}), each agent $i$ maintains $\pessiSet[,i]{t}$ a pessimistic (or $\optiSet[,i]{t}$ an optimistic) belief of the safe locations, obtained by iteratively applying  $\tilPessiOper{}{\cdot}$ the pessimistic ( or $\tilOptiOper{}{\cdot}$ the optimistic) ergodic operators (see \cref{sec:background}) to the previous pessimistic belief $\pessiSet[,i]{t-1}$ (\cref{alg: goose-operation} of \cref{alg: safemac}). Since the agents cannot navigate to an arbitrary location in the constrained case, \safemac computes coverage maximizers on a restricted region, obtained by ignoring the known unsafe locations. To denote such a restricted region, we define a union set $\uniSet[,i]{t} \coloneqq \optiSet[,i]{t} \cup \pessiSet[,i]{t}$, which is the largest set known to be optimistically or pessimistically safe up to time $t$. Moreover, if the agents are topologically disconnected, they cannot travel from one safe region to another and the best strategy for any batch of agents is to maximize coverage locally. For this, we form a collection of batches $\BatchColl{t}$, such that any batch $\batch \in \BatchColl{t}$ contains agents that lie in topologically connected regions determined by the union set (\cref{alg: safemac-batching1}-\ref{alg: safemac-batching2} ). \safemac computes a  \greedy solution for each $\batch \in \BatchColl{t}$ in their corresponding $\uniSet[,B]{t} \coloneqq \cup_{i \in \batch} \uniSet[,i]{t}$. This is the largest set where the agents can find an optimistically safe path to travel. Analogous to $\BatchColl{t}$, we define $\BatchColl{t}^p$ as collection of batches where any $\batch \in \BatchColl{t}^p$ contains agents which are topologically connected in pessimistic set and $\pessiSet[,\batch]{t} \coloneqq \cup_{i \in \batch} \pessiSet[,i]{t}$. 

%given by $\BatchColl{t} = \bigcup_{i \in [\numOfAgents]} \BatchColl{t}'(i)~~\textit{where}~~\BatchColl{t}'(i) = \{ j \in [\numOfAgents] \, | \, \uniSet[,i]{t} \cap \uniSet[,j]{t} \neq \emptyset \},$
% 
% 
% \begin{align*}
%   \BatchColl{t} = \bigcup_{i \in [\numOfAgents]} \BatchColl{t}'(i)~~\textit{where}~~\BatchColl{t}'(i) &= \{ j \in [\numOfAgents] \, | \, \uniSet[,i]{t} \cap \uniSet[,j]{t} \neq \emptyset \}, ~~\textit{and}~~ \uniSet[,i]{t} \coloneqq \optiSet[,i]{t} \cup \pessiSet[,i]{t}.    \numberthis \label{eqn: uni-batching}
% \end{align*}
%

\mypar{Safe expansion} 
Safe expansion is the sub-routine inspired by \goose for goal-oriented exploration of the safe set that we use to learn about the feasibility of sampling targets. It uses a heuristic $h$ to assign priority scores $p$ to points that are optimistically but not pessimistically safe. Those determine locations whose feasibility is relevant to learn that of the sampling targets ( \cref{alg: se-priority}  of \cref{alg:se}).  A simple and effective choice for the heuristic is the inverse of the distance to the targets. Then, it identifies safe locations where the constraint is not yet known $\epsconst$-accurately (\cref{alg: se-eps-uncertain-set}). 
% (in \cref{alg: se-priority}) based on some heuristic $h(\PtInDomain)$ \goose assigns priority to the states about whose safety is still uncertain, In particular states lying in ($\optiSet[]{t} \backslash \pessiOper{0}{\pessiSet[]{t}}$). 
Among them, it determines the $\alpha$-immediate expanders, i.e., those that could potentially add locations with priority $\alpha$ to the pessimistic set, $G^{\epsconst}_t(\alpha) = \{ \PtInDomain \in W^{\epsconst}_t | \exists z \in A_t(\alpha): \ubconst[t](\PtInDomain) - \LipConst d(\PtInDomain, z) \geq 0 \} $. In \cref{alg: se-pick-expander}, it selects the non-empty  $\alpha$-expander set with the highest priority. In \cref{alg: se-max-pick} - \ref{alg: se-measure}, the agent evaluates the constraint at the location with the highest uncertainty in this set (see \citep{turchetta2019safe} for details).
% For more details we refer reader to $\SE$ module of \goose (Page 6 of \citet{turchetta2019safe}).
%All in all, \SE picks a point in the pessimistic set that can provide information about the highest priority state. The priorities are defined by some heuristic function which will ideally decide the direction in which we shall explore to reach our goal $\LocAgent[g]{t}$ efficiently. We refer the reader to \citet{turchetta2019safe} for a more detailed explanation of the safe expansion and the discussion on the choice of heuristics.

% does a goal directed exploration to reach a desired goal $\LocAgent[g]{t}$. It forms a set of potential expander $G(\alpha)$, based on some priority function. This priority function captures the direction in which we can potentially explore to know more about the our goal-to-go. Then it picks a most uncertain point from the set of potential expander. If there is not location in the expander than \goose guarantees that the goal-to-go is unsafe with high probability. 

% guarantees that the solution will be always in the optimistic + pessimistic set and not outside of it.

% the agents who shares a same union graph, will solve the \greedy algorithm together in their union graph.  
%  \greedy algorithm domain guarantees that the solution will be always in the optimistic + pessimistic set and not outside of it.

\mypar{\safemac convergence}
\looseness=-1
% Repeating the safe expansion step towards a fixed goal, $\LocAgent[g, i]{t}$, results in either of two cases: the pessimistic set increases to include $\LocAgent[g, i]{t}$ or the optimistic set  shrinks to exclude $\LocAgent[g, i]{t}$. Since coverage depends on both the density and the domain, with every union set change, \safemac computes a new goal-to-go for all the affected agents. This step ensures that we reach our goal without spending a savable budget on constraint exploration. The first phase is said to be converged in the union set, once the $\sumMaxWidth{t} \leq \epsdensity$. After the first phase if the disks does not cover regions whose safety we are uncertain about i.e, no locations in $(\optiSet[,i]{t} \backslash \pessiSet[,i]{t}) \cap  \Discat[i]_t) \forall i$, then \safemac is said to be converged (\cref{alg: termination-condi} of \cref{alg: safemac}). However, since \greedy algorithm is evaluated over union set, and the disk may cover regions in the uncertain region. 
% Next, we need to explore the safety of these uncertain regions. \safemac then switches to the exploration phase and sets the goal, a state that provide maximum information about the uncertain region $(\optiSet[,i]{t} \backslash \pessiSet[,i]{t}) \cap  \Discat[i]_t)$. In the exploration phase, \safemac is essentially doing safe expansion, but with a different goal. 
The \emph{optimistic coverage} phase switches to \emph{optimistic exploration} phase, when density uncertainty under the disks is low ($\sumMaxWidth{t} \leq \epsdensity$). In the exploration, either the topological connection of the optimistic feasible set changes or will classify the uncertain region as pessimistically safe. In the former case, \safemac will recompute a new coverage location and switch to the coverage phase. Alternatively, if the uncertain region is pessimistically safe, \safemac has converged since the density uncertainty in the exploration phase is already low. The phases show an interesting dynamics; \safemac continuously iterates between the \emph{optimistic exploration} and the \emph{optimistic coverage} phase until we know about the feasibility of the disk and their uncertainty is low. In the worst case, \safemac might explore the entire environment. In this case the sample complexity will be similar to a two-stage algorithm, where we explore the whole domain and then optimize coverage in the resulting known environment. However, in practice, \safemac is much better than this worst case.

%235
% phase, either the optimistic set will shrink and exclude the goal-to-go or will classify the uncertain236
% region as pessimistically safe. If the optimistic set shrink, SAFEMAC will recompute a new coverage237
% location and in case Wt > ερ, then SAFEMAC switches back to the coverage phase and runs again238
% until Wt ≤ ερ. Alternatively, if the uncertain region is pessimistically safe, the SAFEMAC is said to239
% be converged. Here, the phases show interesting dynamics, SAFEMAC continuously iterates between240
% the exploration and the coverage phase, but after every exploration phase, SAFEMAC runs multiple241
% instances of the coverage phase. In the worst case, since the environment is unknown SAFEMAC might242
% explore the entire environment. In terms of the samples, complexity will be similar to a two-stage243
% algorithm, where we explore the whole domain and then optimize coverage in the resulting known244
% environment    % \noindent\usebox{\algleft}\hfill\usebox{\algright}%   can be inserted anywhere with this commond, physically \input{sections/5-algo} should be before the line when you write the command
\section{Analysis}
\label{sec:analysis}
\looseness=-1
We now analyze \macopt's convergence and \safemac's optimality and safety properties.

\mypar{\macopt} 
\looseness=-1
To measure the progress of \macopt, we study its regret, i.e., the difference between its solution and the one we could find if we knew the true density. Since control coverage consists in maximizing a monotone submodular function, we cannot efficiently compute the true optimum even for known densities. However, we can efficiently find a solution that is at least $(1-1/e)$ within the optimum. Thus, we quantify performance using the following notion of cumulative regret,
%  In contrast to our setting, a clairvoyant system with a perfect knowledge of $\density$ can only guarantee a solution at least up to $(1-1/e)$ times the optimal in polynomial time. Hence, over time with iterative density observations, the best we hope is to achieve this near optimal solution. We quantify performance in the unconstrained case using the following notion of cumulative regret, 
\vspace{-2mm}
\begin{align*}
    \actualRegret(T) = \left( 1- \frac{1}{e} \right) \sum_{t=1}^{T} \Objfunc{\LocAgents_{\star}}{\density}{\Domain} - \sum_{t=1}^{T} \Objfunc{\LocAgents_t}{\density}{\Domain}, \numberthis \label{eqn: reg-def}
\vspace{-2mm}
\end{align*}
where $\Objfunc{\LocAgents_{\star}}{\density}{\Domain}$ is the optimal coverage.
% the best we hope for is to achieve $(1-1/e)$ times the optimal
%
% solution up to accuracy $\epsdensity$. Let $\LocAgents_{\star}$ denote the optimal set of locations of the agents. 
% $ \sum_{\batch \in \BatchColl{r}} \Objfunc{\LocAgents^{\batch}}{\density}{\Rbareps{}{\LocAgents^{\batch}_0}}$ %
% \textit{s.t.} $\LocAgents^{\batch} \in  \Rbareps{}{\LocAgents^{\batch}_0}$. 
% Based on this we have the following two objectives to achieve: First, for the known environment setup, we would like to get a no regret algorithm, i.e, cumulative regret $\actualRegret(T)$ defined below grows sublinear in time T.
We now state one of our main results, which guarantees that the cumulative regret of \macopt grows sublinearly in time (proof in \cref{Apx: thm-macopt}).
%
% \begin{theorem}
% \label{thm: macopt}
% Let $\delta \in (0,1)$ and $\betadensity[t]$ as in \citep{beta-chowdhury17a}, i.e.,  ${\betadensity[t]}^{1/2} = B_{\density} + 4 \sigma_{\density} \sqrt{\gammadensity{t + 1} + \ln(1/\delta)}$. With probability at least $1-\delta$, \macopt's regret defined in \cref{eqn: reg-def} is bounded by $\mathcal{O}(\sqrt{T \betadensity[T] \gammadensity{\numOfAgents T}})$,
% \vspace{-1mm}
% \begin{align*}
%     \text{Pr}\Bigg\{ \actualRegret(T) \leq \numOfAgents \sqrt{\frac{8 T  \betadensity[T] \gammadensity{\numOfAgents T}}{\log(1+\numOfAgents\noisedensity)}  }  \Bigg\} \geq 1- \delta. \numberthis \label{eq:macopt_thm}
% \end{align*}
% %
% \end{theorem}
% \begin{restatable*}{theorm}{restatemacopt}
% \label{thm:2goldbach}
% Every even integer greater than 2 can be expressed as the sum of two primes.
% \end{restatable*}
\begin{restatable*}{theorm}{restatemacopt}
\label{thm: macopt}
Let $\delta \in (0,1)$, ${\betadensity[t]}^{1/2} = B_{\density} + 4 \sigma_{\density} \sqrt{\gammadensity{\numOfAgents t} + \ln(1/\delta)}$ and $\DiskCoverageRatio = \max\nolimits_{\LocAgent[i]{} \in \Domain} |\Discat[i]|/|\Domain| \leq 1$. With probability at least $1-\delta$, \macopt's regret defined in \cref{eqn: reg-def} is bounded by $\mathcal{O}(\sqrt{T \betadensity[T] \gammadensity{\numOfAgents T}})$,
\vspace{-1mm}
\begin{align*}
    \text{Pr}\Bigg\{ \actualRegret(T) \leq  \sqrt{\frac{8 \DiskCoverageRatio \numOfAgents T  \betadensity[T] \gammadensity{\numOfAgents T}}{\log(1+\numOfAgents\noisedensity)}  }  \Bigg\} \geq 1- \delta. \tag{6} \label{eq:macopt_thm}
\end{align*}
\end{restatable*}
The proof of \ref{thm: macopt} builds on two key ideas. First, we exploit the conditional linearity of the submodular objective to bound the cumulative regret defined in \cref{eqn: reg-def} with a sum of per agent regrets. Secondly, we bound the per agent regret with the information capacity $\gammadensity{\numOfAgents T}$, a quantity that measures the largest reduction in uncertainty about the density that can be obtained from $\numOfAgents T$ noisy evaluations of it. Since $\gammadensity{\numOfAgents T}$ \citep{gammaT-vakili21a} grows sublinearly with $T$ for commonly used kernels, so does \macopt's regret in \cref{eq:macopt_thm}. The immediate corollary of the above theorem, when the \macopt stopping criteria is reached (\cref{alg:macopt:while_condition} of \cref{alg:macopt}) guarantees a near optimal solution up to $\epsdensity$ precision. 
%%%%%%%%%%%%%%USE THE PARA BELOW IF WE WANT TO INCLUDE INFORMATION GAIN AS WELL
%Secondly, we bound the per agent regret with the information \macopt can acquire through the noisy density observations. We can precisely quantify this notion through \emph{the information gain} $I(y_A;\density) = H(y_A) - H(y_A | \density)$, where H denotes the Shannon entropy and $A$ is the set of locations evaluated by \macopt. Similar to \citet{beta-srinivas}, we can bound information gain with a theoretical quantity \emph{maximum information gain} $\gammadensity{\numOfAgents T}$ obtained after T rounds and defined as $\gammadensity{\numOfAgents T} \coloneqq \sup_{A \subseteq \Domain: |A|=\numOfAgents T} I(y_A | \density)$. Since $\gammadensity{\numOfAgents T}$\citep{gammaT-vakili21a} grows sublinearly with $T$ for commonly used kernels, so does \macopt's regret in \cref{eq:macopt_thm}.
% \begin{remark}
% , it is easy to show that $    \Objfunc{\LocAgents_{t}}{\density}{\Domain} \geq (1- \frac{1}{e}) \Objfunc{\LocAgents_{\star}}{\density}{\Domain} - \epsilon_{\density}$
% \end{remark}
\begin{restatable*}{corolary}{restatemacoptcorollary} \label{cor: macopt}
% \macopt the stopping criteria is met (\cref{alg:macopt:while_condition}) 
Let $\tdensity^{\star}$ be the smallest integer, such that $\frac{\tdensity^{\star}}{\beta_{\tdensity^{\star}} \gamma_{\numOfAgents\tdensity^{\star}}} \revcom{\geq} \frac{8 \DiskCoverageRatio^2 \numOfAgents^2}{ \log(1+\numOfAgents\sigma^{-2}) \epsdensity^2}$, then there exists a $t < \tdensity^{\star}$ such that w.h.p, \macopt terminates and achieves, 
$\Objfunc{\LocAgents_{t}}{\density}{\Domain} \geq (1- \frac{1}{e}) \Objfunc{\LocAgents_{\star}}{\density}{\Domain} - \epsilon_{\density}$.
\end{restatable*}
% \begin{corollary}
% % \macopt the stopping criteria is met (\cref{alg:macopt:while_condition}) 
% Let $\tdensity^{\star}$ be the smallest integer, such that $\frac{\tdensity^{\star}}{\beta_{\tdensity^{\star}} \gamma_{\numOfAgents\tdensity^{\star}}} \leq \frac{8 \numOfAgents^2 }{ \log(1+\numOfAgents\sigma^{-2}) \epsdensity^2}$, then there exists a $t < \tdensity^{\star}$ such that w.h.p, \macopt terminates and achieves, 
% $\Objfunc{\LocAgents_{t}}{\density}{\Domain} \geq (1- \frac{1}{e}) \Objfunc{\LocAgents_{\star}}{\density}{\Domain} - \epsilon_{\density}$.
% \end{corollary}
 
\mypar{\safemac} 
\looseness=-1
This section presents our main result for safety-constrained multi-agent coverage control. In particular, \cref{thm: safeMac} (proof in \cref{Apx: thm-safemac}) guarantees that \safemac safely achieves near-optimal safe coverage in finite time. 
% For the constrained system, the underlying safe exploration algorithm \goose presents a finite time result (Theorem 1 \citep{turchetta2019safe}). We extend the single agent \goose result to a setting, when the multiple agents jointly explore the domain with information sharing while maintaining their own believe of optimistic and pessimistic safe set (Details in \cref{Apx: ma-goose}). Finally, \cref{thm: safeMac} builds on \goose finite time result and guarantees that \safemac can achieve near optimal coverage solution in finite time for multi-agent exploration while guaranteeing safety. 
%
\begin{restatable*}{theorm}{restatesafemac}
\label{thm: safeMac}
Let $\delta \in (0,1)$, \revcom{$\epsdensity\geq0$, $\|\density\|_{\kernelfunc^{\density}} \leq B_{\density}$, ${\beta^{\density}_t}^{1/2} = B_{\density} + 4 \sigma_{\density} \sqrt{\gammadensity{\numOfAgents t} + 1 + \ln(1/\delta)}$, $\gammadensity{\numOfAgents t}$ denote the information capacity associated with the kernel $\kernelfunc^{\density}$. Let $\constrain(\cdot)$ be $\LipConst$-Lipschitz continuous and $\epsconst, \beta^{\constrain}_t$, $\gammaconst{\numOfAgents t}$ be defined analogously. Given $\LocAgents_{0} \neq \emptyset$, $\constrain(\LocAgent[i]{0}) \geq 0$ for all $i \in [\numOfAgents]$. Then, for any heuristic $h_t : \Domain \to \R$, with probability at least $1-\delta$, we have $\constrain(\LocAgent[]{}) \geq 0$, for any $\LocAgent[]{}$ along the state trajectory pursued by any agent in \safemac. Moreover, let $t^{\star}_\density$ be the smallest integer such that $\frac{t^{\star}_{\density}}{\beta_{t^{\star}_{\density}} \gamma_{\numOfAgents t^{\star}_{\density} }} \geq \frac{8 \DiskCoverageRatio^2 \numOfAgents^2 }{ \log(1+\numOfAgents \sigma^{-2}) \epsdensity^2}$, with $\DiskCoverageRatio = \max\limits_{\LocAgent[i]{} \in \Domain} \frac{|\Discat[i]|}{|\Domain|} \leq 1$ and let $\tconst^{\star}$ be the smallest integer such that $\frac{\tconst^{\star}}{\beta_{\tconst^{\star}} \gamma_{\numOfAgents \tconst^{\star}}} \geq \frac{C |\RbarO{\LocAgents_0}|}{\epsconst^2}$, with $C = 8/\log(1+ \noiseconst)$} then, there exists $t \leq t^{\star}_\constrain + t^{\star}_{\density}$, such that with probability at least $1 - \delta$, 
\vspace{-1mm}
\begin{align*}
     \sum_{\batch \in \BatchColl{t}} \Objfunc{\LocAgents_t^{\batch}}{\density}{\RbarO{\LocAgents_0^{\batch}}} \geq \left( 1- \frac{1}{e} \right)\sum_{\batch \in \BatchColl{}} \Objfunc{\LocAgents_{\star}^{\batch}}{\density}{\Rbareps{}{\LocAgents_0^{\batch}}} - \epsdensity. \tag{7}
\end{align*}
\end{restatable*}

The theoretical analysis has two components: (\textit{i}) we show \safemac's coverage is near-optimal at convergence (\cref{lem: const-optimal-eps}), and (\textit{ii}) we prove it converges in finite time. Since \safemac learns the constraint \textit{and} the density, we must bound the sample complexity for both to prove (\textit{ii}). For the constraint, we extend the results for single-agent \goose to our multi-agent setting (\cref{Apx: ma-goose}).
% Sample complexity bound for the constraint is implied directly from the multi-agent \goose result. 
For the density, we use results from \cref{thm: macopt} to show that, within a coverage phase, the cumulative regret is sublinear. Next, we use additivity of the information gain (\cref{lem: additive-info-time-bound}) between any pair of coverage phases to bound the sample complexity of density for the subsequent coverage phases. Combining these results, we obtain \cref{thm: safeMac}.

\mypar{Intermediate recommendation} 
\cref{thm: safeMac} guarantees that \safemac converges to a safe and near-optimal solution. Can it also make sensible recommendations before the stopping criteria are met?
Ideally, such recommendations should (\textit{i}) be safely reachable and (\textit{ii}) ensure a minimum  coverage. To satisfy (\textit{i}), they should be in the pessimistic safe set, $\pessiSet{t}$. To satisfy (\textit{ii}), their coverage should be computed according to $\Objfunc{\cdot}{\lbdensity[t-1]}{\pessiSet[]{t}}$, i.e., assuming a worst-case density, $\lbdensity[t-1]$, and a worst-case feasible set, $\pessiSet{t}$. If the greedy recommendation $\LocAgents_t$ is in $\pessiSet{t}$, we can recommend it at intermediate steps. However, this is not always the case and we need an alternative. To this end, we compute $\LocAgents^{l, \batch}_t$, i.e., the greedy solution w.r.t. the worst-case objective, $\Objfunc{\cdot}{\lbdensity[t-1]}{\pessiSet[, \batch]{t}} \, \forall \batch \in \BatchColl{t}^p$.
% The intermediate recommendation shall satisfy two criteria, i) certify worst case coverage ii) the agents can safely reach the recommended location.
% A simple rule can be that \safemac keeps track of $\LocAgents_t$ and recommends that. This recommendation ensures a worst-case coverage of at least $\sum_{\batch \in \BatchColl{t}} \sum_{i \in \batch} \Objfunc{\{\LocAgent[i]{t}\}}{\lbdensity[t-1]}{\pessiSet[,i]{t}}$ or $\sum_{\batch \in \BatchColl{t}^p} \Objfunc{\LocAgents^{\batch}_t}{\lbdensity[t-1]}{\pessiSet[,\batch]{t}}$, which is obtained from using the  lower bound of density $\lbdensity[t-1]$ in the pessimistic safe set $\pessiSet[,\batch]{t} \; \forall \batch$.
% However, it may be possible that $\LocAgents^{\batch}_{t} \not\in\pessiSet[,\batch]{t} \; \forall \batch$. 
% For this, at any time $t$, \safemac computes a \greedy solution $\LocAgents^{l, \batch}_t$ using the lower confidence bound of density $\lbdensity[t-1]$ in the pessimistic set $\pessiSet[,\batch]{t}$ for each batch $\batch \in \BatchColl{p,t}$ to certify a worst case coverage of at least $\sum_{\batch \in \BatchColl{p,t}} \Objfunc{\LocAgents^{lg, \batch}_t}{\lbdensity[t-1]}{\pessiSet[]{t}}$. 
% Here, $\BatchColl{p,t}$ denotes the collection of batches, such that for any batch $\batch \in \BatchColl{p,t}$, $\batch$ contains agents within the shared pessimistic set. 
At any time $T$, \safemac recommends the best of either strategy up to time $T$ according to the worst-case objective.
% In the intermediate recommendation, \safemac returns best worst case coverage in the pessimistic achieved so far$\LocAgents_T = \argmax_{\LocAgents_t,\LocAgents^{lg}_t, ~\textit{for} ~ \forall t \leq T}  \Big\{ \sum_{\batch \in \BatchColl{t}} \Objfunc{\LocAgents^{\batch}_t}{\lbdensity[t-1]}{\pessiSet[]{t}}, \sum_{\batch \in \BatchColl{p,t}} \Objfunc{\LocAgents^{lg, \batch}_t}{\lbdensity[t-1]}{\pessiSet[]{t}} \Big\}~ s.t. \LocAgents_T \in \pessiSet[]{t}$ 
% we want to recommend locations along with certifying the minimum coverage. For this \safemac records locations of agents by running greedy algorithm in a pessimistic domain. 
%
%%%%%%%%%%%%%%%%%%%%%%%%%%%%%%%%%%%%%%
% \begin{align*}
%     \!\!\!\!    \LocAgents_T = \argmax_{\LocAgents_t,\LocAgents^{l}_t,   \forall t \leq T}  \Big\{ \sum_{\batch \in \BatchColl{t}^p} \Objfunc{\LocAgents^{\batch}_t}{\lbdensity[t-1]}{\pessiSet[,\batch]{t}}, \sum_{\batch \in \BatchColl{t}^p} \Objfunc{\LocAgents^{l, \batch}_t}{\lbdensity[t-1]}{\pessiSet[,\batch]{t}} \Big\}~ s.t. \LocAgents_T \in \pessiSet[]{T}~~\numberthis \label{eqn: pre-mature-recommendation} 
% \end{align*}
%
In \cref{apx: recommendation-rule}, we show that such recommendation is also near optimal at convergence.
\vspace{-1mm}
\section{Experiments}
\vspace{-1mm}
\label{sec:experiments}
\looseness=-1
This section compares \macopt and \safemac to existing methods (or their extensions) on synthetic and real-world problems. We validate our theoretical claims and observe their superiority. We set $\betaconst = 3$ and $\betadensity = 3$ for all $t \geq 1$, it ensures safety as well as efficient exploration in practice~\citep{turchetta2019safe}. Experiment details and extended empirical analysis are in \cref{apx: experiments}.

\begin{figure*}[t]
% 	\hspace{-3.00mm}
% 	\centering
%     \begin{subfigure}[t]{0.52\columnwidth}
%   	\centering
%   	\includegraphics[scale=0.9]{images/macopt2D-regret.pdf}
% %   	\setlength{\abovecaptionskip}{0pt}
% %   	\setlength{\belowcaptionskip}{0pt}
%     \caption{}
%     \label{fig: macopt2D-regret}
%     \end{subfigure}
%     \hspace{-6.00mm}
	\hspace{-3.00mm}
	\centering
    \begin{subfigure}[t]{0.33\columnwidth}
  	\centering
  	\includegraphics[scale=0.9]{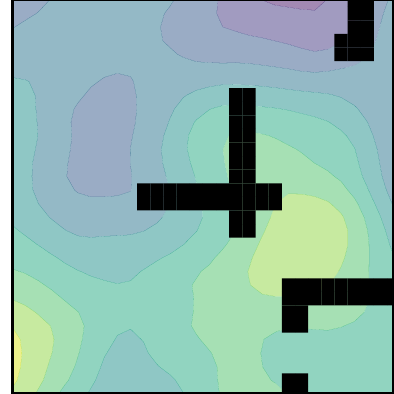}
    \caption{Obstacles environment}
    \label{fig: obstacle-env}
    \end{subfigure}
    \hspace{-6.00mm}
    ~
	\hspace{-4.00mm}
	\centering
    \begin{subfigure}[t]{0.33\columnwidth}
  	\centering
  	\includegraphics[scale=0.92]{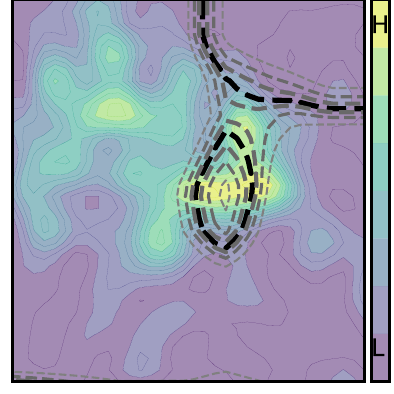}
    \caption{Gorilla nest environment}
    \label{fig: gorilla-env}
    \end{subfigure}
    \hspace{-6.00mm}
~
    \begin{subfigure}[t]{0.4\columnwidth}
  	\centering
  	\includegraphics[scale=0.9]{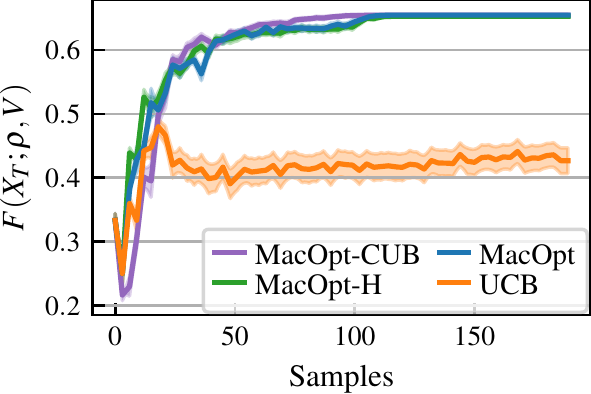}
    \caption{Coverage on gorilla nests (Sunny day)}
    \label{fig: gorilla_macopt2D}
    \end{subfigure}
    \hspace{-6.00mm}
\caption{
The contours in: a) show the synthetic density and the obstacles marked by the black blocks, b) show the Gorilla nests distribution with weather constraints marked by the black dashed line, and its contours with grey dashed line. c) Compares \macopt with \ucb in the safe gorilla environment. \revcom{\macopt does a more principled exploration of the coverage and does not stick to a local minimum.}
% The contours show the density sampled from the \GP in \cref{fig: obstacle-env} and the distribution of Gorilla nests in \cref{fig: gorilla-env}. In \cref{fig: obstacle-env} obstacles are marked by black blocks. In \cref{fig: gorilla-env}, the black dashed line shows constraints. Light grey lines show the contour of constraint distribution (plotted only near to threshold). \cref{fig: gorilla_macopt2D} compares the performance of \macopt with \ucb
}
\vspace{-6 mm}
\end{figure*}

\mypar{Environments} 
\looseness=-1
We perform our experiments with $\numOfAgents = 3$ agents in a $30 \times 30$ grid world where states are evenly spaced over $[0,3]^2$. Each agent's disk is defined as the region an agent can reach in $\diskradius = 5$ steps in the defined grid. We normalize coverage with a maximum value $\sum_{\PtInDomain \in \RbarO{\LocAgents_0}} \density(\PtInDomain) / \factor$. Below, we present the $3$ environments we consider.
%30$\times$30, 40$\times$40, 50$\times$50, 60$\times$60

i) In \emph{synthetic data}, both the density $\density$ and the constrain $\constrain$ are sampled from a $\GP$ with zero mean and Mat\'ern Kernel with $\nu=2.5$, scale $\sigma_k = 1$, and lengthscale $l= 2$. The observations are perturbed by i.i.d noise $\mathcal{N}(0, 10^{-3})$.
ii) In \emph{obstacles}, we sample maps with several block-shaped obstacles (\cref{fig: obstacle-env}) and we aim to maximize coverage while avoiding dangerous collisions. At $v$, each agent senses the distance to the nearest obstacle $d_m(\PtInDomain)$, which could be given by sensors such as 1D-Lidars. We use $\constrain'(\PtInDomain) = 1/(1+\exp(-1.5 d_m(v)))$, to map the distance between $[0,3]$ and saturate the constraint value for large distances, and we set $\constrain(\PtInDomain) = \constrain'(\PtInDomain)- 0.5$ to avoid collisions.
% We randomly prepared maps with blocks . The robot has a $360^{\circ}$ sensor, such as 1D-Lidar, and at any location $\PtInDomain$, the agent observes .
% we use constraints arising in autonomous robot navigation in unknown environments.
% We define $\constrain(\PtInDomain) = 1/(1+\exp(d_m(v)))$, to map the distance between $[0,1]$ and saturate the far away distance to 1. We choose a conservative value that $d_m>0.2$ to avoid the collision. 
The density is sampled from the same \GP as the synthetic case. 
iii) In \emph{gorilla nest}, we simulate a bio-diversity monitoring task, where we aim to cover areas with high density of gorilla nests with a quadrotor in the  Kagwene Gorilla Sanctuary (\cref{fig: gorilla-env}) . Regions affected by adverse weather (e.g. rain and storms) are unsafe for the drone due to higher chances of crashes and should be avoided. 
% The aerial vehicle shall maximize coverage while avoiding the regions of extreme weather. 
As a proxy for bad weather, we use the cloud coverage data over the KGS from OpenWeather \citep{OpenWeather}. The nest density is obtained by fitting a smooth rate function \citep{mojmir-cox} over Gorilla nest counts \citep{gorilla-kagwene}. 
% In the online setting, the agent measures density by counting number of nest in an image, and measures constrain based on humidity and wind speed. 
% The task is to maximize coverage over the density while satisfying the constraint. 
% To validate our theoretical claims, we show results compared with three quantities, i) optimal value $\Objfunc{\LocAgents_{\star}}{\density}{\Domain}$ computed using brute force approach (NP-Hard) ii) $(1-1/e)\Objfunc{\LocAgents_{\star}}{\density}{\Domain}$, minimum achievable objective for a clairvoyant system and iii) Obtaining noisy measurements at the disk center compute using \greedy UCB. 

\mypar{\macopt} 
\looseness=-1
We compare \macopt to \ucb, a baseline that skips the uncertainty sampling step from \cref{sec:macopt} and obtains measurements at the centers of the \greedy sensing regions. \revcom{We further develop two sample-efficient extensions of \macopt: i) Correlated upper bound (CUB), a variant of \macopt that constructs tighter upper confidence bound of the coverage function utilizing the covariance of density, instead of using the sum of density \ucb. ii) Hallucinated uncertainty sampling (H), a variant of \macopt that samples at the most informative location for each agent $i$, after hallucinating sampling locations of $\{1, \ldots i-1\}$ agents. Please see \cref{sec: variants-macopt} for theoretical analysis.}
\cref{fig: gorilla_macopt2D} shows a comparison in the \textit{gorilla} environment on a day of good weather, i.e. when all locations are safe. Here, \ucb gets stuck in a local optimum as it does not reduce the uncertainty of the density, whereas \macopt explores more and achieves a higher coverage value \revcom{up to 25\%}. \revcom{Moreover, variants of \macopt account for correlation and condition on other agents' measurement locations, which results in achieving the same coverage but more efficiently.}

\mypar{\safemac} 
\looseness=-1
We compare \safemac with two baselines: {\em i)} a two-stage algorithm \cite{safeRL-CMDP}, that first fully explores the feasible region, and then uses \macopt to maximize the coverage;  {\em ii)} \passivemac, a baseline inspired by \cite{safe-bo-sui15} that runs \macopt in the pessimistic set and passively measures the constraint in the process.
% The idea resonates with the baseline used in \citet{safe-bo-sui15}. \passivemac runs \macopt in the pessimistic set, and the constraints are observed passively at the exact location as density measurements. 
% Essentially, there is no active constraint exploration. 
%This baseline denotes the case of simply focusing on maximization in a constraint maximization problem. 
% We performed experiments on all the three environment setup described above.
\cref{fig: bar-gaincoverage,fig: bar-gainsamples} show the coverage at convergence and the number of samples to converge for \safemac and the two baselines across all the environments. The results are averaged over 50 instances produced using different seeds and samples for every environment. \revcom{In \cref{fig: bar-gainsamples}, the y-axis is normalized with the maximum number of samples in the instance and then averaged over all instances.}
\passivemac converges quickly but gets stuck in a local optimum as it does not actively explore the constraint. \safemac and Two-Stage converge to much higher coverage values. However, \safemac is \revcom{up to 50\%} more sample efficient 
% as it avoids fully exploring the feasible set 
thanks to its goal-oriented exploration. 
% In\cref{fig: bar-gaincoverage} we observe that \passivemac, converges quickly to a local minima where as \safemac and Two-Stage algorithm obtains a higher coverage value. On \cref{fig: bar-gainsamples}, we see that to converge to good solution Two-stage algorithm needs a lot of samples and where as \safemac with goal directed coverage is able to converge efficiently. 
\cref{fig: gorilla} shows the coverage value of the intermediate safe recommendations (\cref{sec:analysis}) in the \textit{gorilla} environment as a function of the number of  samples. It confirms the previous results: \safemac finds solutions comparable to Two-Stage more efficiently and \passivemac gets stuck in a local optimum.  
% Experiment details and analogous result with other environments are in \cref{apx: experiments}
% In \cref{fig: gorilla}, we plot intermediate safe recommendation for all the three methods based on our premature recommendation strategy. We observe that \safemac recommends better solution than \passivemac which get stuck in a minima and in less samples than two-stage achieves comparable coverage to that of two-stage algorithm.

\mypar{Scalability}
\looseness=-1
\revcom{\safemac utilizes the \greedy algorithm, which is linear in the number of nodes (domain size). In each iteration, \safemac computes a greedy solution $\numOfAgents$ times (one for each agent), which makes it linear in the number of agents. We model density and constraint using \GP, which scales cubically with the number of samples. To demonstrate scalability in practice, we conducted experiments with $\numOfAgents = 3,6,10,15$ agents each with domain length of $30, 40, 50$ and $60$ in \cref{apx: scaling}}
%$30\times30, 40\times40, 50\times50$ and $60\times60$
\begin{figure}
	\hspace{-3.00mm}
	\centering
    \begin{subfigure}[t]{0.3\columnwidth}
  	\centering
  	\includegraphics[scale=1]{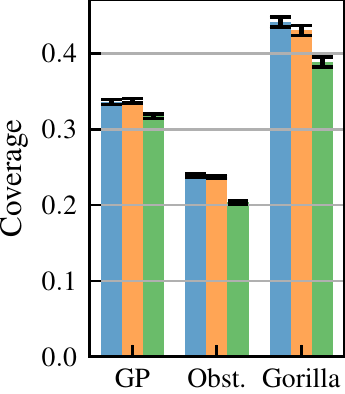}
    \caption{Coverage at convergence}
    \label{fig: bar-gaincoverage}
    \end{subfigure}
    \hspace{-6.00mm}
~
    \begin{subfigure}[t]{0.3\columnwidth}
  	\centering
  	\includegraphics[scale=1]{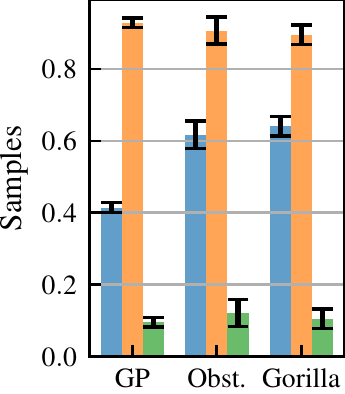}
    \caption{Total number of samples}
    \label{fig: bar-gainsamples}
    \end{subfigure}
    \hspace{-6.00mm}
    ~
    \begin{subfigure}[t]{0.44\columnwidth}
  	\centering
\includegraphics[scale=1]{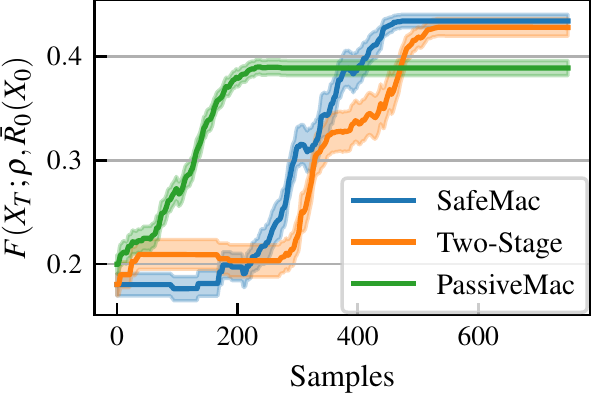}
    \caption{Coverage on Gorilla nests (Rainy day)}
    \label{fig: gorilla}
    \end{subfigure}
    \hspace{-6.00mm}
\caption{Comparison of \safemac with \passivemac and Two-Stage in all environments at convergence (a) and (b) and during optimization for the gorilla environment in (c). \revcom{\safemac trades-off learning about density and constraints, such that it finds a solution comparable to Two-Stage more efficiently, whereas \passivemac gets stuck in a local optimum.}   }\vspace{-4.00mm}
\end{figure}
\vspace{-2.00mm}

\section{Related work}
\label{sec:related_work}
\looseness=-1
Our work relates to multiple fields. We highlight the most relevant connections, referencing surveys where possible; an exhaustive overview is beyond the scope of this paper.
% In this work we consider a general setting of optimizing coverage over \emph{a priori} unknown density and constraint, using tools of bayeisan optimization, safety, coverage control and submodular optimization. Hence, our work connects and addresses challenges from different fields, but we limit here only to main directions and closely related works,

\mypar{Bayesian optimization} 
\looseness=-1
In BO, an agent sequentially evaluates a noisy objective, seeking to maximize it \citep{bo-review}.
In contrast, the quantity we measure  {\em differs} from our objective. Partial monitoring \citep{lattimore-bandit-book-partial} addresses such issues in an abstract setting
%with randomized algorithms rooted in information theory 
\citep{lattimore2019information, kirschner2020information}. We exploit special structure in our problem.
In coverage control with unknown density, this challenge
%complex exploration-exploitation dilemma 
is often addressed by learning the density uniformly over the domain \citep{voronoi-regret, andrea2015}. In contrast, \macopt 
%smartly allocates sampling budget and 
learns the density only at promising locations.
% This partial feedback amplifies the existing exploration-exploitation dilemma. 
% To overcome this, in the existing coverage control work, the agent naively learns the density up to arbitrary precision everywhere in the domain, which is unwarranted \citep{voronoi-regret, andrea2015} which in contrast to \macopt, where agent only reduces uncertainty at potential maximizers.

% partitions the space into Voronoi cells. 
\mypar{Coverage control} 
\looseness=-1
\mac with known densities is a well-studied NP hard \citep{sensor-placement-andreas} problem.
% The underlying \mac problem, even with a known density, is NP hard \citep{sensor-placement-andreas}. 
Many algorithms use efficient heuristics to converge quickly to a local optimum. One popular strategy is Lloyd's algorithm \citep{Lloyds-algo}, which has been studied in different settings, e.g., 
% which is based on Voronoi partitions of the space. Coverage control results based on Lloyd's algorithm are studied from different perspectives assuming 
with known densities \citep{cortes2004coverage, Lekien2009coverage}, \emph{a-priori} unknown densities \citep{andrea2015, Xu2011ccgp, Luo2018ccgp, Benevento2018ccgp}, \revcom{using graph neural networks \citep{gnn_policy}}, taking into account agent's dynamics and constraints \citep{CARRON20206107}, or in case of non-identical robots \citep{diff-robots}. These methods apply to continuous state and action spaces and show convergence to local optima, but lack optimality guarantees ~\citep{andrea2015, CARRON20206107, voronoi-regret} and sample complexity bounds. Moreover, their extension to non-convex, disconnected domains is not trivial \citep{bullo2012gossip-nonconvex}. \revcom{Coverage control is also studied in the episodic setting to learn the unknown policy or the environment using deep RL methods \citep{rl_unknown_Faryadi, Battocletti_rl_planning_unknown}.} 

\mypar{Submodular optimization} 
\looseness=-1
Submodular functions 
% are an important class of set functions \citep{krause2014submodular} that 
%formalize the notion of diminishing returns, which is ubiquitous in real-world problems, and 
are ubiquitous in machine learning \citep{bilmes2022submodularity} as they can be efficiently approximately maximized under different kinds of constraints \citep{krause2014submodular}. For example, the \greedy algorithm can be used in case of cardinality constraints \citep{Nemhauser-1minus1by-e} to maximize quantities like mutual information \citep{krause2012near} or weighted coverage functions \citep{Feige-no-other-effi-algo}. Online submodular maximization aims at optimizing unknown submodular functions from noisy measurements \citep{Chen2017InteractiveSB}. It has multiple applications, including optimization of numerical solvers \citep{streeter2007combining}, information gathering \citep{golovin2014online} \revcom{and crowd-sourced image collection summarization \citep{singla2016noisy}}.
% Another popular alternative for solving the submodular coverage problem is the \greedy algorithm, which establishes near optimal results \citep{Nemhauser-1minus1by-e}. 
% \citet{submodularity-andreas} demonstrates a greedy solution to the sensing optimization task in a non-sequential setting. 
Particularly related to ours is the work in \citep{YisongLSB}, which  proposes an algorithm for contextual news recommendation for linear user preferences with strong regret guarantees.
% Further, there are two lines of related work, i) online submodular bandit where at each round agents recommends a set of actions \citep{YisongLSB} ii) interactive submodular bandit, where agent recommends action one by one over T rounds \citep{Chen2017InteractiveSB}. 
In contrast to that setting, we consider dynamic agents, safety constraints and partial feedback.
% \macopt considers an environment defined by an MDP and sequentially optimize coverage with only partial feedback. 

\mypar{Safety}
\looseness=-1
%
% Safe learning in dynamical systems is a vast topic and many surveys exist \citep{brunke2021safe-review,ray2019benchmarking,leike2017ai}. 
Depending on the safety formulation and the assumptions, many algorithms have been proposed for safe learning in dynamical systems, e.g., based on model predictive control \citep{hewing2020learning}, curriculum learning \citep{turchetta2020safe}, Lyapunov functions \citep{berkenkamp2017safe,chow2018lyapunov}, reachability \citep{fisac2018general}, CMDPs \citep{achiam2017constrained}, behavioral system theory \citep{coulson2021distributionally}, and more \citep{brunke2021safe-review,ray2019benchmarking,leike2017ai}.  Here, we focus on the setting that is most closely related to ours, i.e., one with unknown but sufficiently regular instantaneous constraints that must be satisfied at all times. For stateless problems, e.g. \bo,  \citep{safe-bo-sui15,two-stage} propose algorithms with safety and optimality guarantees with different exploration strategies. For stateful problems, \citep{turchetta2016safemdp} studies the pure exploration case, while \citep{safeRL-CMDP} extends the two-stage approach from \citep{two-stage}. These approaches may be sample inefficient as they may explore the constraint in  regions  irrelevant for the objective. \goose \citep{turchetta2019safe} addresses this problem for both the stateful and stateless setting. The only work in this context that addresses multi-agent problems is \citep{dorsa-multi}. However, their objective differs from ours, and they do not establish safety guarantees.
\vspace{-1mm}
\section{Conclusion}
\vspace{-1mm}
\label{sec: conclusion}
\looseness=-1
We present two novel algorithms for multi-agent coverage control in unconstrained (\macopt) and safety critical environments (\safemac). 
% \macopt decouples unconstrained coverage problem in two steps. First, each agent optimizes coverage based on the \greedy algorithm on the \ucb of the density; second, it interacts with the environment only at the highest uncertainty locations below their respective disk. In \safemac, each agent additionally employs \goose to reach the desired goal safely. 
% We theoretically analyze both the algorithms and present two main results. 
We show \macopt  achieves sublinear cumulative regret, despite the  challenge of partial observability. Moreover, we prove \safemac achieves near optimal coverage in finite time while navigating safely. We demonstrate the superiority of our algorithms in terms of sample efficiency and coverage in real-world applications such as safe biodiversity monitoring. 

\looseness=-1
Currently, our algorithms choose informative targets but do not plan informative trajectories, which is crucial in robotics. We aim to address this in future work. Finally, while in many real-world applications the density and the constraints are as regular as assumed here, in some they are not. In these cases, our optimality and safety guarantees would not apply.
%We dedicated this paper to the high-level problem of choosing informative locations. In future, we plan to extend this study to a robotics setup while considering the design of low-level motion and trajectory planning. Moreover, we plan to extend our algorithms to different sensor observation models (such as camera or lidar) for applications such as  collaborative SLAM \cite{Swarm-SLAM}.

\section*{Acknowledgements}
\revcom{Manish Prajapat is supported by an ETH AI Center doctoral fellowship. Matteo Turchetta is supported by the Swiss National Science Foundation under NCCR Automation, grant agreement 51NF40 180545. We would like to thank Pawel Czyz for insightful discussions.}

\clearpage

% \input{extra/todo}

% \section*{References}
\bibliographystyle{unsrtnat} %amsalpha abbrvnat unsrt apalike
\bibliography{ref.bib}

% References follow the acknowledgments. Use unnumbered first-level heading for
% the references. Any choice of citation style is acceptable as long as you are
% consistent. It is permissible to reduce the font size to \verb+small+ (9 point)
% when listing the references.
% Note that the Reference section does not count towards the page limit.
% \medskip

% {
% \small

% [1] Alexander, J.A.\ \& Mozer, M.C.\ (1995) Template-based algorithms for
% connectionist rule extraction. In G.\ Tesauro, D.S.\ Touretzky and T.K.\ Leen
% (eds.), {\it Advances in Neural Information Processing Systems 7},
% pp.\ 609--616. Cambridge, MA: MIT Press.

% [2] Bower, J.M.\ \& Beeman, D.\ (1995) {\it The Book of GENESIS: Exploring
%   Realistic Neural Models with the GEneral NEural SImulation System.}  New York:
% TELOS/Springer--Verlag.

% [3] Hasselmo, M.E., Schnell, E.\ \& Barkai, E.\ (1995) Dynamics of learning and
% recall at excitatory recurrent synapses and cholinergic modulation in rat
% hippocampal region CA3. {\it Journal of Neuroscience} {\bf 15}(7):5249-5262.
% }

%%%%%%%%%%%%%%%%%%%%%%%%%%%%%%%%%%%%%%%%%%%%%%%%%%%%%%%%%%%%
\section*{Checklist}

% %%% BEGIN INSTRUCTIONS %%%
% The checklist follows the references.  Please
% read the checklist guidelines carefully for information on how to answer these
% questions.  For each question, change the default \answerTODO{} to \answerYes{},
% \answerNo{}, or \answerNA{}.  You are strongly encouraged to include a {\bf
% justification to your answer}, either by referencing the appropriate section of
% your paper or providing a brief inline description.  For example:
% \begin{itemize}
%   \item Did you include the license to the code and datasets? \answerYes{See Section 2 }%~\ref{gen_inst}.}
%   \item Did you include the license to the code and datasets? \answerNo{The code and the data are proprietary.}
%   \item Did you include the license to the code and datasets? \answerNA{}
% \end{itemize}
% Please do not modify the questions and only use the provided macros for your
% answers.  Note that the Checklist section does not count towards the page
% limit.  In your paper, please delete this instructions block and only keep the
% Checklist section heading above along with the questions/answers below.
% %%% END INSTRUCTIONS %%%

\begin{enumerate}

\item For all authors...
\begin{enumerate}
  \item Do the main claims made in the abstract and introduction accurately reflect the paper's contributions and scope?
    \answerYes{See \cref{sec:analysis} for the main theorems and \cref{sec:experiments} for the experimental results}
  \item Did you describe the limitations of your work?
    \answerYes{See \cref{sec: conclusion}}
  \item Did you discuss any potential negative societal impacts of your work?
    \answerNA{}
  \item Have you read the ethics review guidelines and ensured that your paper conforms to them?
    \answerYes{}
\end{enumerate}

\item If you are including theoretical results...
\begin{enumerate}
  \item Did you state the full set of assumptions of all theoretical results?
    \answerYes{See \cref{sec:background}}
        \item Did you include complete proofs of all theoretical results?
    \answerYes{See \cref{sec:analysis} with corresponding links to the Appendix for Proofs}
\end{enumerate}

\item If you ran experiments...
\begin{enumerate}
  \item Did you include the code, data, and instructions needed to reproduce the main experimental results (either in the supplemental material or as a URL)?
    \answerYes{Code is attached in the supplemental material along with the environment maps used. The code folder contains a ReadMe file containing instructions to reproduce the result.}
  \item Did you specify all the training details (e.g., data splits, hyperparameters, how they were chosen)?
    \answerYes{See \cref{sec:experiments} with corresponding links to \cref{apx: experiments} for complete experimental setup details}
        \item Did you report error bars (e.g., with respect to the random seed after running experiments multiple times)?
    \answerYes{ Errors bars are reported by running experiments with multiple random seeds and random environments}
        \item Did you include the total amount of compute and the type of resources used (e.g., type of GPUs, internal cluster, or cloud provider)?
    \answerYes{Compute details are in \cref{apx: experiments}}
\end{enumerate}

\item If you are using existing assets (e.g., code, data, models) or curating/releasing new assets...
\begin{enumerate}
  \item If your work uses existing assets, did you cite the creators?
    \answerYes{ In \cref{sec:experiments} we cited the data and the model used in the Gorilla nest dataset}
  \item Did you mention the license of the assets?
    \answerYes{In \cref{apx: experiments} along with the experiment setup details we mentioned the license of the relevant code and data used}
  \item Did you include any new assets either in the supplemental material or as a URL?
    \answerNA{}
  \item Did you discuss whether and how consent was obtained from people whose data you're using/curating?
    \answerNA{}
  \item Did you discuss whether the data you are using/curating contains personally identifiable information or offensive content?
    \answerNA{}
\end{enumerate}

\item If you used crowdsourcing or conducted research with human subjects...
\begin{enumerate}
  \item Did you include the full text of instructions given to participants and screenshots, if applicable?
    \answerNA{}
  \item Did you describe any potential participant risks, with links to Institutional Review Board (IRB) approvals, if applicable?
    \answerNA{}
  \item Did you include the estimated hourly wage paid to participants and the total amount spent on participant compensation?
    \answerNA{}
\end{enumerate}

\end{enumerate}

%%%%%%%%%%%%%%%%%%%%%%%%%%%%%%%%%%%%%%%%%%%%%%%%%%%%%%%%%%%%
\newpage

\appendix
\addcontentsline{toc}{section}{Appendix} % Add the appendix text to the document TOC
\part{Appendix} % Start the appendix part
\parttoc % Insert the appendix TOC
% \tableofcontents

% \todo{
% \begin{itemize}
%     \item Returnability and ergodic operators
%     \item Batching operation \answerYes{Done}
%     \item change from optimistic to union  \answerYes{Done}
%     \item submodularity and monotocity proof \answerYes{Done} Ask about $\density \geq 0$
%     \item change coverage objective to the linear definition \answerYes{Done}
%     \item multi-agent goose proof them \answerYes{Done}
%     \item restatable theorem env \answerYes{Done}
%     \item submodularity rate eqn proof \answerYes{Done}
%     \item clear up things 
%     \item recommendation rule proof \answerYes{Done}
%     \item \safemac \answerYes{Done}
%     \item \macopt \answerYes{Done}
%     \item experimental results, details, cite license as well \answerYes{Done}
% \end{itemize}}

\newpage

\twocolumn
\section{Definitions}
\label{Apx: definitions}
\setcounter{equation}{7}
\subsection{Notations}

% \begin{table}[htbp]\caption{Notations}
% \begin{center}% used the environment to augment the vertical space
% between the caption and the table
\setlength{\tabcolsep}{2.5pt}
\begin{xtabular}{lcp{0.75\linewidth}}
% \toprule
% \multicolumn{3}{c}{}\\

&&\underline{\textbf{Problem Formulation}}\\
% \multicolumn{3}{c}{}\\
$\Obj$&$\triangleq$&Submodular function, $\Obj : 2^V \to \R$\\
$\Domain$&$\triangleq$&Domain \\
$\PtInDomain$&$\triangleq$&An element in the domain $\Domain$\\
$\!\!\Objfunc{\cdot}{\cdot}{\cdot}\!\!$&$\triangleq$&Coverage objective defined in \cref{eqn: disk-coverage}\\
$i $&$\triangleq$&Agent index\\
$\density$&$\triangleq$&Density function, $\density:\Domain\to \R$\\  
$\constrain$&$\triangleq$&Constraint function, $\constrain:\Domain\to \R$\\
$\Discat[i]$&$\triangleq$&Sensing region around agent $i$\\
$\Discat[1:i] $&$\triangleq$&$\cup_{j =1}^{i} \Discat[j]$, union of sensing regions of agents $1:i$\\
$\Discat[i-] $&$\triangleq$&$\Discat[i]\setminus \Discat[1:i-1]$, region occupied by agent $i$, but not by $1:i-1$ agents\\
$\tlDiscat[i] $&$\triangleq$&Sensing region occupied by greedy optimal location of agent $i$\\
$\tlDiscat[i-] $&$\triangleq$&$\tlDiscat[i]\setminus \Discat[1:i-1]$ \\
$\DiskCoverageRatio $&$\triangleq$&$\max_i |\Discat[i]|/|\Domain| \leq 1$, Maximum fraction of area covered by an agent \\
$\numOfAgents$&$\triangleq$&Total number of agents\\
&&\underline{\textbf{Batch Operation}}\\
$\batch$&$\triangleq$&A batch of agents, $\{1,2 \hdots |\batch|\}$ \\
$\BatchColl{t}'(i) $&$\triangleq$& $\{ j  \in [\numOfAgents] | \uniSet[,i]{t}  \cap  \uniSet[,j]{t}  \neq  \emptyset \}$, agents connected in union set with agent $i$ \\ % Set of agents having connected union set with agent $i$. \\
$\BatchColl{t} $&$\triangleq$& $\bigcup\nolimits_{i \in [\numOfAgents]}  \BatchColl{t}'(i)$. Collection of batches sharing the union set. \\
$\BatchColl{} $&$\triangleq$&Collection of batches sharing the largest reachable set ($\Rbareps{}{\LocAgents^{\batch}_0}$)\\
$\BatchColl{t}^p $&$\triangleq$&Collection of batches sharing the pessimistic set \\
&&\underline{\textbf{$\LocAgents$ Notations}}\\
% \multicolumn{3}{c}{}\\
$\LocAgent[i]{t} $&$\triangleq$&Planned location of agent $i$ at time $t$\\
$\LocAgent[g,i]{t} $&$\triangleq$&Goal of agent $i$ at time $t$, defined by \cref{alg: coverage-phase} and \cref{alg: exploration-phase} in \cref{alg: safemac}\\
$\tilde{\LocAgent[]{}}^{i} $&$\triangleq$&Greedy location of agent $i$, \cref{eqn: tilde-defined}\\
$\LocAgents_t $&$\triangleq$&$\cup_{i \in [\numOfAgents]} \{ \LocAgent[i]{t} \}$, A set of agents at time $t$\\
$\LocAgents^{\batch}_t $&$\triangleq$&$\cup_{i \in \batch} \{ \LocAgent[i]{t} \}$, Agents in $\batch$ at time $t$\\
$\LocAgents^{\batch}_{\star} $&$\triangleq$&Optimal location of batch $\batch$ agents\\
$\LocAgents_{\star} $&$\triangleq$&$\cup_{\batch \in \BatchColl{}} \LocAgents^{\batch}_{\star}$  \\  
$\LocAgents^{1:i} $&$\triangleq$&A set of agents $1$ to $i$\\  
$\LocAgent[g, 1:\numOfAgents]{1:T} $&$\triangleq$&A set of $1:\numOfAgents$ agents' goal locations up to time $T$ \\
% \multicolumn{3}{c}{}}\\
% \multicolumn{3}{c}{}\\
& & \underline{\textbf{Density $(\density)$ and Constraint $(\constrain)$ GP}}\\
% Constraint
$\lbconst[t] $&$\triangleq$& LCB of the constraint at time $t$\\
$\ubconst[t] $&$\triangleq$& UCB of the constraint at time $t$\\
$\betaconst[t] $&$\triangleq$&Scaling, defined as per \citep{beta-chowdhury17a}\\
$\LipConst $&$\triangleq$&Lipschitz constant\\
$\epsconst $&$\triangleq$&Statistical confidence up to which constraint function $\constrain$ is learnt\\
$d(\PtInDomain,\zDeci) $&$\triangleq$&Distance metric\\
$\sigma_{\constrain} $&$\triangleq$&Standard deviation of constraint observations noise\\ 
$\sigconst[t] $&$\triangleq$&Posterior standard deviation of $\constrain$ GP\\
$B_{\constrain}$&$\triangleq$&Norm bound of $\constrain$, $\|\constrain\|_{\kernelfunc^{\constrain}} \leq B_{\constrain}$\\ 
$\noise_{\constrain} $&$\triangleq$&Noise in constraint observations\\ 

%Density
$\lbdensity[t] $&$\triangleq$& LCB of the density at time $t$\\
$\updensity[t]  $&$\triangleq$& UCB of the density at time $t$\\
$\betadensity[t] $&$\triangleq$&Scaling, defined as per \citep{beta-chowdhury17a}\\
$\sumMaxWidth{t} $&$\triangleq$&$\sum_{i=1}^{\numOfAgents} u^\density_{t-1}(\LocAgent[g, i]{t})- l^\density_{t-1}(\LocAgent[g, i]{t})$, sum of highest uncertainty below disks\\
$\epsdensity $&$\triangleq$&Accuracy threshold for learning the density, $\sumMaxWidth{} \leq \epsdensity$\\
$\sigdensity[t] $&$\triangleq$&Posterior standard deviation of $\density$ GP\\
$\sigma_{\density} $&$\triangleq$&Standard deviation of density observations noise\\ 
$B_{\density} $&$\triangleq$&Norm bound of $\density$, $\|\density\|_{\kernelfunc^{\density}} \leq B_{\density}$\\ 
$\delta $&$\triangleq$&$\in(0,1)$ for high probability \\ 
$H(y_A) $&$\triangleq$&Shannon entropy\\ 
$I(y_A;\density) $&$\triangleq$&$H(y_A) - H(y_A | \density)$, Information gain\\ 
$\gamma $&$\triangleq$&Information capacity\\ 
$\gammadensity{\numOfAgents T} $&$\triangleq$&$\sup\nolimits_{A \subset V} \! I(Y_{A};\density)$, A is set of $\numOfAgents T$ obs. $ \gammadensity{\numOfAgents T} \! \coloneqq \gamma_{\numOfAgents \Tdensity}$,\! $\density$ is clear in $T$. \\ 
$\gammaconst{\numOfAgents T} $&$\triangleq$&$\sup\nolimits_{A \subset V} \! I(Y_{A};\constrain)$, A is set of $\numOfAgents T$ obs. $ \gammaconst{\numOfAgents T} \! \coloneqq \gamma_{\numOfAgents \Tconst}$,\! $\constrain$ is clear in $T$.\\ 
$\trace $&$\triangleq$&Trace of a Matrix\\ 
$\densitykernel $&$\triangleq$&Posterior kernel matrix with $\density$ obs.\\ 
$\lambda_{i,t} $&$\triangleq$&Eigenvalue of the kernel matrix\\ 
$\noise_{\density} $&$\triangleq$&Noise in the density observations\\
% $B_{\constrain}$&$\triangleq$&Density\\ 
% $\delta $&$\triangleq$&Density\\ 
&&\underline{\textbf{Time}}\\
% \multicolumn{3}{c}{}\\
$t $&$\triangleq$&Any round of the algorithm\\
$T $&$\triangleq$&Algorithm termination time\\
$\tconst^{\star} $&$\triangleq$&Maximum number of $\constrain$ observations\\
$\tdensity^{\star} $&$\triangleq$&Maximum number of $\density$ observations\\
${\tdensity^{\star}}^1 $&$\triangleq$&Maximum number of density observations for the first coverage phase\\  
$\deltime{n} $&$\triangleq$&Maximum number of density obs. from $(n-1)^{th}$ to $n^{th}$ coverage phase\\ 
$\smdeltime{n} $&$\triangleq$&Number of density obs. from $(n-1)^{th}$ to $n^{th}$ coverage phase \\ 
$\tdensity^{n} $&$\triangleq$&Number of density observations till $n^{th}$ coverage phase \\
\end{xtabular}
\onecolumn

\begin{xtabular}{lcp{0.75\linewidth}}
\multicolumn{3}{c}{\underline{\textbf{\goose and Safe Expansion}}}\\
% \multicolumn{3}{c}{}\\
$\pessiOper{}{S} $&$\triangleq$&pessimistic operator $\{ \PtInDomain \in \Domain, | \exists \zDeci \in S: \lbconst[t](\zDeci) - \LipConst d(\PtInDomain,\zDeci) \geq 0 \}$ \\  
$\optiOper{\epsconst}{S} $&$\triangleq$&optimistic operator $\{ \PtInDomain \in \Domain, | \exists \zDeci \in S: \ubconst[t](\zDeci) - \epsconst - \LipConst d(\PtInDomain,\zDeci) \geq 0 \}$ \\
$\tilPessiOper{}{\cdot} $&$\triangleq$&Pessimistic expansion operator\\
$\tilOptiOper{}{\cdot} $&$\triangleq$&Optimistic expansion operator\\  
$\Rbareps{}{\{ \LocAgent[i]{0} \}}$&$\triangleq$&Maximum safely reachable set up to $\epsconst$, \cref{eqn: rbar-eps} \\
$\Rbareps{}{\LocAgents^{\batch}_0}$&$\triangleq$&$\cup_{i\in\batch}\Rbareps{}{\{ \LocAgent[i]{0} \} }$\\
$\pessiSet[,i]{t} $&$\triangleq$&Pessimistic set of agent $i$, $\tilPessiOper{}{\pessiSet[,i]{t-1}}$\\  
$\pessiSet[,\batch]{t} $&$\triangleq$&$\cup_{i \in \batch} \pessiSet[,i]{t}$\\  
$\pessiSet[]{t} $&$\triangleq$&Pessimistic set of all $\numOfAgents$ agents\\  
$\optiSet[,i]{t} $&$\triangleq$&Optimistic set of agent $i$, $\tilOptiOper{\epsconst}{\pessiSet[,i]{t-1}}$\\ 
$\optiSet[,\batch]{t} $&$\triangleq$&$\cup_{i \in \batch} \optiSet[,i]{t}$\\  
$\optiSet[]{t} $&$\triangleq$&Optimistic set of all $\numOfAgents$ agents\\  
$\uniSet[,i]{t} $&$\triangleq$&Union set, $\optiSet[,i]{t} \cup \pessiSet[,i]{t}$\\ 
$\uniSet[,\batch]{t} $&$\triangleq$&$\cup_{i \in \batch} \uniSet[,i]{t}$\\  
$\uniSet[]{t} $&$\triangleq$&Union set of all $\numOfAgents$ agents\\

$\Roperator{safe}{\epsconst}{S} $&$\triangleq$&True safety constraint operator, \cref{eqn: safety-constraint-operator} \\ 
$\Roperator{reach}{n}{S} $&$\triangleq$&n step reachability in the graph, \cref{eqn: n-reach}\\  
$\Rtiloperator{reach}{}{S} $&$\triangleq$&$\lim_{n \to \infty}\Roperator{reach}{n}{S}$\\  
% $\Roperator{ergodic}{}{\pessiOper{}{S},S} $&$\triangleq$&\todo{remove}\\
$\Roperator{n}{\epsconst}{S} $&$\triangleq$&n step safely reachable set in the graph, \cref{eqn: rbar-eps}\\  
$\Rbareps{}{S} $&$\triangleq$&$\lim_{n \to \infty} \Roperator{ n}{\epsconst}{S} $\\ 

$W^{\epsconst}_t $&$\triangleq$&Set of locations whose safety is not $\epsconst$-accurate, \cref{alg:se}\\  
$G^{\epsconst}_t(\alpha) $&$\triangleq$&A set of potential immediate expanders, \cref{alg:se}\\  
$p $&$\triangleq$&Priority, \cref{alg:se}\\ 
$h(\PtInDomain) $&$\triangleq$&Heuristic function, \cref{alg:se} \\ 
$A_t(\alpha) $&$\triangleq$&Subset of locations with equal priority, \cref{alg:se}\\ 

\\
\multicolumn{3}{c}{\underline{\textbf{Regret}}}\\
% \multicolumn{3}{c}{}\\
$\Obj(\LocAgents)$&$\triangleq$&$\Objfunc{\LocAgents}{\density}{\Domain}$, short notation when $\density$ and $\Domain$ are obvious \\
$\delgain{\LocAgent[i]{}|\LocAgents^{1:i-1}}{\density}{\Domain}$&$\triangleq$&Marginal coverage gain by agent $i$, \cref{eqn: marginal-gain} \\
$\Delta(\LocAgent[i]{}|\LocAgents^{1:i-1})$&$\triangleq$&$\delgain{\LocAgent[i]{}|\LocAgents^{1:i-1}}{\density}{\Domain}$, short notation when $\density$ and $\Domain$ are obvious \\
$\actualRegret(T) $&$\triangleq$&Actual regret in unconstrained case, \cref{eqn: reg-def}\\ 
$\cumOPTloc $&$\triangleq$&Per agent cumulative optimal gain, \cref{eqn: tilde-defined-opt-loc}\\ 
$\LocReg(T) $&$\triangleq$&Per agent regret, \cref{def: local_reg} \\ 
$\cumOPT $&$\triangleq$&$\sum_{t=1}^T \Obj(\LocAgents_{\star})$\\ 
$\simActReg{} $&$\triangleq$&Simple actual regret, constrained case, \cref{def: sim-act-reg-loc-reg}\\ 
% $\simLocReg{t} $&$\triangleq$&$\sum_{t=1}^T \Big( \max_{\LocAgent[i]{t}} \Obj(\LocAgents^{1:i-1}_t \cup \{\LocAgent[i]{t}\}) - \Obj(\LocAgents^{1:i-1}_t) \Big)$\\ 
$\simOptiReg{t} $&$\triangleq$&Simple actual regret in union set, constrained case, \cref{def: sim-act-reg-loc-reg}\\ 
$\simReg{} $&$\triangleq$&Simple per agent regret, constrained case, \cref{def: sim-act-reg-loc-reg}\\ 

$\optiRegret(T) $&$\triangleq$&Cumulative actual regret, \cref{eqn: cum-reg-const-opti}\\ 
$\optiLocReg(T) $&$\triangleq$&Sum of cumulative per agent regret, \cref{eqn: cum-reg-const-opti}\\ 

% \multicolumn{3}{c}{}\\
% \multicolumn{3}{c}{\underline{Decision Variables}}\\
% \multicolumn{3}{c}{}\\
% $y_f$&$=$&\(\left\{\begin{array}{rl}
% 1, &\text{if Supplier located at site $f$ is open} \\
% 0, &\text{otherwise} \end{array} \right.\)\\
% \bottomrule
\end{xtabular}
% \end{center}
% \label{tab:TableOfNotationForMyResearch}
% \end{table}
\onecolumn
\subsection{\goose operators}
We denote with $\G = (\Domain, \edges{})$ the undirected graph describing the dependency among locations, 
$\Domain$ indicates the vertices of the graph, i.e., the state space of the problem and $\edges{} \subseteq \Domain \times \Domain$ denotes the edges. In our setting, there are $\numOfAgents$ identical agents having the same transition dynamics. Each agent can have a separate $\Rtilde{}{\{\LocAgent[i]{0}\}}$. 

The baseline as per true safety constraint operator:
\begin{align*}
    \Roperator{safe}{\epsconst}{S} = S \cup \{ \PtInDomain \in \Domain \backslash S, | \exists \zDeci \in S: \constrain(\zDeci) - \epsconst - \LipConst d(\PtInDomain,\zDeci) \geq 0 \} \numberthis \label{eqn: safety-constraint-operator}
\end{align*}
Now, we define reachability operator as all the locations that can be reached starting from set S.
\begin{align*}
    \Roperator{reach}{}{S} &= S \cup \{ \PtInDomain \in \Domain \backslash S, | \exists \zDeci \in S: (\zDeci,\PtInDomain) \in \edges{} \}, \\
    \Roperator{reach}{n}{S} &= \Roperator{reach}{n}{\Roperator{reach}{n-1}{S}} ~\text{with}~ \Roperator{reach}{1}{S} = \Roperator{reach}{}{S} \numberthis \label{eqn: n-reach}\\
    \Rtiloperator{reach}{}{S} &= \lim_{n \to \infty}\Roperator{reach}{n}{S}, \numberthis \label{eqn: inf-reach}
\end{align*}
For defining $\Rbareps{}{S}$,
\begin{align*}
    \Roperator{}{\epsconst}{S} &= \Roperator{safe}{\epsconst}{S} \cap \Rtiloperator{reach}{}{S} \\
      \Roperator{n}{\epsconst}{S} &= \Roperator{}{\epsconst}{\Roperator{n-1}{\epsconst}{S}} ~~\text{with}~ \Roperator{1}{\epsconst}{S} = \Roperator{}{\epsconst}{S}  \numberthis \label{eqn: n-rbar-eps} \\
    \Rbareps{}{S} &= \lim_{n \to \infty} \Roperator{ n}{\epsconst}{S} \numberthis \label{eqn: rbar-eps}
\end{align*}
% \manish{ Need to verify this, another possiblity is:  (This aligns more with \goose)
% Now for defining $\Rtilde{}{.}$,
% \begin{align*}
%     \Roperator{safe, n}{\epsconst}{S} &= \Roperator{safe}{\epsconst}{S}(\Roperator{safe, n-1}{\epsconst}{S}) ~~\text{with}~ \Roperator{safe,1}{\epsconst}{S} = \Roperator{safe}{\epsconst}{S}\\
%     \Rtiloperator{safe}{\epsconst}{S} &= \lim_{n \to \infty} \Roperator{safe, n}{\epsconst}{S}\\
%     \Rbareps{}{S} &= \Rtiloperator{safe}{\epsconst}{S} \cap \Rtiloperator{reach}{}{S}
% \end{align*}
% In my opinion this definition would not include the safely reachable region "region 2" whose safety is infered from safe unreachable region "region 1".}
Optimistic and pessimistic constrain satisfaction operators:
\begin{align*}
    \optiOper{\epsconst}{S} &= \{ \PtInDomain \in \Domain, | \exists \zDeci \in S: \ubconst[t](\zDeci) - \epsconst - \LipConst d(\PtInDomain,\zDeci) \geq 0 \} \\
    \pessiOper{\epsconst}{S} &= \{ \PtInDomain \in \Domain, | \exists \zDeci \in S: \lbconst[t](\zDeci) - \epsconst - \LipConst d(\PtInDomain,\zDeci) \geq 0 \}
\end{align*} 
% \manish{With the above definition of the optimistic set, if after a measuremnt in the pessimistic set if the uncertainity drops such that $\ubconst{} < \epsconst + thr$, then we optimistic set will have 2 parts or unreachable parts inside it.}
% \manish{Should we define above operators using $S \cup \{ . . .  \}$, similar to $\Roperator{safe}{\epsconst}{S}$ ? No since for optimistic set we need it to shrink and we didn;t want union there. and for pessimistic set since eps = 0, it does not really matter.} \\
In this section, for simplicity, we have considered an undirected graph. This results in the same reachability and returnability operators since the edges are bidirectional. The extension to the directed graph is easy by using the reachability, the returnability and the ergodic operator. (Appendix A of \citet{turchetta2019safe} does it for the directed graph, so we did not repeat it here) 

The optimistic and pessimistic expansion operators are given by,
\begin{align*}
    \OptiOperReach{\epsconst}{S} &= \optiOper{\epsconst}{S} \cap  \Rtiloperator{reach}{}{S} \\
    \OptiOperReach{\epsconst, n}{S} &= \OptiOperReach{\epsconst}{\OptiOperReach{\epsconst, n-1}{S}} ~\text{with}~  \OptiOperReach{\epsconst,1}{S} = \OptiOperReach{\epsconst}{S} \\
    \tilOptiOper{\epsconst}{S} &= \lim_{n \to \infty} \OptiOperReach{\epsconst, n}{S}
    % \barOptiOper{\epsconst}{S} &= \tilOptiOper{\epsconst}{S} \cap \Rtiloperator{reach}{}{S}
\end{align*}
% \manish{Equivalent form (Take it down later on:)
% \begin{align*}
%     \OptiOperReach{\epsconst,1}{S} &= \optiOper{\epsconst}{S} \cap  \Rtiloperator{reach}{}{S} \\
%     \OptiOperReach{\epsconst, n}{S} &= \optiOper{\epsconst}{\OptiOperReach{\epsconst, n-1}{S}} \cap \Rtiloperator{reach}{}{\OptiOperReach{\epsconst, n-1}{S}} \\
%     \tilOptiOper{\epsconst}{S} &= \lim_{n \to \infty} \OptiOperReach{\epsconst, n}{S}
%     % \barOptiOper{\epsconst}{S} &= \tilOptiOper{\epsconst}{S} \cap \Rtiloperator{reach}{}{S}
% \end{align*}
% Using the definition,
% \begin{align*}
%       \OptiOperReach{\epsconst, 2}{S} &= \optiOper{\epsconst}{\optiOper{\epsconst}{S} \cap  \Rtiloperator{reach}{}{S}} \cap \Rtiloperator{reach}{}{\optiOper{\epsconst}{S} \cap  \Rtiloperator{reach}{}{S}} \\
%   &= \optiOper{\epsconst}{\optiOper{\epsconst}{S} \cap  \Rtiloperator{reach}{}{S}} \cap \Rtiloperator{reach}{}{s} \\
% \end{align*}
% Using the above definition,
% \begin{align*}
%         \OptiOperReach{\epsconst, 2}{S} &= \OptiOperReach{\epsconst}{\OptiOperReach{\epsconst, 1}{S}} \\
%         &= \optiOper{\epsconst}{\optiOper{\epsconst}{S} \cap  \Rtiloperator{reach}{}{S}} \cap \Rtiloperator{reach}{}{\optiOper{\epsconst}{S} \cap  \Rtiloperator{reach}{}{S}} \\
%   &= \optiOper{\epsconst}{\optiOper{\epsconst}{S} \cap  \Rtiloperator{reach}{}{S}} \cap \Rtiloperator{reach}{}{s} \\
% \end{align*}
% }
Pessimistic expansion operator
\begin{align*}
    \PessiOperReach{\epsconst}{S} &= \pessiOper{\epsconst}{S} \cap  \Rtiloperator{reach}{}{S} \\
    \PessiOperReach{\epsconst, n}{S} &= \PessiOperReach{\epsconst}{\PessiOperReach{\epsconst, n-1}{S}} ~\text{with}~  \PessiOperReach{\epsconst,1}{S} = \PessiOperReach{\epsconst}{S} \\
    \tilPessiOper{\epsconst}{S} &= \lim_{n \to \infty} \PessiOperReach{\epsconst, n}{S} 
    % \barOptiOper{\epsconst}{S} &= \tilOptiOper{\epsconst}{S} \cap \Rtiloperator{reach}{}{S}
\end{align*}
This gives the optimistically and pessimistically, safe and reachable set respectively as:
\begin{align*}
    \optiSet{t}{} &= \tilOptiOper{\epsconst}{\pessiSet{t-1}}\\
    \pessiSet{t}{} &= \tilPessiOper{0}{\pessiSet{t-1}}
\end{align*}

Now in our setting with $\numOfAgents$ agents, we denote with $\optiSet[,i]{t}$ and $\pessiSet[,i]{t}$, the optimistic and the pessimistic set respectively of agent $i$. The union set for any agent $i$ is defined as, 
\begin{align*}
    \uniSet[,i]{t} \coloneqq \optiSet[,i]{t} \cup \pessiSet[,i]{t} \numberthis
\end{align*}

\subsection{Batching operation}

\begin{figure}
% \hspace{-40.00mm}
\centering
\scalebox{0.65}{
\tikzset{every picture/.style={line width=0.75pt}} %set default line width to 0.75pt        

\begin{tikzpicture}[x=0.75pt,y=0.75pt,scale=0.4,yscale=-1,xscale=1]
%uncomment if require: \path (0,531); %set diagram left start at 0, and has height of 531

%Shape: Polygon Curved [id:ds4226761684384539] 
\draw  [fill={rgb, 255:red, 191; green, 144; blue, 144 }  ,fill opacity=1 ] (90.63,45.35) .. controls (120.05,27.39) and (707.8,61.23) .. (659.9,93.86) .. controls (612,126.5) and (673.19,321.24) .. (632,398.5) .. controls (590.81,475.76) and (37.67,451.41) .. (8.25,397.51) .. controls (-21.17,343.61) and (61.21,63.32) .. (90.63,45.35) -- cycle ;
%Shape: Polygon Curved [id:ds5555723749272381] 
\draw  [draw opacity=0][fill={rgb, 255:red, 31; green, 119; blue, 180 }  ,fill opacity=0.85 ] (35,260.5) .. controls (64.42,242.53) and (167.2,173.23) .. (216,210.5) .. controls (264.8,247.77) and (461.58,346.6) .. (491,400.5) .. controls (520.42,454.4) and (67.42,444.4) .. (38,390.5) .. controls (8.58,336.6) and (5.58,278.47) .. (35,260.5) -- cycle ;
%Shape: Polygon Curved [id:ds2411793319931821] 
\draw  [draw opacity=0][fill={rgb, 255:red, 31; green, 119; blue, 180 }  ,fill opacity=1 ] (172,75.5) .. controls (201.42,57.53) and (498.99,69.57) .. (549,105.5) .. controls (599.01,141.43) and (602.58,147.6) .. (632,201.5) .. controls (661.42,255.4) and (333.42,283.4) .. (304,229.5) .. controls (274.58,175.6) and (142.58,93.47) .. (172,75.5) -- cycle ;
%Shape: Donut [id:dp4915194280047097] 
\draw  [draw opacity=0][fill={rgb, 255:red, 188; green, 189; blue, 34 }  ,fill opacity=0.8 ,even odd rule] (342.45,154.25) .. controls (342.45,149.94) and (345.61,146.45) .. (349.5,146.45) .. controls (353.39,146.45) and (356.55,149.94) .. (356.55,154.25) .. controls (356.55,158.56) and (353.39,162.05) .. (349.5,162.05) .. controls (345.61,162.05) and (342.45,158.56) .. (342.45,154.25)(272,154.25) .. controls (272,111.03) and (306.7,76) .. (349.5,76) .. controls (392.3,76) and (427,111.03) .. (427,154.25) .. controls (427,197.47) and (392.3,232.5) .. (349.5,232.5) .. controls (306.7,232.5) and (272,197.47) .. (272,154.25) ;
%Shape: Donut [id:dp24534422096458952] 
\draw  [draw opacity=0][fill={rgb, 255:red, 103; green, 189; blue, 183 }  ,fill opacity=0.96 ,even odd rule] (497.45,174.25) .. controls (497.45,169.94) and (500.61,166.45) .. (504.5,166.45) .. controls (508.39,166.45) and (511.55,169.94) .. (511.55,174.25) .. controls (511.55,178.56) and (508.39,182.05) .. (504.5,182.05) .. controls (500.61,182.05) and (497.45,178.56) .. (497.45,174.25)(427,174.25) .. controls (427,131.03) and (461.7,96) .. (504.5,96) .. controls (547.3,96) and (582,131.03) .. (582,174.25) .. controls (582,217.47) and (547.3,252.5) .. (504.5,252.5) .. controls (461.7,252.5) and (427,217.47) .. (427,174.25) ;
%Shape: Donut [id:dp03985920976038271] 
\draw  [draw opacity=0][fill={rgb, 255:red, 255; green, 127; blue, 14 }  ,fill opacity=0.76 ,even odd rule] (105.45,338.25) .. controls (105.45,333.94) and (108.61,330.45) .. (112.5,330.45) .. controls (116.39,330.45) and (119.55,333.94) .. (119.55,338.25) .. controls (119.55,342.56) and (116.39,346.05) .. (112.5,346.05) .. controls (108.61,346.05) and (105.45,342.56) .. (105.45,338.25)(35,338.25) .. controls (35,295.03) and (69.7,260) .. (112.5,260) .. controls (155.3,260) and (190,295.03) .. (190,338.25) .. controls (190,381.47) and (155.3,416.5) .. (112.5,416.5) .. controls (69.7,416.5) and (35,381.47) .. (35,338.25) ;
%Shape: Donut [id:dp03337219191934371] 
\draw  [draw opacity=0][fill={rgb, 255:red, 227; green, 218; blue, 132 }  ,fill opacity=1 ,even odd rule] (368.45,357.25) .. controls (368.45,352.94) and (371.61,349.45) .. (375.5,349.45) .. controls (379.39,349.45) and (382.55,352.94) .. (382.55,357.25) .. controls (382.55,361.56) and (379.39,365.05) .. (375.5,365.05) .. controls (371.61,365.05) and (368.45,361.56) .. (368.45,357.25)(298,357.25) .. controls (298,314.03) and (332.7,279) .. (375.5,279) .. controls (418.3,279) and (453,314.03) .. (453,357.25) .. controls (453,400.47) and (418.3,435.5) .. (375.5,435.5) .. controls (332.7,435.5) and (298,400.47) .. (298,357.25) ;
%Shape: Polygon Curved [id:ds4516938259422951] 
\draw  [color={rgb, 255:red, 255; green, 255; blue, 255 }  ,draw opacity=0.01 ][fill={rgb, 255:red, 191; green, 144; blue, 144 }  ,fill opacity=1 ] (330,269.5) .. controls (349,258.5) and (538,289.5) .. (518,309.5) .. controls (498,329.5) and (484,355.5) .. (504,385.5) .. controls (524,415.5) and (472,378.5) .. (443,354.5) .. controls (414,330.5) and (311,280.5) .. (330,269.5) -- cycle ;
%Shape: Donut [id:dp45776463211414176] 
\draw  [draw opacity=0][fill={rgb, 255:red, 82; green, 142; blue, 82 }  ,fill opacity=1 ,even odd rule] (226.45,337.25) .. controls (226.45,332.94) and (229.61,329.45) .. (233.5,329.45) .. controls (237.39,329.45) and (240.55,332.94) .. (240.55,337.25) .. controls (240.55,341.56) and (237.39,345.05) .. (233.5,345.05) .. controls (229.61,345.05) and (226.45,341.56) .. (226.45,337.25)(156,337.25) .. controls (156,294.03) and (190.7,259) .. (233.5,259) .. controls (276.3,259) and (311,294.03) .. (311,337.25) .. controls (311,380.47) and (276.3,415.5) .. (233.5,415.5) .. controls (190.7,415.5) and (156,380.47) .. (156,337.25) ;

% Text Node
\draw (84,305.4) node [anchor=north west][inner sep=0.75pt]  [font=\large]  {$x^{2}$};
\draw (75,355.4) node [anchor=north west][inner sep=0.75pt]  [font=\Large]  {$\Discat[2-]$};
% Text Node
\draw (205,304.4) node [anchor=north west][inner sep=0.75pt]  [font=\large]  {$x^1$};
\draw (215,354.4) node [anchor=north west][inner sep=0.75pt]  [font=\Large]  {$\Discat[1]$};
% Text Node
\draw (348,325.4) node [anchor=north west][inner sep=0.75pt]  [font=\large]  {$x^{3}$};
\draw (348,375.4) node [anchor=north west][inner sep=0.75pt]  [font=\Large]  {$\Discat[3-]$};
% Text Node
\draw (322,123.4) node [anchor=north west][inner sep=0.75pt]  [font=\large]  {$x^{4}$};
\draw (322,173.4) node [anchor=north west][inner sep=0.75pt]  [font=\Large]  {$\Discat[4]$};
% Text Node
\draw (478,141.4) node [anchor=north west][inner sep=0.75pt]  [font=\large]  {$x^{5}$};
\draw (478,203.4) node [anchor=north west][inner sep=0.75pt]  [font=\Large]  {$\Discat[5]$};
% Text Node
\draw (68,128.4) node [anchor=north west][inner sep=0.75pt]  [font=\LARGE]  {$V$};
% Text Node
\draw (113,209.4) node [anchor=north west][inner sep=0.75pt]  [font=\normalsize]  {$\overline{R}_{0}\!(\!X_{0}^{B_{1}})$};
% Text Node
\draw (178,68.4) node [anchor=north west][inner sep=0.75pt]  [font=\normalsize]  {$\overline{R}_{0}\!(\!X_{0}^{B_{2}})$};
% Text Node
\draw (512.61,312.72) node  [font=\normalsize]  {$ \begin{array}{l}
\!\mathcal{B}  \! = \{1,2,3\} ,\\
\ \ \ \ \ \ \ \ \ \{4,5\}\!\}
\end{array}$};

\end{tikzpicture}}
\label{fig: disk}
\caption{Disconnected safe regions: Agents are partitioned into two batches. Agent 1 covers $\Discat[1]$ (green), 2 covers $\Discat[2-]$ (orange) and 3 covers $\Discat[3-]$ (yellow).
} 
\vspace{-4.00mm}
\end{figure}
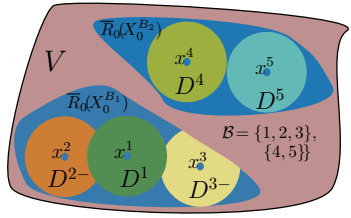

For a set of agents, we partition them in batches, such that each batch $\batch$ contains the agents that share at least a node in the union set. The total collection of batches, $\BatchColl{}$, is defined as, 
\begin{align*}
\BatchColl{t} = \bigcup\nolimits_{i \in [\numOfAgents]} \BatchColl{t}'(i) ~~where~~\BatchColl{t}'(i) = \{ j \in [\numOfAgents] \, | \, \uniSet[,i]{t} \cap \uniSet[,j]{t} \neq \emptyset \} \numberthis \label{eqn: batching}
\end{align*}

Analogous to $\BatchColl{t}$, we define $\BatchColl{t}^p$ (or $\BatchColl{}$) as collection of batches where any $\batch \in \BatchColl{t}^p$ (or $\BatchColl{}$) contains agents which are topologically connected in the pessimistic (or maximum safely reachable) set. Precisely,
%and $\pessiSet[,\batch]{t} \coloneqq \cup_{i \in \batch} \pessiSet[,i]{t}$
\begin{align*}
\BatchColl{t}^p &= \bigcup\nolimits_{i \in [\numOfAgents]} \BatchColl{t}'(i) ~~where~~\BatchColl{t}'(i) = \{ j \in [\numOfAgents] \, | \, \pessiSet[,i]{t} \cap \pessiSet[,j]{t} \neq \emptyset \} \numberthis \\
\BatchColl{} &= \bigcup\nolimits_{i \in [\numOfAgents]} \BatchColl{}'(i) ~~where~~\BatchColl{}'(i) = \{ j \in [\numOfAgents] \, | \, \Rbareps{}{\{ \LocAgent[i]{0} \}} \cap \Rbareps{}{\{ \LocAgent[j]{0} \}}\neq \emptyset \} \numberthis
\end{align*}

% \begin{align*}
%     \BatchColl{t}'(i; A) &= \{ j \in A \, | \, \uniSet[,i]{t} \cap \uniSet[,j]{t} \neq \emptyset \} \tag{With duplicate batches}\\
%     \BatchColl{t}(A) &= \bigcup_{i \in A} \BatchColl{t}'(i; A) \numberthis \label{eqn: batching}
% \end{align*}

% By definition of the optimistic expansion and intersection operator, optimistic set of the agents are mutually exclusive, i.e, either the optimistic sets of any 2 agents are equal or they do not share any node in common. 
The resulting batch collection are mutually exclusive that is $\forall~\batch_1,\batch_2 \in \BatchColl{t}, \batch_1 \neq \batch_2, \batch_1 \cap \batch_2 = \emptyset$ and also, $\sum_{\batch \in \BatchColl{t}} |\batch| = \numOfAgents$.\\ 
% \textbf{Short notations:} $\BatchColl{t}:=\BatchColl{t}([\numOfAgents])$ and $\BatchColl{r}:=\BatchColl{}([\numOfAgents])$ with the $\Rbareps{}{.}$ instead of $\optiSet[,.]{t}$
For any batch $\batch$ we can define their combined union set, pessimistic set and the maximum safely reachable set as ,
\begin{align*}
    \uniSet[,\batch]{t} \coloneqq \cup_{i \in \batch} \uniSet[,i]{t}, ~~ \pessiSet[,\batch]{t} \coloneqq \cup_{i \in \batch} \pessiSet[,i]{t},  ~~ \Rbareps{}{\LocAgents^{\batch}_0}=\cup_{i\in\batch}\Rbareps{}{\{ \LocAgent[i]{0} \}}.
\end{align*}

\section{Disk Coverage as a submodular function}
\label{Apx: disk-submodular}
% A set function $\Obj: 2^{\Domain} \to \R$ is called \textit{submodular} if, $\forall A \subseteq B \subseteq V$ and $ \forall e \in V \backslash B $, we have, $$F(A \cup \{ e \}) - F(A) \geq  F(B \cup \{e \}) - F(B).$$
\textbf{Set functions} Function $\Obj : 2^V \to \R$ that assign each subset $A \subseteq V$ a value $\Obj(A)$.

\textbf{Discrete Derivative} For a set function $\Obj : 2^V \to R$ , $A \subseteq V$, and $e \in V$, let $\Delta_{F}(e|A) := F(A \cup \{ e \}) - F(A)$ is discrete derivative of $\Obj$ at $A$ with respect to $e$.

\textbf{Submodular functions} A function F(.) is a submodular if, $\forall A \subseteq B \subseteq V$ and $ \forall  e \in V \backslash B $
	\begin{align*}
F(A \cup \{ e \}) - F(A) &\geq  F(B \cup \{e \}) - F(B), \\
\Delta_{F}(e|A) &\geq \Delta_{F}(e|B).
\end{align*}
For the disk coverage function $F(A)$, defined in \cref{eqn: disk-coverage},
\begin{equation*}
    \Objfunc{\LocAgents}{\density}{\Domain} = \sum_{\LocAgent[i]{} \in \LocAgents} \sum_{\PtInDomain \in \Discat[i-]} \density(\PtInDomain) / \factor,
% \vspace{-1mm}
\end{equation*}

We can write marignal gain as,

\begin{align*}
   F(A \cup \{ e \}) - F(A) &= \sum_{\LocAgent[i]{} \in A \cup \{ e \}} \sum_{\PtInDomain \in \Discat[i-]} \density(\PtInDomain) / \factor - \sum_{\LocAgent[i]{} \in A } \sum_{\PtInDomain \in \Discat[i-]} \density(\PtInDomain) / \factor  \\
   &= \sum_{\LocAgent[i]{} \in A} \sum_{\PtInDomain \in \Discat[i-]} \density(\PtInDomain) / \factor + \sum_{\LocAgent[i]{} \in \{e\}} \sum_{\PtInDomain \in \Discat[i] \setminus \Discat[1:|A|] } \density(\PtInDomain) / \factor - \sum_{\LocAgent[i]{} \in A } \sum_{\PtInDomain \in \Discat[i-]} \density(\PtInDomain) / \factor \\
%  &=  \sum_{v \in \Reachable} \Bigg( \density(v) \min \Big \{k, \sum_{x^i \in A \cup \{ e \}} c(x^i , v) \Big\} - \density(v) \min \Big \{k, \sum_{x^i \in A} c(x^i , v) \Big\} \Bigg) \\
&=\sum_{\LocAgent[i]{} \in \{e\}} \sum_{\PtInDomain \in \Discat[i] \setminus \Discat[1:|A|] } \density(\PtInDomain) / \factor\\
&\geq \sum_{\LocAgent[i]{} \in \{e\}} \sum_{\PtInDomain \in \Discat[i] \setminus \Discat[1:|B|] } \density(\PtInDomain) / \factor \quad \textit{(Since, A $\subseteq$ B, $|\Discat[i] \setminus \Discat[1:|A|]| \geq |\Discat[i] \setminus \Discat[1:|B|]| $} \\
&= \sum_{\LocAgent[i]{} \in B \cup \{ e \}} \sum_{\PtInDomain \in \Discat[i-]} \density(\PtInDomain) / \factor - \sum_{\LocAgent[i]{} \in B } \sum_{\PtInDomain \in \Discat[i-]} \density(\PtInDomain) / \factor  \\ 
&= F(B \cup\{ e \}) - F(B) \\
               {\implies} F(A \cup \{ e \}) - F(A) &\geq  F(B \cup \{e \}) - F(B)
\end{align*}
This shows that the coverage function defined in \cref{eqn: disk-coverage} is a Submodular function. 

Monotonicity is directly implied by the definition of $F(A)$, as an additive function of $\density$. Since, $\density(\PtInDomain) \geq 0, \, \forall \PtInDomain \in  \Domain$ $\implies$ $F(A) \leq F(B),$ if $A \subseteq B$.

% \begin{align*}
%   F(A \cup \{ e \}) - F(A) &= \sum_{v \in \Reachable} c'(A  \cup \{ e \}, v) - \sum_{v \in \Reachable} c'(A , v) \\
%              &=  \sum_{v \in \Reachable} \Bigg( \density(v) \min \Big \{k, \sum_{x^i \in A \cup \{ e \}} c(x^i , v) \Big\} - \density(v) \min \Big \{k, \sum_{x^i \in A} c(x^i , v) \Big\} \Bigg) \\
%  &= \sum_{v \in (\{ e \}\backslash A)^r} \density(v)\\
%               &\geq \sum_{v \in (\{ e \}\backslash B)^r} \density(v) \quad \textit{(Since, A $\subset$ B, $|(\{ e \}\backslash A)^r| \geq |(\{ e \}\backslash B)^r| $} \\
%               &= \sum_{v \in V} \Bigg( \density(v) \min \Big \{k, \sum_{x^i \in B \cup\{ e \}} c(x^i , v) \Big\} - \density(v) \min \Big \{k, \sum_{x^i \in B} c(x^i , v) \Big\} \Bigg) \\ 
%               &= F(B \cup\{ e \}) - F(B) \\
%               {\implies} \Delta_{F}(e|A)&\geq \Delta_{F}(e|B)
% \end{align*}
% . Extension of the proof to the constraint case is straight forward.
\newpage
\section{Agent wise regret bound}
In this section, we upper bound the actual ("greedy") regret with the per agent regret in the unconstrained and the constrained case. The proof methodology to bound with per agent regret is motivated from \cite{YisongLSB}. We first define marginal gain and agent-wise regret. Then we give a proposition for the submodularity rate equation, which will be central to our lemmas. Finally, we bound the actual regret with the sum of per agent regret for unconstrained and then constrained case in %\cref{}

\textbf{Marginal coverage gain:}
\begin{align*}
    \delgain{\LocAgent[i]{t}|\LocAgents^{1:i-1}_t}{\density}{\Domain} &= \Objfunc{\LocAgents^{1:i-1}_t \cup \{\LocAgent[i]{t}\}}{\density}{\Domain} - \Objfunc{\LocAgents^{1:i-1}_t}{\density}{\Domain} \\
    &= \sum_{\LocAgent[i]{t} \in \LocAgents^{1:i}_t} \sum_{\PtInDomain \in \Discat[i-]_t} \density(\PtInDomain) / \factor - \sum_{\LocAgent[i]{t} \in \LocAgents^{1:i-1}_t} \sum_{\PtInDomain \in \Discat[i-]_t} \density(\PtInDomain) / \factor \\
    &= \sum_{\PtInDomain \in \Discat[i-]_t} \density(\PtInDomain) / \factor \label{eqn: marginal-gain}\numberthis 
\end{align*}
Using, $\LocAgents^{1:0} = \{ \emptyset \}, \Obj(\LocAgents^{1:0}) = 0$, it follows that, 
\begin{align*}
    \sum_{i=1}^\numOfAgents \delgain{\LocAgent[i]{t}|\LocAgents^{1:i-1}_t}{\density}{\Domain} = \Objfunc{\LocAgents^{1:\numOfAgents}_t}{\density}{\Domain} \numberthis \label{eqn: sum_delta}
\end{align*}

\textbf{Tilde Notations:} 
\begin{align*}
    \tilde{\LocAgent[]{}}^i_t = \argmax_{\LocAgent[i]{t}}  \delgain{\LocAgent[i]{t}|\LocAgents^{1:i-1}_t}{\density}{\Domain} \numberthis \label{eqn: tilde-defined}
\end{align*}

\begin{proposition}[\eq~(3-7), \citep{krause2014submodular}, Submodular rate equation] \label{lem: submodularity-rate}
For a monotone Submodular function $\Obj$ the following holds,
\begin{align*}
    \max_{\LocAgent[i]{}} \Obj(\LocAgents^{1:i-1} \cup \{\LocAgent[i]{}\}) - \Obj(\LocAgents^{1:i-1}) \geq \frac{\Obj(\LocAgents_{\star}) - \Obj(\LocAgents^{1:i-1})}{\numOfAgents}, \numberthis \label{eqn: submodularity-rate-eqn}
\end{align*}
where $\LocAgents^{1:i}$ is the set of $i$ agents being picked greedily and $\numOfAgents$ is the number of agents in $\LocAgents_{\star}$. 
\end{proposition}
\begin{proof} Let $\LocAgents_{\star} = \{ \LocAgent[1]{\star}, \hdots ,\LocAgent[\numOfAgents]{\star} \}$ 
\begin{align*}
    \Obj(\LocAgents_{\star}) &\leq \Obj(\LocAgents_{\star} \cup \LocAgents^{1:i-1}) \tag{With monotonicity of $\Obj$}\\
    &=\Obj(\LocAgents^{1:i-1}) + \sum_{j=1}^{\numOfAgents} \Delta(\LocAgent[j]{\star}|\LocAgents^{1:i-1} \cup \{\LocAgent[1]{\star}, \hdots ,\LocAgent[j-1]{\star} \}) \tag{Telescopic sum}\\
    &\leq \Obj(\LocAgents^{1:i-1}) + \sum_{\LocAgent[]{} \in \LocAgents_{\star}} \Delta(\LocAgent[]{}|\LocAgents^{1:i-1}) \tag{Follows by Submodularity of $\Obj$}\\
    &\leq \Obj(\LocAgents^{1:i-1}) + \sum_{\LocAgent[]{} \in \LocAgents_{\star}} ( \Obj(\LocAgents^{1:i}) - \Obj(\LocAgents^{1:i-1})) \tag{since, $\LocAgent[i]{}$ is added greedily to maximize $\Delta(\LocAgent[]{}|\LocAgents^{1:i-1})$}\\
    &\leq \Obj(\LocAgents^{1:i-1}) + \numOfAgents ( \Obj(\LocAgents^{1:i}) - \Obj(\LocAgents^{1:i-1})) \tag{$\numOfAgents$ agents in $\LocAgents_{\star}$} \\
    \implies & \frac{\Obj(\LocAgents_{\star}) - \Obj(\LocAgents^{1:i-1})}{\numOfAgents} \leq  \Obj(\LocAgents^{1:i}) - \Obj(\LocAgents^{1:i-1})
\end{align*}
The proposition follows directly since $\LocAgent[i]{}$ is added greedily to $\LocAgents^{1:i-1}$.
\end{proof}

% \subsection{Upper bound the greedy regret \texorpdfstring{$\actualRegret(T)$}{Lg} with individual agent-wise regret:}
\subsection{Unconstrained case} 

Note that for unconstrained case domain $\Domain$ and utility $\density$ is obvious, so for convenience we use short hand notation, i.e, $\Objfunc{\cdot}{\density}{\Domain} = \Obj(\cdot)$ and $\delgain{\cdot}{\density}{\Domain} = \Delta(\cdot)$. 

\mypar{Locally optimal gain} Let us define $\cumOPTloc$ as the local optimal coverage gained by agent $i$, given all the locations of agents $1:i-1$, formally given by,
\begin{align*}
    \cumOPTloc &= \sum_{t=1}^T \Big( \max_{\LocAgent[i]{t}} \Obj(\LocAgents^{1:i-1}_t \cup \{\LocAgent[i]{t}\}) - \Obj(\LocAgents^{1:i-1}_t) \Big) = \sum_{t=1}^T \Delta(\tilde{\LocAgent[]{}}^i_t|\LocAgents^{1:i-1}_t) \numberthis \label{eqn: tilde-defined-opt-loc}
    % &= \sum_{t=1}^T \Big( \Obj(\LocAgents^{1:i-1}_t \cup \{\tilde{\LocAgent[]{}}^i_t\}) - \Obj(\LocAgents^{1:i-1}_t) \Big)\\
\end{align*}
We denote with $\cumOPT$, the optimal coverage, precisely $\cumOPT = \sum_{t=1}^T \Obj(\LocAgents_{\star})$.% the optimal location of the agents with the known $\density$ and computed with the brute force approach (NP-hard).

\mypar{Per agent regret} Let us define local regret, as the difference in coverage gain in picking state $\tilde{\LocAgent[]{}}^i_t$ vs the picked location $\LocAgent[i]{t}$ (this disparity is due to not knowing the actual density) 
%$(\Delta(\LocAgent[i]{t}|\LocAgents^{1:i-1}_t) \to \Delta^i(\LocAgent[i]{t}|\LocAgents^{1:i-1}_t))$, if every agent has different reward function.
\begin{align*}
    \LocReg(T) &= \sum_{t=1}^T \Delta(\tilde{\LocAgent[]{}}^i_t|\LocAgents^{1:i-1}_t) - \sum_{t=1}^T \Delta(\LocAgent[i]{t}|\LocAgents^{1:i-1}_t) =\cumOPTloc - \sum_{t=1}^T \Delta(\LocAgent[i]{t}|\LocAgents^{1:i-1}_t) \label{def: local_reg} \numberthis
\end{align*}
\mypar{Actual regret} The actual regret is given by,
\begin{align*}
    % \actualRegret(T) &\neq \sum_{t=1}^T F(\LocAgents^{greedy}) - \sum_{t=1}^{T} F(X_t) \\
    \actualRegret(T) &= \big(1 - \frac{1}{e}\big)  \sum_{t=1}^T F(\LocAgents_{\star}) - \sum_{t=1}^{T} F(X_t) = \big(1 - \frac{1}{e}\big)  \cumOPT - \sum_{t=1}^{T} F(X_t) \numberthis \label{def: reg-act}
\end{align*}

% \todo{We can probably cite \cite{YisongLSB}, state differences and remove the unconstrained case sec}
% \label{sec: bound_act_reg}

\mypar{To prove} In this section we aim to show that actual regret bounded by sum of per agent regret, precisely,
\begin{align*}
\actualRegret(T) &\leq \sum_{i=1}^{\numOfAgents} \LocReg(T) \\
\sum_{i=1}^{\numOfAgents} \LocReg(T) &\geq  \big(1 - \frac{1}{e}\big)  \cumOPT - \sum_{t=1}^{T} \Obj(\LocAgents^{1:\numOfAgents}_t) \tag{Using defi. of $\actualRegret(T)$ from \cref{def: reg-act}}
\end{align*}

\begin{lemma}\label{lem: margin_gain}
For all $\numOfAgents$ agents' local per agent regret $\LocReg(T)$, we have, 
\begin{align*}
    \sum_{t=1}^T \Delta(\LocAgent[i]{t}|\LocAgents^{1:i-1}_t) &\geq \frac{1}{\numOfAgents} \Big( \cumOPT - \sum_{t=1}^T \Obj(\LocAgents^{1:i-1}_t) \Big) - \LocReg(T) \label{eqn: margin_gain} \numberthis
\end{align*}
\end{lemma}
\begin{proof}
\begin{align*}
    \Delta(\tilde{\LocAgent[]{}}^i_t|\LocAgents^{1:i-1}_t) &=  \max_{\LocAgent[i]{t}} \Obj(\LocAgents^{1:i-1}_t \cup \{\LocAgent[i]{}\}) - \Obj(\LocAgents^{1:i-1}_t) \tag{Using definition}\\
    &\geq \frac{\Obj(\LocAgents_{\star}) - \Obj(\LocAgents^{1:i-1}_t)}{\numOfAgents} \tag{Using \cref{eqn: submodularity-rate-eqn} from \cref{lem: submodularity-rate}}\\
    OPT_l^i &\geq \frac{1}{\numOfAgents} \Big( \sum_{t=1}^T \Obj(\LocAgents_{\star}) - \sum_{t=1}^T \Obj(\LocAgents^{1:i-1}_t) \Big) \tag{Sum over time }\\
    &=\frac{1}{\numOfAgents} \Big( \cumOPT - \sum_{t=1}^T \Obj(\LocAgents^{1:i-1}_t) \Big) \tag{Using definition of $\cumOPT$} \\
    \sum_{t=1}^T \Delta(\LocAgent[i]{t}|\LocAgents^{1:i-1}_t) &\geq \frac{1}{\numOfAgents} \Big( \cumOPT - \sum_{t=1}^T \Obj(\LocAgents^{1:i-1}_t) \Big) - \LocReg(T) \tag{Using def. of $\LocReg(T)$ \cref{def: local_reg}}
\end{align*}
\end{proof}

\begin{lemma} \label{lem: bound_act_reg}
For any time t, $\LocAgents_t$ being the recommended location by \macopt, we have
\begin{align*}
    \sum_{t=1}^{T} \Obj(\LocAgents^{1:\numOfAgents}_t) &\geq \big(1 - \frac{1}{e}\big)  \cumOPT - \sum_{i=1}^{\numOfAgents} \LocReg(T)   \label{eqn: bound_by_agent_wise_regret} \numberthis
\end{align*}
And using definition of $\actualRegret(T)$ from \cref{def: reg-act}, this further implies that, 
\begin{align*}
\actualRegret(T) &\leq \sum_{i=1}^{\numOfAgents} \LocReg(T) \numberthis \label{eqn: bound-act-reg-agentwise-reg}
\end{align*}
\end{lemma}
\begin{proof}
The proof is similar to the Lemma 2 from \cite{YisongLSB}. We begin to prove by induction, 
\begin{align*}
    \cumOPT - \sum_{t=1}^{T} \Obj(\LocAgents^{1:i}_t) \leq \Big( 1-\frac{1}{\numOfAgents} \Big)^{i}\cumOPT + \sum_{m=1}^{i}Reg^{m}_l(T) \label{eqn: gap_to_opt} \numberthis
\end{align*}

Our main goal, i.e, \cref{eqn: bound_by_agent_wise_regret} can be proved by substituting $i = \numOfAgents$ and using the inequality $(1-1/\numOfAgents)^\numOfAgents < 1/e$ in \cref{eqn: gap_to_opt}. \\ 

For i = 0, corresponds to no agent case. So it's trivial.

Let's consider gap to optimal value, when i elements are already selected,
\begin{align*}
    \gap^i &= \cumOPT - \sum_{t=1}^{T} \Obj(\LocAgents^{1:i}_t)  \tag{LHS of \cref{eqn: gap_to_opt}}\\
    &= \cumOPT - \sum_{t=1}^{T} \sum_{m=1}^i \Delta(\LocAgent[m]{t}|\LocAgents^{1:m-1}_t) \tag{Sum marginal gain; Using \cref{eqn: sum_delta}}\\
    \gap^{i-1} &= \cumOPT - \sum_{t=1}^{T} \sum_{m=1}^{i-1} \Delta(\LocAgent[m]{t}|\LocAgents^{1:m-1}_t) \\
    \implies \gap^{i} &= \gap^{i-1} - \sum_{t=1}^T \Delta(\LocAgent[i]{t}|\LocAgents^{1:i-1}_t)  \tag{Subtract $\gap^{i-1}$ from $\gap^{i}$}\\
    \implies   \sum_{t=1}^T & \Delta(\LocAgent[i]{t}|\LocAgents^{1:i-1}_t) = \gap^{i-1} - \gap^{i}  \label{eqn: gap_diff} \numberthis
\end{align*}
This says that the gap to optimal reduces by $\sum_{t=1}^T \Delta(\LocAgent[i]{t}|\LocAgents^{1:i-1}_t)$ after adding element $\LocAgent[i]{t} \ \forall \ t$. 
\begin{align*}
\sum_{t=1}^T \Delta(\LocAgent[i]{t}|\LocAgents^{1:i-1}_t) &\geq \frac{1}{\numOfAgents}(\gap^{i-1})  - \LocReg(T) \tag{From \cref{eqn: margin_gain} and $\gap^{i}$ definition}\\
\implies \gap^{i-1} - \gap^{i} &\geq \frac{1}{\numOfAgents}(\gap^{i-1})  - \LocReg(T) \tag{From \cref{eqn: gap_diff}}\\
\implies \gap^{i} &\leq \Big(1-\frac{1}{\numOfAgents} \Big) \gap^{i-1}  + \LocReg(T)  \\
&\leq \Big(1-\frac{1}{\numOfAgents} \Big)^2 \gap^{i-2}  + \sum_{m=1}^2 \LocReg(T) \tag{Subs $\gap^{i-1}$, Doing the telescopic bound}\\ 
\vdots\\
&\leq \Big(1-\frac{1}{\numOfAgents} \Big)^i \gap^{0}  + \sum_{m=1}^i \LocReg(T)\\
&= \Big(1-\frac{1}{\numOfAgents} \Big)^i \cumOPT  + \sum_{m=1}^i \LocReg(T) \\
\cumOPT - \sum_{t=1}^{T} \Obj(\LocAgents^{1:i}_t) &\leq \Big( 1-\frac{1}{\numOfAgents} \Big)^{i}\cumOPT + \sum_{m=1}^{i}Reg^{m}_l(T) \tag{Using $\gap^i$ definition}
\end{align*}
Hence proved.
% Also, 
% \begin{align*}
%     \actualRegret(T) &\leq \sum_{i=1}^{\numOfAgents} \LocReg(T) \numberthis \label{eqn: bound_act_reg}
% \end{align*}
\end{proof}

%%%%%%%%%%%%%%%%%%%%%%%%%%%%%%%%%%%%%%%%%%%%%%%%%%%%%%%%%%%%%%
%%%%%%%%%%%%%%%%%%%%%%%%%%%%%%%%%%%%%%%%%%%%%%%%%%%%%%%%%%%%%%
%%%%%%%%%%%%%%%%%%%%%%%%%%%%%%%%%%%%%%%%%%%%%%%%%%%%%%%%%%%%%%
%%%%%%%%%%%%%%%%%%%%%%%%%%%%%%%%%%%%%%%%%%%%%%%%%%%%%%%%%%%%%%

\subsection{Constrained case}
\label{sec: defi-reg-const}
\mypar{Simple regret} We define for a particular $t$, simple regret $\simActReg{}$ and per agent local regret $\simReg{}$ respectively as:
\begin{align*}
\simActReg{} &= (1- \frac{1}{e})  \sum_{\batch \in \BatchColl{}} \Objfunc{\LocAgents^{\batch}_{\star}}{\density}{\Rbareps{}{\LocAgents^{\batch}_0}}  -  \sum_{\batch \in \BatchColl{t}} \sum_{i\in \batch}  \delgain{\LocAgent[i]{t}|\LocAgents^{1:i-1}_t}{\density}{\uniSet[,\batch]{t}}, \\
\simOptiReg{t} &= (1- \frac{1}{e}) \sum_{\batch \in \BatchColl{t}}  \Objfunc{\LocAgents_{\star}^{\batch}}{\density}{\uniSet[, \batch]{t}} - \sum_{\batch \in \BatchColl{t}}  \sum_{i\in \batch}  \delgain{\LocAgent[i]{t}|\LocAgents^{1:i-1}_t}{\density}{\uniSet[, \batch]{t}} \\
\simReg{} &= \sum_{\batch \in \BatchColl{t}} \sum_{i\in \batch}  \delgain{\tilde{\LocAgent[]{}}^i|\LocAgents^{1:i-1}_t}{\density}{\uniSet[, \batch]{t}} -  \delgain{\LocAgent[i]{t}|\LocAgents^{1:i-1}_t}{\density}{\uniSet[, \batch]{t}}  \numberthis \label{def: sim-act-reg-loc-reg}
\end{align*}
\mypar{Cumulative regret} The actual cumulative regret $\optiRegret(T)$ and the per agent cumulative regret $\optiLocReg(T)$ are respectively given by,
\begin{align*}
    \optiRegret(T) = \sum_{t=1}^{T} \simActReg{}~\text{and}~ \optiLocReg(T) =  \sum_{t=1}^{T} \simReg{} \numberthis \label{eqn: cum-reg-const-opti}
\end{align*}

\mypar{On bounding per batch regret} 

\emph{Optimal coverage in a batch $\batch$}
\begin{align*}
    \simOPT &= \Objfunc{\LocAgents_{\star}^{\batch}}{\density}{\uniSet[, \batch]{t}} \\
    \simOPTloc &= \max_{\LocAgent[i]{}}\Objfunc{\LocAgents^{1:i-1}_t \cup \{ \LocAgent[i]{} \}}{\density}{\uniSet[, \batch]{t}} - \Objfunc{\LocAgents^{1:i-1}_t}{\density}{\uniSet[, \batch]{t}}\\
    &= \max_{\LocAgent[i]{}} \delgain{\LocAgent[i]{}|\LocAgents^{1:i-1}_t}{\density}{\uniSet[, \batch]{t}} = \delgain{\tilde{\LocAgent[]{}}^i|\LocAgents^{1:i-1}_t}{\density}{\uniSet[, \batch]{t}} \\
    % \simOptiReg[\batch]{t} & = \big(1 - \frac{1}{e}\big)  \simOPT - \Objfunc{\LocAgents^{\batch}_t}{\density}{\uniSet[,\batch]{t}} \\
    \simLocReg{t} &= \delgain{\tilde{\LocAgent[]{}}^i|\LocAgents^{1:i-1}_t}{\density}{\uniSet[, \batch]{t}} -  \delgain{\LocAgent[i]{t}|\LocAgents^{1:i-1}_t}{\density}{\uniSet[, \batch]{t}} \numberthis \label{def: sim-opti-local-reg}
\end{align*}

% \begin{align*}
%     \simLocReg{t} &= \delgain{\tilde{\LocAgent[]{}}^i|\LocAgents^{1:i-1}_t}{\density}{\uniSet[, \batch]{t}} -  \delgain{\LocAgent[i]{t}|\LocAgents^{1:i-1}_t}{\density}{\uniSet[, \batch]{t}}  \numberthis 
% \end{align*}

\emph{To prove:}
\begin{align*}
% \simOptiReg[\batch]{t} \leq  \sum_{i\in \batch} \simLocReg{t}\\
\Objfunc{\LocAgents^{\batch}_t}{\density}{\uniSet[,\batch]{t}} &\geq \big(1 - \frac{1}{e}\big)  \simOPT -  \sum_{i\in \batch} \simLocReg{t}  \numberthis
\end{align*}

% \subsubsection{On bounding the greedy regret \texorpdfstring{$\simActReg{}$}{Lg} with individual agent-wise regret}
% \label{sec: opts_bound_act_reg}

% For a set of agents \batch, having the same union set, $\uniSet[,\batch]{t}$, the following bound holds,
% \begin{multline*}
%   (1- \frac{1}{e}) \Objfunc{\LocAgents_{\star}}{\density}{\uniSet[, \batch]{t}} - \sum_{i\in \batch}  \delgain{\LocAgent[i]{t}|\LocAgents^{1:i-1}_t}{\density}{\uniSet[, \batch]{t}}  \\
%     \leq \sum_{i\in \batch}  \delgain{\tilde{\LocAgent[]{}}^i|\LocAgents^{1:i-1}_t}{\density}{\uniSet[, \batch]{t}} -  \delgain{\LocAgent[i]{t}|\LocAgents^{1:i-1}_t}{\density}{\uniSet[, \batch]{t}}
% \end{multline*}
%
% \begin{align*}
%   \optiRegret(T) &= (1- \frac{1}{e})\sum_{t=1}^{T}\Objfunc{\LocAgents_{\star}}{\density}{\Rbar} - \sum_{t=1}^{T} \sum_i^{\numOfAgents}  \delgain{\LocAgent[i]{t}|\LocAgents^{1:i-1}_t}{\density}{\uniSet[]{t}} \\
%   \optiLocReg(T) &= \sum_{t=1}^{T}  \delgain{\tilde{\LocAgent[]{}}^i|\LocAgents^{1:i-1}_t}{\density}{\uniSet[]{t}} - \sum_{t=1}^{T} \delgain{\LocAgent[i]{t}|\LocAgents^{1:i-1}_t}{\density}{\uniSet[]{t}} \numberthis \label{def: opts_local_reg}
% \end{align*}
%
\begin{proposition}\label{lem: opts_margin_gain}
Let $\numOfAgents_{\batch}$ be the number of agents in batch $\batch$ and for all such agents per agent regret is $\simLocReg{t}$. Then the following holds, 
\begin{align*}
     \delgain{\LocAgent[i]{t}|\LocAgents^{1:i-1}_t}{\density}{\uniSet[,\batch]{t}} &\geq \frac{1}{\numOfAgents_{\batch}} \Big( \simOPT -  \Objfunc{\LocAgents^{1:i-1}_t}{\density}{\uniSet[,\batch]{t}} \Big) - \simLocReg{t} \label{eqn: opts_margin_gain} \numberthis
\end{align*}
\end{proposition}
\begin{proof}
\begin{align*}
    \delgain{\tilde{\LocAgent[]{}}^i_t|\LocAgents^{1:i-1}_t}{\density}{\uniSet[,\batch]{t}} &=  \max_{\LocAgent[i]{t}} \Objfunc{\LocAgents^{1:i-1}_t \cup \{\LocAgent[i]{}\}}{\density}{\uniSet[,\batch]{t}} - \Objfunc{\LocAgents^{1:i-1}_t}{\density}{\uniSet[,\batch]{t}} \tag{Using definition}\\
    &\geq \frac{ \Objfunc{\LocAgents_{\star}}{\density}{\uniSet[,\batch]{t}} - \Objfunc{\LocAgents^{1:i-1}_t}{\density}{\uniSet[,\batch]{t}}}{\numOfAgents_{\batch}} \tag{Using \cref{eqn: submodularity-rate-eqn} from \cref{lem: submodularity-rate}} \\
    % &\geq \frac{\Objfunc{\LocAgents_{\star}}{\density}{\Rbar} - \Objfunc{\LocAgents^{1:i-1}_t}{\density}{\uniSet[,\batch]{t}}}{\numOfAgents_{\batch}} \tag{since $\uniSet[,\batch]{t} \supseteq \Rbar, \implies \Objfunc{\LocAgents_{\star}}{\density}{\uniSet[,\batch]{t}} \geq \Objfunc{\LocAgents_{\star}}{\density}{\Rbar}$ (Note that both $\LocAgents_{\star}$ are diff)} \\
     \simOPTloc &\geq \frac{1}{\numOfAgents_{\batch}} \Big( \simOPT - \Objfunc{\LocAgents^{1:i-1}_t}{\density}{\uniSet[,\batch]{t}} \Big) \tag{Using definition of $\simOPT$ and $\simOPTloc$} \\
     \delgain{\LocAgent[i]{t}|\LocAgents^{1:i-1}_t}{\density}{\uniSet[,\batch]{t}} &\geq \frac{1}{\numOfAgents_{\batch}} \Big( \simOPT -  \Objfunc{\LocAgents^{1:i-1}_t}{\density}{\uniSet[,\batch]{t}} \Big) - \simLocReg{t} \tag{Using def. of $\simLocReg{t}$ \cref{def: sim-opti-local-reg}} 
\end{align*}
\end{proof}

\begin{lemma}\label{lem: opts_bound_act_reg}
For any time t, $\LocAgents^{\batch}_t$ being the recommended location by \safemac in the union set $\uniSet[,\batch]{t}$, we have
\begin{align*}
    \Objfunc{\LocAgents^{\batch}_t}{\density}{\uniSet[,\batch]{t}} &\geq \big(1 - \frac{1}{e}\big)  \simOPT -  \sum_{i\in \batch} \simLocReg{t},  \label{eqn: opts_bound_by_agent_wise_regret} \numberthis
\end{align*}
% And this further implies that, 
% \begin{align*}
%     \simOptiReg[\batch]{t} \leq  \sum_{i\in \batch} \simLocReg{t}
% \end{align*}
\end{lemma}
\begin{proof}
The proof is similar to the Lemma 2 from \cite{YisongLSB}. We begin to prove by induction, 
\begin{align*}
    \simOPT -  \Objfunc{\LocAgents^{1:i}_t}{\density}{\uniSet[,\batch]{t}} \leq \Big( 1-\frac{1}{\numOfAgents_{\batch}} \Big)^{i} \simOPT + \sum_{m=1}^{i}\simLocReg{t} \label{eqn: opts_gap_to_opt} \numberthis
\end{align*}
For i = 0, corresponds to no agent case. So it's trivial.

Let's consider gap to optimal value, when i elements are already selected,
\begin{align*}
    \gap^i &= \simOPT -  \Objfunc{\LocAgents^{1:i}_t}{\density}{\uniSet[,\batch]{t}}  \tag{LHS of \cref{eqn: opts_gap_to_opt}}\\
    &= \simOPT -  \sum_{m=1}^i \delgain{\LocAgent[m]{t}|\LocAgents^{1:m-1}_t}{\density}{\uniSet[,\batch]{t}} \tag{sum of marginal gain}\\
    \gap^{i-1} &= \simOPT -  \sum_{m=1}^{i-1} \delgain{\LocAgent[m]{t}|\LocAgents^{1:m-1}_t}{\density}{\uniSet[,\batch]{t}}  \\
    \implies \gap^{i} &= \gap^{i-1} - \delgain{\LocAgent[i]{t}|\LocAgents^{1:i-1}_t}{\density}{\uniSet[,\batch]{t}}  \tag{Subtract $\gap^{i-1}$ from $\gap^{i}$}\\
    \implies  \quad & \delgain{\LocAgent[i]{t}|\LocAgents^{1:i-1}_t}{\density}{\uniSet[,\batch]{t}} = \gap^{i-1} - \gap^{i}  \label{eqn: opts_gap_diff} \numberthis
\end{align*}
This says that the gap to optimal reduces by $\delgain{\LocAgent[i]{t}|\LocAgents^{1:i-1}_t}{\density}{\uniSet[,\batch]{t}}$ after adding element $\LocAgent[i]{t}$. 
\begin{align*}
 \delgain{\LocAgent[i]{t}|\LocAgents^{1:i-1}_t}{\density}{\uniSet[,\batch]{t}} &\geq \frac{1}{\numOfAgents_{\batch}}(\gap^{i-1})  - \simLocReg{t} \tag{From \cref{eqn: opts_margin_gain} and $\gap^{i}$ definition}\\
\implies \gap^{i-1} - \gap^{i} &\geq \frac{1}{\numOfAgents_{\batch}}(\gap^{i-1})  - \simLocReg{t} \tag{From \cref{eqn: gap_diff}}\\
\implies \gap^{i} &\leq \Big(1-\frac{1}{\numOfAgents_{\batch}} \Big) \gap^{i-1}  + \simLocReg{t}  \\
&\leq \Big(1-\frac{1}{\numOfAgents_{\batch}} \Big)^2 \gap^{i-2}  + \sum_{m=1}^2 \simLocReg{t} \tag{Subs $\gap^{i-1}$, Doing the telescopic bound}\\ 
&\vdots\\
&\leq \Big(1-\frac{1}{\numOfAgents_{\batch}} \Big)^i \gap^{0}  + \sum_{m=1}^i \simLocReg{t}\\
&= \Big(1-\frac{1}{\numOfAgents_{\batch}} \Big)^i \simOPT  + \sum_{m=1}^i \simLocReg{t} \\
\simOPT - \Objfunc{\LocAgents^{1:i}_t}{\density}{\uniSet[,\batch]{t}} &\leq \Big( 1-\frac{1}{\numOfAgents_{\batch}} \Big)^{i}\simOPT + \sum_{m=1}^{i}\simLocReg{t} \tag{Using $\gap^i$ definition}
\end{align*}
Our main goal, i.e, \cref{eqn: opts_bound_by_agent_wise_regret} can be proved by substituting $i = \numOfAgents$ and using the inequality $(1-1/\numOfAgents)^\numOfAgents < 1/e$ in \cref{eqn: opts_gap_to_opt}. Hence proved.
\end{proof}

\mypar{On combining all the batches}
\begin{lemma}\label{lem: opti-bound-act-reg}
For any time t, $\LocAgents_t$ being the location recommended by \safemac, we have
\begin{align*}
    \simActReg{} \leq \simOptiReg{t} \leq \simReg{} \label{eqn: sim-reg-const-bound} \numberthis
\end{align*}
This further implies that,
\begin{align*}
    \optiRegret(T) \leq \optiLocReg(T) \numberthis \label{eqn: opti-cum-regret-bound}
\end{align*}
\end{lemma}
\begin{proof} For a batch $\batch$ of agents, using \cref{eqn: opts_gap_to_opt} from \cref{lem: opts_bound_act_reg} and substituting $\simLocReg{t}$ from \cref{def: sim-opti-local-reg} we know that, 
\begin{multline*}
   (1- \frac{1}{e}) \Objfunc{\LocAgents_{\star}^{\batch}}{\density}{\uniSet[, \batch]{t}} - \sum_{i\in \batch}  \delgain{\LocAgent[i]{t}|\LocAgents^{1:i-1}_t}{\density}{\uniSet[, \batch]{t}}  \\
    \leq \sum_{i\in \batch}  \delgain{\tilde{\LocAgent[]{}}^i|\LocAgents^{1:i-1}_t}{\density}{\uniSet[, \batch]{t}} -  \delgain{\LocAgent[i]{t}|\LocAgents^{1:i-1}_t}{\density}{\uniSet[, \batch]{t}} 
\end{multline*}
By summing over all the $\batch \in \BatchColl{t}$, we get 
\begin{multline*}
    \simOptiReg{t} = (1- \frac{1}{e}) \sum_{\batch \in \BatchColl{t}}  \Objfunc{\LocAgents_{\star}^{\batch}}{\density}{\uniSet[, \batch]{t}} - \sum_{\batch \in \BatchColl{t}}  \sum_{i\in \batch}  \delgain{\LocAgent[i]{t}|\LocAgents^{1:i-1}_t}{\density}{\uniSet[, \batch]{t}}  \\
    \leq \sum_{\batch \in \BatchColl{t}}  \sum_{i\in \batch}  \delgain{\tilde{\LocAgent[]{}}^i|\LocAgents^{1:i-1}_t}{\density}{\uniSet[, \batch]{t}} -  \delgain{\LocAgent[i]{t}|\LocAgents^{1:i-1}_t}{\density}{\uniSet[, \batch]{t}} \numberthis \label{eqn: sim-reg-opti-bound}
\end{multline*}
Note that in \cref{def: sim-act-reg-loc-reg}, both the $\LocAgents^{\batch}_{\star}$ represents optimal agent's location in their respective coverage set i.e, $\Rbareps{}{\LocAgent[i]{0}}$ and $\uniSet[, \batch]{t}$, hence both the $\LocAgents_{\star}^{\batch}$ are different. Since, $\bigcup_{i \in \batch} \Rbareps{}{\{\LocAgent[i]{0}\}} \subseteq \optiSet[, \batch]{t} \subseteq \uniSet[, \batch]{t}$ 
$ \implies \sum_{\batch \in \BatchColl{}} \Objfunc{\LocAgents_{\star}^{\batch}}{\density}{\Rbareps{}{\LocAgents_{0}^{\batch}}} \leq \sum_{\batch \in \BatchColl{t} }\Objfunc{\LocAgents_{\star}^{\batch}}{\density}{\uniSet[, \batch]{t}} $,

Moreover on using \cref{def: sim-act-reg-loc-reg}, \cref{eqn: sim-reg-opti-bound} and  we can conclude, 
$$\simActReg{} \leq \simOptiReg{t} \leq \simReg{}.$$
This further implies \cref{eqn: opti-cum-regret-bound} using definition in \cref{eqn: cum-reg-const-opti}. Hence Proved
\end{proof}
%%%%%%%%%%%%%%%%%%%%%%%%%%%%%%%%%%%%%%
%%% Copied from macopt tex file
%%%%%%%%%%%%%%%%%%%%%%%%%%%%%%%%%%%%%%

\newpage
\section{Proof. for \texorpdfstring{\cref{thm: macopt}}{Lg} (\macopt)}
\label{Apx: thm-macopt}
% \texorpdfstring{$\actualRegret(T)$}{Lg}
\restatemacopt

\begin{proof}
The proof for \cref{thm: macopt} goes in the following steps:
\begin{enumerate}
        \item We first exploit the conditional linearity of the submodular objective to bound the cumulative regret defined in \cref{eqn: reg-def} with a sum of per agent regret ($\sum_{i=1}^{\numOfAgents} \LocReg(T)$). Precisely, we show $\actualRegret(T) \leq \sum_{i=1}^{\numOfAgents} \LocReg(T)$ in \cref{lem: bound_act_reg}.
        \item We next bound the per agent regret with the information capacity $\gammadensity{\numOfAgents T}$, a quantity that measures the largest reduction in uncertainty about the density that can be obtained from $\numOfAgents T$ noisy evaluations of it.
        \begin{itemize}
        \item For this, we quantify the information \macopt acquires through the noisy density observations in \cref{lem: mutual-info-macopt}, through \emph{the information gain} $I(y_A;\density) = H(y_A) - H(y_A | \density)$, where H denotes the Shannon entropy and $A$ is the set of locations evaluated by \macopt. 
        \item Next we bound the per agent regret $\LocReg(T)$ with the information gain in \cref{lem: bound-agent-wise-regret-with-sigma}-\ref{lem: reg-bound-info} which is in turn bounded by the maximum information capacity.
        \end{itemize}
\end{enumerate}
Finally, \cref{thm: macopt} is a direct consequence of \cref{lem: bound_act_reg} and \cref{lem: reg-bound-info}.
\end{proof}
In the end of the section, we proof \cref{cor: macopt} which guarantees near optimal result in finite time.\\

%  Similar to \citet{beta-srinivas}, we can bound information gain with a theoretical quantity \emph{maximum information gain} $\gammadensity{\numOfAgents T}$ obtained after T rounds and defined as $\gammadensity{\numOfAgents T} \coloneqq \sup_{A \subseteq \Domain: |A|=\numOfAgents T} I(y_A | \density)$. Since $\gammadensity{\numOfAgents T}$\citep{gammaT-vakili21a} grows sublinearly with $T$ for commonly used kernels, so does \macopt's regret in \cref{eq:macopt_thm}.

% \todo{once done may be change $\LocAgent[g,1:\numOfAgents]{t}$  to set $A$}
\begin{lemma} \label{lem: mutual-info-macopt}
The information gain for the points observed by \macopt can be expressed as:
\begin{align*}
    I(Y_{\LocAgent[g,1:\numOfAgents]{1:T}};\density) &= \frac{1}{2} \sum_{t=1}^T \log(\det( I + \noisedensity \densitykernel_{\LocAgent[g,1:\numOfAgents]{t}})) = \frac{1}{2} \sum_{t=1}^T \sum_{i=1}^{\numOfAgents} \log( 1 + \noisedensity \lambda_{i,t}), 
\end{align*}
where $\LocAgent[g,1:\numOfAgents]{1:T}$ is the set of goal locations set by \macopt for all $1:\numOfAgents$ agents up to time $T$. $\densitykernel_{\LocAgent[g,1:\numOfAgents]{t}}$ is the positive definite kernel matrix formed by the observed locations and $\lambda_{i,t}$ represents eigenvalue of the matrix.
\end{lemma}
%Proof
\begin{proof} We can precisely quantify this notion through \emph{the information gain}
\begin{align*}
  I(Y_{\LocAgent[g,1:\numOfAgents]{1:T}};\density) = H(Y_{\LocAgent[g,1:\numOfAgents]{1:T}}) - H(Y_{\LocAgent[g,1:\numOfAgents]{1:T}}| \density)  \label{eqn: mutual-info-definition} \numberthis
\end{align*}
where H denotes the Shannon entropy. It can be defined as,
\begin{align*}
    H(Y_{\LocAgent[g,1:\numOfAgents]{1:T}}) &= H(Y^{1:\numOfAgents}_{T} |Y_{\LocAgent[g,1:\numOfAgents]{1:T-1}}) + H(Y_{\LocAgent[g,1:\numOfAgents]{1:T-1}}) \tag{Defined $Y^{1:\numOfAgents}_{T} \coloneqq \{ y^1_T,y^2_T,...,y^{\numOfAgents}_T \}$}\\
    &= \frac{1}{2} \log(\det(2\pi e(\sigma^2 I + \densitykernel_{\LocAgent[g,1:\numOfAgents]{T}}))) + H(Y^{1:\numOfAgents}_{T-1}|Y_{\LocAgent[g,1:\numOfAgents]{1:T-2}}) + ... \numberthis \label{eqn: MV-Gaussian}\\
     &= \frac{1}{2} \numOfAgents \log(2\pi e \sigma^2) +  \frac{1}{2} \log(\det( I + \noisedensity \densitykernel_{\LocAgent[g,1:\numOfAgents]{T}})) + H(Y^{1:\numOfAgents}_{T-1}|Y_{\LocAgent[g,1:\numOfAgents]{1:T-2}}) + ... \numberthis \label{eqn: refactoring-det} \\
    &= \frac{1}{2} \sum_{t=1}^T \numOfAgents \log(2\pi e \sigma^2) +  \frac{1}{2} \sum_{t=1}^T \log(\det( I + \noisedensity \densitykernel_{\LocAgent[g,1:\numOfAgents]{t}}))  \numberthis \label{eqn: recursion}
\end{align*}
For \cref{eqn: MV-Gaussian}, we used that, $Y^{1:\numOfAgents}_{T} \sim \mathcal{N}( \mudensity[T-1](\LocAgent[g,1:\numOfAgents]{T}), \sigma^2 I + \densitykernel_{\LocAgent[g,1:\numOfAgents]{T}} )$ is jointly a multivariate Gaussian. 
% Hence $H(Y^{1:\numOfAgents}_{T} |Y_{\LocAgent[g,1:\numOfAgents]{1:T-1}}) = \frac{1}{2} \log(\det(2\pi e(\sigma^2 I + \densitykernel_{\LocAgent[g,1:\numOfAgents]{T}})))$.
\cref{eqn: refactoring-det} follows by simplifying $\det$, precisely, $\frac{1}{2} \log(\det(2\pi e(\sigma^2 I + \densitykernel_{\LocAgent[g,1:\numOfAgents]{T}})))$ = $\frac{1}{2} \log({(2\pi e\sigma^2)}^{\numOfAgents}\det( I + \noisedensity\densitykernel_{\LocAgent[g,1:\numOfAgents]{T} }))$ and finally \cref{eqn: recursion} by recursively repeating above 2 steps till $t=1$.
$H(Y_{\LocAgent[g,1:\numOfAgents]{1:T}}|\density) = \frac{1}{2} \sum_{t=1}^T \numOfAgents \log(2\pi e \sigma^2)$ is the entropy because of the noise. On substituting this, with \cref{eqn: recursion} in \cref{eqn: mutual-info-definition} we obtain,
\begin{align*}
I(Y_{\LocAgent[g,1:\numOfAgents]{1:T}};\density) &= \frac{1}{2} \sum_{t=1}^T \log(\det( I + \noisedensity \densitykernel_{\LocAgent[g,1:\numOfAgents]{t}}))   \\
&= \frac{1}{2} \sum_{t=1}^T \log(\prod_{i=1}^{\numOfAgents}( 1 + \noisedensity \lambda_{i,t})) \tag{Using \eq~\ref{eqn: log_mat_inequality}}\\
&= \frac{1}{2} \sum_{t=1}^T \sum_{i=1}^{\numOfAgents} \log( 1 + \noisedensity \lambda_{i,t}) \numberthis \label{eqn: macopt-information-gain}
\end{align*}
Hence Proved.
\end{proof}

\textbf{Log mat inequality:}
\begin{align*}
    \log(\det(I + \noisedensity \densitykernel)) &= \log(\det(RR^{\trans} + \noisedensity R{\Lambda}R^{\trans})) \tag{$\densitykernel = R{\Lambda}R^{\trans}, RR^{\trans} = I$} \\
    &= \log(\det( R (I + \noisedensity \Lambda) R^{\trans} )) \\
    &= \log(\det( R R^{\trans} )) + \log(\det(I + \noisedensity \Lambda)) \tag{k is dimension of $\densitykernel$}\\
    &= \log(\prod_{i=1}^{k}(1 + \noisedensity \lambda_i)) \numberthis \label{eqn: log_mat_inequality} \\
\end{align*}

\begin{lemma}\label{lem: bound-agent-wise-regret-with-sigma}
Till any time $T$, ${\beta^{\density}_t}^{1/2} = B_{\density} + 4 \sigma_{\density} \sqrt{\gamma^{\density}_{\numOfAgents t} + 1 + \ln(1/\delta)}$, if $|\density(\PtInDomain) - \mudensity[t-1](\PtInDomain)| \leq \beta^{1/2}_{t} \sigdensity[t-1](\PtInDomain)$ for all $\PtInDomain \in \Domain$, then the agent wise cumulative regret $\LocReg(T)$, is bounded by $\sum_{t=1}^T  2 \sqrt{\betadensity[t]} \sum_{\PtInDomain \in \Discat[i-]_t} \sigdensity[t-1](\PtInDomain) / \factor $ for agent $i$. 
\end{lemma}
\begin{proof} For notation convenience: $\Discat[i-]_{t} := \Discat[i]_{t}\backslash \Discat[1:i-1]_{t}$ and $\tlDiscat[i-]_{t} := \tlDiscat[i]_{t} \backslash \Discat[1:i-1]_{t}$ 

In \macopt $\LocAgent[i]{t}$ is defined such that, 
\begin{align*}
    \LocAgent[i]{t} = \argmax_{\PtInDomain}\sum_{\PtInDomain \in \Discat[i-]_t}  \mudensity[t-1](\PtInDomain) + \sqrt{\betadensity[t]} \sigdensity[t-1](\PtInDomain) \label{eqn: pick_strategy} \numberthis
\end{align*}
Due to our picking strategy,
\begin{align*}
    \sum_{\PtInDomain \in \tlDiscat[i-]_t }\density(\PtInDomain) &\leq \sum_{\PtInDomain \in \tlDiscat[i-]_t} \big( \mudensity[t-1](\PtInDomain) + \sqrt{\betadensity[t]}  \sigdensity[t-1](\PtInDomain) \big) 
    \leq \sum_{\PtInDomain \in \Discat[i-]_t} \big( \mudensity[t-1](\PtInDomain) + \sqrt{\betadensity[t]}  \sigdensity[t-1](\PtInDomain) \big)  \numberthis \label{eqn: xtilde_inequality}
\end{align*}
This first inequality follows due to upper bound and the second one follows based on how $\LocAgent[i]{t}$ is picked (\cref{eqn: pick_strategy}).
\begin{align*}
     \LocReg(T) &=   \sum_{t=1}^T \Delta(\tilde{\LocAgent[]{}}^i_t|X^{1:i-1}_t) - \sum_{t=1}^T \Delta(\LocAgent[i]{t}|X^{1:i-1}_t)  \tag{with definition \cref{def: local_reg}}\\
    % &=   \sum_{t=1}^T \Obj(X^{1:i-1}_t \cup \{\tilde{\LocAgent[]{}}^i_t\} ) - \sum_{t=1}^T \Obj(X^{1:i-1}_t \cup \{\LocAgent[i]{t}\} )  \tag{Defi. of $\Delta$, $\Obj(.)$ cancels out}\\
    &=   \sum_{t=1}^T \Big( \sum_{\PtInDomain \in \tlDiscat[i-]_t} \density(\PtInDomain) - \sum_{\PtInDomain \in \Discat[i-]_t} \density(\PtInDomain) ) \Big) / \factor \tag{Using defi. $\Delta(.|X^{1:i-1}_t)$ \cref{eqn: marginal-gain}}\\
    &\leq   \sum_{t=1}^T \Big(  \sum_{\PtInDomain \in \Discat[i-]_t}  \mudensity[t-1](\PtInDomain) + \sqrt{\betadensity[t]}  \sigdensity[t-1](\PtInDomain)  -  \sum_{\PtInDomain \in \Discat[i-]_t} \density(\PtInDomain)  \Big)/ \factor \tag{From \cref{eqn: xtilde_inequality}} \\
     &\leq   \sum_{t=1}^T \Big(  \sum_{\PtInDomain \in \Discat[i-]_t}  \mudensity[t-1](\PtInDomain) + \sqrt{\betadensity[t]}  \sigdensity[t-1](\PtInDomain)   - \sum_{\PtInDomain \in \Discat[i-]_t}  \mudensity[t-1](\PtInDomain) - \sqrt{\betadensity[t]} \sigdensity[t-1](\PtInDomain)  \Big)/ \factor \tag{Since, $\density(\PtInDomain) \geq \mudensity[t-1](\PtInDomain) - \sqrt{\betadensity[t]} \sigdensity[t-1](\PtInDomain) \ \forall \ \PtInDomain$ }\\
      &=  \sum_{t=1}^T  2 \sqrt{\betadensity[t]} \sum_{\PtInDomain \in \Discat[i-]_t} \sigdensity[t-1](\PtInDomain) / \factor 
      %\leq  \sum_{t=1}^T  2 \sqrt{\betadensity[t]} \max_{\PtInDomain \in \Discat[i-]_t} \sigdensity[t-1](\PtInDomain) 
      \numberthis  \label{eqn: bound-agentwise-reg}
\end{align*}
% The last inequality follows since $\sum_{\PtInDomain\in \Discat[i-]_{t}} \sigdensity[t-1](\PtInDomain) \leq N \max_{\PtInDomain\in \Discat[i-]_{t}} \sigdensity[t-1](\PtInDomain)$ and $|\Discat[i-]_t| \leq \factor$.
\end{proof}

\begin{lemma}\label{lem: reg-bound-info}
Let $\delta \in (0,1)$ and let ${\beta^{\density}_t}^{1/2} = B_{\density} + 4 \sigma_{\density} \sqrt{\gamma^{\density}_{\numOfAgents t} + 1 + \ln(1/\delta)}$. Then for $\numOfAgents$ agents, $\forall T \geq 1$ the following holds with probability $1-\delta$,
\begin{align*}
    ( \sum_{i=1}^{\numOfAgents} \LocReg(T))^2 \leq \frac{8 \DiskCoverageRatio \numOfAgents T \betadensity[T] I(Y_{\LocAgent[g, 1:\numOfAgents]{1:T}};\density)}{\log(1+\numOfAgents\noisedensity)}   \leq \frac{8 \DiskCoverageRatio \numOfAgents T \betadensity[T] \gammadensity{\numOfAgents T}}{\log(1+\numOfAgents\noisedensity)}  
\end{align*}
\end{lemma}
\begin{proof}
By sum over all the $\numOfAgents$ agents from \cref{lem: bound-agent-wise-regret-with-sigma},  we get
\begin{align*}
    \sum_{i=1}^{\numOfAgents} \LocReg(T) \leq \sum_{i=1}^{\numOfAgents} \sum_{t=1}^T  2 \sqrt{\betadensity[t]} \sum_{\PtInDomain \in \Discat[i-]_t} \sigdensity[t-1](\PtInDomain) / \factor  \numberthis \label{eqn: max-single-bound_agent_wise}
\end{align*}
Let's consider, part of \cref{eqn: max-single-bound_agent_wise}, that is 
\begin{align*}
    \Big( 2 \sqrt{\betadensity[t]} \sum_{i=1}^{\numOfAgents}  \sum_{\PtInDomain \in \Discat[i-]_t} \sigdensity[t-1](\PtInDomain) / \factor \Big)^2 &\leq 4 \betadensity[t]  \sum_{i=1}^{\numOfAgents} \sum_{\PtInDomain \in \Discat[i-]_t} \Big( \sigdensity[t-1](\PtInDomain) \Big)^2 / \factor \numberthis \label{eqn: CS-ineq}\\
    &\leq 4 \betadensity[t]  \sum_{i=1}^{\numOfAgents} |\Discat[i-]_t| \Big( \argmax_{\Discat[i-]}\sigdensity[t-1](\PtInDomain) \Big)^2 / \factor \numberthis \label{eqn: max-ineq}\\
    &\leq 4 \betadensity[t] \DiskCoverageRatio \sum_{i=1}^{\numOfAgents} \Big( \sigdensity[t-1](\LocAgent[g,i]{t}) \Big)^2 \numberthis \label{eqn: w-sq-bound}\\
    % \sumMaxWidth{t} &:= 2 \sqrt{\betadensity[t]} \sum_{i=1}^{\numOfAgents}  \max_{\PtInDomain \in \Discat[i-]_t} \sigdensity[t-1](\PtInDomain)  \tag{part of \cref{eqn: max-single-bound_agent_wise}}\\
    % \sumMaxWidth[2]{t} &= 4 \betadensity[t] \Big(  \sum_{i=1}^{\numOfAgents} \max_{\PtInDomain \in \Discat[i-]_t} \sigdensity[t-1](\PtInDomain)  \Big)^2 \tag{Square operation}\\
    % &\leq 4 \betadensity[t] \numOfAgents \sum_{i=1}^{\numOfAgents} \Big( \sigdensity[t-1](\LocAgent[g,i]{t}) \Big)^2 \tag{Cauchy-Schwarz inequality, $\LocAgent[g,i]{t} = \argmax\limits_{\PtInDomain \in \Discat[i-]_t} \sigma^2_{\density_{t-1}}(\PtInDomain)$} \\
    &= 4 \betadensity[t] \DiskCoverageRatio \sum_{i=1}^{\numOfAgents} \lambda_{i,t} = 4 \betadensity[t] \DiskCoverageRatio \sum_{i=1}^{\numOfAgents} \sigma^2_{\density} \noisedensity \lambda_{i,t} \label{eqn: sig-kernel-trick} \numberthis \\
    & \leq 4 \betadensity[t] \DiskCoverageRatio \sum_{i=1}^{\numOfAgents} \sigma^2_{\density} C_1 \log ( 1 + \noisedensity \lambda_{i,t}) \label{eqn: log-ineq} \numberthis\\
    &\leq  \frac{8 \DiskCoverageRatio \numOfAgents \betadensity[t]}{\log(1+\numOfAgents\noisedensity)} \sum_{i=1}^{\numOfAgents} \frac{1}{2} \log ( 1 + \noisedensity \lambda_{i,t})  \numberthis \label{eqn: max-single-bound_sq_reg}\\
\end{align*}   
\cref{eqn: CS-ineq} follows from Cauchy-Schwarz inequality and using that $\sum_{i=1}^{\numOfAgents} |\Discat[i-]| \leq |\Domain|$. \cref{eqn: max-ineq} follows since $\sum_{\PtInDomain\in \Discat[i-]_{t}} (\sigdensity[t-1](\PtInDomain))^2 \leq |\Discat[i-]| \max_{\PtInDomain\in \Discat[i-]_{t}} (\sigdensity[t-1](\PtInDomain))^2$. We define $\DiskCoverageRatio = \frac{\max_{i} |\Discat[i]|}{|\Domain|}$ denoting maximum coverage fraction possible by a disk. \cref{eqn: sig-kernel-trick} follows from $\sum_{i=1}^{\numOfAgents} (\sigdensity[t-1](\LocAgent[g,i]{t}))^2 = \trace(\densitykernel) = \sum_{i=1}^{\numOfAgents} \lambda_{i,t}$. Since, $s \leq C_1 \log(1+s) $ for $ s \in [0, \numOfAgents \noisedensity],$ where $ C_1 = \numOfAgents\noisedensity/\log(1+\numOfAgents\noisedensity) \geq 1$.\cref{eqn: log-ineq} follows for $s = \noisedensity \lambda_{i,t} \leq \noisedensity \lambda_{\max} \leq \noisedensity \sum_i \lambda_{i,t} = \noisedensity \trace(\densitykernel) \leq \noisedensity \numOfAgents, \; (wlog \; \kernelfunc(v,v) \leq 1)$. 

From \cref{eqn: max-single-bound_agent_wise},
\begin{align*}
    % \Big( \sum_{i=1}^{\numOfAgents} \LocReg(T) \Big)^2 &\leq \DiskCoverageRatio \Big(\sum_{t=1}^T  \sumMaxWidth{t}\Big)^2 \leq \DiskCoverageRatio T \sum_{t=1}^T  \sumMaxWidth[2]{t} \tag{Using Cauchy-Schwarz inequality}\\
    \Big( \sum_{i=1}^{\numOfAgents} \LocReg(T) \Big)^2 &\leq T \sum_{t=1}^{T} \Big( 2 \sqrt{\betadensity[t]} \sum_{i=1}^{\numOfAgents}  \sum_{\PtInDomain \in \Discat[i-]_t} \sigdensity[t-1](\PtInDomain) / \factor \Big)^2 \tag{Cauchy-Schwarz inequality}\\
    &\leq T  \sum_{t=1}^T \frac{8 \DiskCoverageRatio \numOfAgents \betadensity[t]}{\log(1+\numOfAgents\noisedensity)} \sum_{i=1}^{\numOfAgents} \frac{1}{2} \log ( 1 + \noisedensity \lambda_{i,t})   \tag{Using \cref{eqn: max-single-bound_sq_reg}}\\
    &=  \frac{8  \DiskCoverageRatio  \numOfAgents T \betadensity[T]}{\log(1+\numOfAgents\noisedensity)}  I(Y_{\LocAgent[g, 1:\numOfAgents]{1:T}};\density) \tag{Since $\betadensity[t]$ is non-decreasing, using \cref{eqn: macopt-information-gain}} \\
    &\leq  \frac{8  \DiskCoverageRatio \numOfAgents  T  \betadensity[T] \gammadensity{\numOfAgents T} }{\log(1+\numOfAgents\noisedensity)}  \tag{$\gammadensity{\numOfAgents T} = \sup_{{\LocAgent[g, 1:\numOfAgents]{1:T}} \subset V} I(Y_{\LocAgent[g, 1:\numOfAgents]{1:T}};\density)$} \\
    \implies \sum_{i=1}^{\numOfAgents} \LocReg(T) &\leq \sqrt{ \frac{8 \DiskCoverageRatio  \numOfAgents T \betadensity[T] \gammadensity{\numOfAgents T}}{\log(1+\numOfAgents\noisedensity)}}  \numberthis \label{eqn: macopt-CR-bound}
\end{align*}
Hence Proved.
\end{proof}
\cref{thm: macopt} follows from \cref{lem: bound-agent-wise-regret-with-sigma}, \cref{lem: reg-bound-info} and \cref{eqn: bound-act-reg-agentwise-reg},
\begin{align*}
    \actualRegret(T) \leq \sum_{i=1}^{\numOfAgents} \LocReg(T) \leq \sqrt{\frac{8  \DiskCoverageRatio \numOfAgents T   \betadensity[T] \gammadensity{\numOfAgents T}}{\log(1+\numOfAgents\noisedensity)}}
\end{align*}
% \manish{
% \begin{align*}
%     \actualRegret(T) \leq \sum_{i=1}^{\numOfAgents} \LocReg(T) \leq  \sqrt{\frac{ \numOfAgents N 8 T \betadensity[T] \gammadensity{\numOfAgents T}}{|\Domain|\log(1+\numOfAgents\noisedensity)}}
% \end{align*}
% }
% \todo{May be take it to the mutual info section}
% \begin{itemize}
%     \item Assuming exponential eigen decay
%     \begin{align*}
%     \actualRegret &\leq \numOfAgents \sqrt{\frac{8 T  \betadensity[T]}{\log(1+\numOfAgents\noisedensity)}  \bigg( \Big(\frac{2}{\expExponentCoeff} \big( \log(\numOfAgents T) + C_{\expExponentExponent} 
%     \big) \Big)^{\frac{1}{\expExponentExponent}} + 1 \bigg) \log( 1 + \noisedensity  \numOfAgents  T \bar{\funckernel})}\\
% \actualRegret &= \mathcal{O} \big( \numOfAgents T^{\frac{1}{2}}  \log^{\frac{1}{\expExponentExponent}}( \numOfAgents  )\log^{1 + \frac{1}{\expExponentExponent}}(  T ) \big)
% \end{align*}
%     \item Assuming polynomial eigen decay
% \begin{align*}
%     \actualRegret &\leq \numOfAgents \sqrt{\frac{8 T  \betadensity[T]}{\log(1+\numOfAgents\noisedensity)}   \Big( (\polyCoeff \bar{\phi}^2 \numOfAgents T)^{\frac{1}{\polyExponent}}
% \sigma^{\frac{-2}{\polyExponent}} \log^{\frac{-1}{\polyExponent}} ( 1 + \noisedensity \numOfAgents  T \bar{\funckernel}) + 1 \Big) \log( 1 + \noisedensity  \numOfAgents  T \bar{\funckernel})} \\
% \actualRegret &= \mathcal{O} \big( \numOfAgents^{1 + \frac{ 1}{2 \polyExponent}}T^{\frac{\polyExponent + 1}{2 \polyExponent}} \log^{\frac{- 1}{2 \polyExponent}}( \numOfAgents ) \log^{\frac{\polyExponent - 1}{2 \polyExponent}}( T ) \big)
% \end{align*}
% \end{itemize}

\textbf{Proof for the corollary 1:}
\restatemacoptcorollary
\begin{proof} The proof for the corollary goes in the following 2 steps. First, we show that once $\sumMaxWidth[]{t} \leq \epsdensity$ implies $\Objfunc{\LocAgents_{t}}{\density}{\Domain} \geq (1- \frac{1}{e}) \Objfunc{\LocAgents_{\star}}{\density}{\Domain} - \epsilon_{\density}$. Secondly, in \cref{lem: macopt-finite-time-bound} we show \macopt achieves $\sumMaxWidth[]{t} \leq \epsdensity$, at $t < \tdensity^{\star}$ where $\tdensity^{\star}$ be the smallest integer satisfying $\frac{\tdensity^{\star}}{\beta_{\tdensity^{\star}} \gamma_{\numOfAgents\tdensity^{\star}}} \leq \frac{8 \DiskCoverageRatio^2 \numOfAgents^2 }{ \log(1+\numOfAgents\sigma^{-2}) \epsdensity^2}$.

% % \begin{lemma}
% For any $t \geq 1$, if $\sum_i^{\numOfAgents} 2 \sqrt{\betadensity[t]} \max_{\PtInDomain \in \Discat[i-]_t} \sigma_{\density_{t -1}}(\PtInDomain) \leq \epsdensity$ the following holds,
% \begin{align*}
%     \Obj(X_{t}) \geq (1- \frac{1}{e}) \Obj(X^{\star}) - \epsilon_{\density}
% \end{align*}
% % \end{lemma}

Similar to steps in \cref{lem: bound-agent-wise-regret-with-sigma} for a fix t, (\cref{eqn: bound-agentwise-reg}), we get
\begin{align*}
    \Delta(\tilde{\LocAgent[]{}}^i|\LocAgents^{1:i-1}_t) - \Delta(\LocAgent[i]{t}|\LocAgents^{1:i-1}_t) & \leq 2 \sqrt{\betadensity[t]} \sum_{\PtInDomain \in \Discat[i-]_t} \sigdensity[t-1](\PtInDomain) / \factor \\
    &\leq  2 \sqrt{\betadensity[t]} \DiskCoverageRatio \max_{\PtInDomain \in \Discat[i-]_t} \sigdensity[t-1](\PtInDomain) 
\end{align*}
From \cref{eqn: sim-reg-const-bound} (for constrained case) one can show for unconstrained case,
\begin{align*}
    (1- \frac{1}{e}) \Objfunc{\LocAgents_{\star}}{\density}{\Domain} - \sum_i^{\numOfAgents} \Delta(\LocAgent[i]{t}|X^{1:i-1}_t) 
     & \leq \sum_i^{\numOfAgents}  \Delta(\tilde{\LocAgent[]{}}^i|X^{1:i-1}_t) - \Delta(\LocAgent[i]{t}|X^{1:i-1}_t) \\ 
       &\leq \sum_i^{\numOfAgents} 2 \sqrt{\betadensity[t]} \DiskCoverageRatio \max_{\PtInDomain \in \Discat[i-]_t} \sigma_{\density_{t -1}}(\PtInDomain) \leq \epsdensity \\
     \implies \Objfunc{\LocAgents_{t}}{\density}{\Domain} &\geq (1- \frac{1}{e}) \Objfunc{\LocAgents_{\star}}{\density}{\Domain} - \epsilon_{\density}
\end{align*}
\end{proof}

\begin{lemma}\label{lem: macopt-finite-time-bound}
Let $\delta \in (0,1)$ and $\betadensity[t]$ as in \citep{beta-chowdhury17a}, i.e., ${\beta^{\density}_t}^{1/2} = B_{\density} + 4 \sigma_{\density} \sqrt{\gamma^{\density}_{\numOfAgents t} + 1 + \ln(1/\delta)}$ and $\tdensity^{\star}$ is the smallest integer such that $\frac{\tdensity^{\star}}{\beta_{\tdensity^{\star}} \gamma_{\numOfAgents \tdensity^{\star}}} \geq \frac{8 \DiskCoverageRatio^2 \numOfAgents^2}{ \log(1+\numOfAgents\sigma^{-2}) \epsdensity^2}$, then with probability $1-\delta$ that there exists $\tdensity < \tdensity^{\star}$ such that $\sumMaxwidth{\tdensity+1} \leq \epsdensity$, where $\sumMaxwidth{t} =  2 \sqrt{\betadensity[t]} \DiskCoverageRatio \sum\limits_{i =1}^{\numOfAgents}  \max\limits_{\PtInDomain\in \Discat[i-]_{t}} \sigdensity[t-1](\PtInDomain) \leq \epsdensity.$ 
\end{lemma}

\begin{proof} It is defined that,
\begin{align*}
 \sumMaxWidth[]{t} &\coloneqq 2 \sqrt{\betadensity[t]} \DiskCoverageRatio \sum\limits_{i =1}^{\numOfAgents}  \max\limits_{\PtInDomain\in \Discat[i-]_{t}} \sigdensity[t-1](\PtInDomain) \\
 \implies \sumMaxWidth[2]{t} &\leq  4 \betadensity[t] \DiskCoverageRatio^2 \numOfAgents \sum_{i=1}^{\numOfAgents} \Big( \sigdensity[t-1](\LocAgent[g,i]{t}) \Big)^2 \numberthis \label{eqn: macopt-time-bound-w-sq}\\
 \implies \Big( \sum_{t=1}^{T} \sumMaxWidth[]{t} \Big)^2 \leq T \sum_{t=1}^{T} \sumMaxWidth[2]{t} &\leq (\DiskCoverageRatio \numOfAgents) T \sum_{t=1}^{T} 4 \betadensity[t] \DiskCoverageRatio \sum_{i=1}^{\numOfAgents} \Big( \sigdensity[t-1](\LocAgent[g,i]{t}) \Big)^2 \numberthis \label{eqn: w-sq-bound-cumm}\\
 &\leq (\DiskCoverageRatio \numOfAgents)  \frac{8 \DiskCoverageRatio  \numOfAgents T \betadensity[T] \gammadensity{\numOfAgents T}}{\log(1+\numOfAgents\noisedensity)} \numberthis \label{eqn: macopt-time-bound-w-sq-cum}
\end{align*}
\cref{eqn: macopt-time-bound-w-sq} follows from Cauchy-Schwarz inequality and $\LocAgent[g,i]{t} = \argmax\limits_{\PtInDomain \in \Discat[i-]_t} \sigma_{\density_{t-1}}(\PtInDomain)$. The RHS of \cref{eqn: macopt-time-bound-w-sq} resembles \cref{eqn: w-sq-bound} in \cref{lem: reg-bound-info}, with an additional factor of $(\DiskCoverageRatio \numOfAgents)$. \cref{eqn: w-sq-bound-cumm} directly follows from Cauchy-Schwarz inequality and \cref{eqn: macopt-time-bound-w-sq}. Following the steps in \cref{lem: reg-bound-info} will result in \cref{eqn: macopt-time-bound-w-sq-cum}. 

Since, 
\begin{align*}
\frac{\tdensity^{\star}}{\beta_{\tdensity^{\star}} \gammadensity{\numOfAgents \tdensity^{\star}}} \geq \frac{8 \DiskCoverageRatio^2 \numOfAgents^2  }{ \log(1+\numOfAgents\sigma^{-2}) \epsdensity^2} \\
\implies (\DiskCoverageRatio \numOfAgents)^{1/2} \sqrt{\frac{8  \DiskCoverageRatio \numOfAgents \beta_{\tdensity^{\star}} \gammadensity{\numOfAgents \tdensity^{\star}}}{\tdensity^{\star} \log(1+\numOfAgents\sigma^{-2})}} &\leq \epsdensity \tag{Rearranging terms}\\
\frac{\sum_{t=1}^{\tdensity^{\star}} \sumMaxwidth{t}}{\tdensity^{\star}} \leq (\DiskCoverageRatio \numOfAgents)^{1/2} \sqrt{\frac{8 \DiskCoverageRatio \numOfAgents \beta_{\tdensity^{\star}} \gammadensity{\numOfAgents \tdensity^{\star}}}{\tdensity^{\star} \log(1+\numOfAgents\sigma^{-2})}}  &\leq  \epsdensity \tag{From \cref{eqn: macopt-time-bound-w-sq-cum}}\\
\implies \min_{t \in [1, \tdensity^{\star}]}  \sumMaxwidth{t} &\leq   \epsdensity \tag{$\frac{\tdensity^{\star} \min\limits_{t \in [1, \tdensity^{\star}]} \sumMaxwidth{t}}{\tdensity^{\star}} \leq \frac{\sum_{t=1}^{\tdensity^{\star}} \sumMaxwidth{t}}{\tdensity^{\star}}$}
\end{align*}
Hence there exists $\tdensity < \tdensity^{\star}$, such that $\sumMaxwidth{\tdensity +1} \leq \epsdensity$.

\end{proof}

\subsection{Variants of \macopt} \label{sec: variants-macopt}
\begin{itemize}
   \item \emph{Hallucinated uncertainty sampling:} Let $M_t$ and $H_t$ be the sets of measurements collected by \macopt and \macopt-H respectively at time $t$. In any iteration, $I(Y_{M_t};\density) \leq I(Y_{\star};\density)$, where $I(Y_{\star};\density)$ is the maximum information under the disc constraints. Since mutual information is a submodular function \citep{beta-srinivas}, it is a typical submodular maximization under partition matroid constraint (disc constraint). The greedy algorithm (Hallucianted strategy) yields a $1/2-$ times the optimal solution \citep{Nemhauser-1minus1by-e}. Hence, using $I(Y_{H_t};\density) \geq 1/2 I(Y_{\star};\density)$, we get $I(Y_{M_t};\density) \leq 2 I(Y_{H_t};\density)$. Analogous to \cref{lem: reg-bound-info}, for \macopt-H we obtain,
% \begin{align*}
%     (\actualRegret(T))^2 \leq \frac{8 T \DiskCoverageRatio \numOfAgents \betadensity[T]}{\log(1+\numOfAgents\noisedensity)}  I(Y_{g};\density) \leq  \frac{8 T \DiskCoverageRatio \numOfAgents \betadensity[T]}{\log(1+\numOfAgents\noisedensity)}  I(Y_{\star};\density) \leq  \frac{8 T \DiskCoverageRatio \numOfAgents \betadensity[T]}{\log(1+\numOfAgents\noisedensity)}  \frac{I(Y_{m};\density)}{1-1/e}
% \end{align*}
\begin{align*}
    \actualRegret^{H}(T) \leq \sqrt{\frac{16  \DiskCoverageRatio \numOfAgents T   \betadensity[T] \gammadensity{\numOfAgents T}}{\log(1+\numOfAgents\noisedensity)}}
\end{align*}
The regret bound worsens by two folds to account for the greedy selection in a partition matroid constraint. But practically, it can improve sample efficiency in environments with high coverage to domain $(\DiskCoverageRatio)$ ratio.

\item \emph{Correlated upper bound:} The coverage function is a linear functional of density. For any agent $i$, $\Obj(\{\LocAgent[i]{}\}) = \sum_{\PtInDomain \in \Discat[i-]} \density(\PtInDomain)$. Since density is correlated, we can construct a tighter upper bound(in contrast to the sum of density \ucb) of the coverage function utilizing the covariance of density. Practically, the sampling rule is given by $\LocAgent[i]{t} = \argmax_{\PtInDomain}\sum_{\PtInDomain \in \Discat[i-]_t}  \mudensity[t-1](\PtInDomain) + \sqrt{\beta_{t}'} \sqrt{\sum_{\PtInDomain \in \Discat[i-]_t} \sigma^2_{t-1}(\PtInDomain) + \sum_{\PtInDomain, \PtInDomain' \in \Discat[i-]_t} \sigma_{t-1}(\PtInDomain,\PtInDomain')}$. We believe theoretical analysis for constructing confidence bounds for linear functionals of a sample from RKHS can be carried out utilizing ideas from \citet{mutny2022experimental}. % upper bound of the coverage
% \begin{align*}
%     \LocAgent[i]{t} &= \argmax_{\PtInDomain}\sum_{\PtInDomain \in \Discat[i-]_t}  \mudensity[t-1](\PtInDomain) + \sqrt{\beta_{t}'} \sqrt{\sum_{\PtInDomain \in \Discat[i-]_t} \sigma^2_{t-1}(\PtInDomain) + \sum_{\PtInDomain, \PtInDomain' \in \Discat[i-]_t} \sigma_{t-1}(\PtInDomain,\PtInDomain') }\\
%     \beta_{t}' &= 2 \log(|\Domain|\pi_t/\delta), \sum_{t\geq1}\pi_t=1, \pi_t>0
% \end{align*}
% All proof follows since,
% \begin{align*}
%     \sqrt{\sum_{\PtInDomain \in \Discat[i-]_t} \sigma^2_{t-1}(\PtInDomain) + \sum_{\PtInDomain, \PtInDomain' \in \Discat[i-]_t} \sigma_{t-1}(\PtInDomain,\PtInDomain') } \leq \sum_{\PtInDomain \in \Discat[i-]_t} \sigdensity[t-1](\PtInDomain)
% \end{align*}

\end{itemize}

% We know that from \cref{thm: macopt}, $\sum_{t=1}^{T} r_t = \LocReg(T)$ is sublinear in T, i.e, $\LocReg(T)/T \to 0 ~\textit{as}~ T \to \infty$.  We can quantifying the time, before which the \macopt stopping criteria ($r_t \leq \epsdensity$) much have been satisfied.\\
% \begin{align*}
% \frac{\LocReg(t^{\star})}{t^{\star}} &\leq \epsdensity \\
% \numOfAgents \sqrt{\frac{8  \beta_{t^{\star}} \gammadensity{\numOfAgents t^{\star}}}{t^{\star} \log(1+\numOfAgents\noisedensity)}} &\leq \epsdensity \tag{Substituting $CR/T$ from \cref{thm: macopt}}\\
% \tag{On simplification}\\
% \frac{t^{\star}}{\beta_{t^{\star}} I(Y_{\LocAgent[m,1:\numOfAgents]{1:t^{\star}}};\density)} &\geq \frac{8 \numOfAgents^2 }{ \log(1+\numOfAgents\noisedensity) \epsdensity^2} \numberthis \label{eqn: t-bound}
% \end{align*}
% Since, \macopt terminates when uncertainty is below $\epsdensity$ and does not collect measurement at that time, $t < t^{\star}$. 
% \end{proof}
\newpage
\section{Proof. for \texorpdfstring{\cref{thm: safeMac}}{Lg} (\safemac)}
\label{Apx: thm-safemac}

\restatesafemac

\begin{proof}
The proof for \cref{thm: safeMac} goes in the following two steps:
\begin{enumerate}
\item \safemac's coverage is near-optimal at the convergence 
    \begin{itemize}
        \item We first bound the actual regret with the sum of per agent regret in \cref{lem: opti-bound-act-reg}. Precisely, we show the following (\cref{eqn: opti-cum-regret-bound}), $$\optiRegret(\Tdensity) \leq \optiLocReg(\Tdensity) $$ 
        \item Next, we establish in \cref{lem: const-CR} that the $\optiLocReg(\Tdensity)$ grows sublinear with the density measurements.
        \item Next, we show that if $\sumMaxwidth[]{t}<\epsdensity$, the coverage is near optimal (\cref{lem: const-optimal-eps}). The condition $\sumMaxwidth[]{t}<\epsdensity$ will eventually happen since $\optiLocReg(\Tdensity)$ is sublinear and hence over time will shrink to zero. 
        \item Finally using \cref{lem: safemac-convergence-set-equality}, the near optimality in the pessimistic set can be established at convergence when the $2^{nd}$ termination condition is satisfied, precisely $ \{ \optiSet[,i]{t} \backslash \pessiSet[,i]{t}) \cap  \Discat[i]_t , \forall i \in [\numOfAgents] \} = \emptyset$
    \end{itemize}
\item \safemac converges in a finite time $t < t^{\star}_\constrain + t^{\star}_{\density}$, where $t^{\star}_\density$ be the smallest integer such that $\frac{t^{\star}_{\density}}{\beta_{t^{\star}_{\density}} \gamma_{\numOfAgents t^{\star}_{\density} }} \geq \frac{8 \DiskCoverageRatio^2 \numOfAgents^2 }{ \log(1+\numOfAgents \noisedensity) \epsdensity^2}$ and $\tconst^{\star}$ be the smallest integer such that $\frac{\tconst^{\star}}{\beta_{\tconst^{\star}} \gamma_{\numOfAgents \tconst^{\star}}} \geq \frac{C_1 |\RbarO{\LocAgents_0}|}{\epsconst^2}$, with $C_1 = 8/\log(1+ \noiseconst)$.
\begin{itemize}
    \item Since \safemac runs by iterating between the coverage and the exploration phase, we decouple it and analyze both the phases separately. Starting with the \emph{coverage phase}, In \cref{lem: const-time-bound}, we establish a bound on density samples required to terminate the first coverage phase
    \item Next, in the \cref{lem: delta-time-regret-bound}, we show that cumulative regret grows sublinear with the density measurements for any coverage phase and utilizes this to bound the density samples between two consecutive coverage phases in \cref{lem: delta-time-bound} 
    \item Utilizing the above two statements, we present the sample complexity bound to terminate the $n^{th}$ coverage phase till convergence, using that the information gain is additive for consecutive coverage phases in \cref{lem: additive-info-time-bound} 
    \item For the \emph{exploration phase}, the worst case time complexity bound is given by the multi-agent version of the \goose in \cref{lem: ma-goose} when the agents safely explore the complete domain. The resulting worst case time bound for \safemac is sum of the time bound of the \textit{coverage} and the \textit{exploration} phase.
\end{itemize}
\end{enumerate}
So, near optimality at convergence in \cref{thm: safeMac} is a direct consequence of \cref{lem: const-optimal-eps} and \cref{lem: safemac-convergence-set-equality} and the finite time argument of \cref{thm: safeMac} is a direct consequence of \cref{lem: additive-info-time-bound} and \cref{lem: ma-goose}.
\end{proof}
% Finally, since $\optiLocReg(\Tdensity)$ is sublinear, we know that cummulative regret over time will shrink to zero. 
% The proof of \cref{thm: safeMac} goes in following steps:
% \begin{itemize}
% \item We first use the analysis of \cref{sec: opti-bound-act-reg}, to bound the actual regret with agent wise regret, precisely from \cref{eqn: opti-cum-regret-bound}, we know that
% $$\optiRegret(\Tdensity) \leq \optiLocReg(\Tdensity) $$
% \item This lemma nicely connects/links the near optimal coverage in the reachable set i.e, $(1- \frac{1}{e})  \sum_{\batch \in \BatchColl{}} \Objfunc{\LocAgents^{\batch}_{\star}}{\density}{\Rbareps{}{\LocAgents^{\batch}_0}}$,  with the coverage in a possibly disjoint optimistic sets. (Note that the  only requirement is that optimistic set needs to always super set $\Rbar$. (and We don't have to worry about the change/shrinkage of the optimistic set). Also the equation above shows that it is sufficient to bound the agent wise regret. 
% \item 
% \end{itemize}
% % \begin{lemma} \label{lem: opti-slot-regret-bound}
\begin{lemma} \label{lem: const-CR}
Let $\delta \in (0,1)$ and $\betadensity[t]$ as in \citep{beta-chowdhury17a}, i.e., ${\betadensity[t]}^{1/2} = B_{\density} + 4 \sigma \sqrt{\gammadensity{t} + 1 + \ln(1/\delta)}$. With probability at least $1 - \delta$, \safemac's sum of per agent regret $\optiLocReg(\Tdensity)$ is bounded by $\mathcal{O}(\sqrt{\Tdensity \betadensity[T] \gammadensity{\numOfAgents T}})$. Precisely,
\begin{align*}
    \optiLocReg(\Tdensity) \leq \sqrt{\frac{8 \DiskCoverageRatio \numOfAgents \Tdensity  \betadensity[t] \gammadensity{\numOfAgents T}}{\log(1+\numOfAgents\noisedensity)}  } 
\end{align*}
where $\Tdensity$ is density samples per agent and $\optiLocReg(\Tdensity) = \sum_{t=1}^{\Tdensity} \simReg{}$ where $\simReg{} = \sum_{\batch \in \BatchColl{t}} \sum_{i\in \batch}  \delgain{\tilde{\LocAgent[]{}}^i|\LocAgents^{1:i-1}_t}{\density}{\uniSet[, \batch]{t}} -  \delgain{\LocAgent[i]{t}|\LocAgents^{1:i-1}_t}{\density}{\uniSet[, \batch]{t}} $ 
\end{lemma}
\begin{proof}
\mypar{Given}
\begin{align*}
    \optiLocReg(\Tdensity) &= \sum_{t=1}^{\Tdensity} \simReg{}\\
    &= \sum_{t=1}^{\Tdensity} \sum_{\batch \in \BatchColl{t}} \sum_{i\in \batch} \delgain{\tilde{\LocAgent[]{}}^i|\LocAgents^{1:i-1}_t}{\density}{\uniSet[, \batch]{t}} -  \delgain{\LocAgent[i]{t}|\LocAgents^{1:i-1}_t}{\density}{\uniSet[, \batch]{t}} 
\end{align*}
WLOG, every batch $\batch$, is indexed by iterator $i \ = \ 1 \ to \ |\batch|$ sequentially.\\
Let $\tilde{\LocAgent[]{}}^i = \argmax \delgain{\LocAgent[i]{t}|\LocAgents^{1:i-1}_t}{\density}{\uniSet[, \batch]{t}}$ and $\tlDiscat[i]_{t}$ is a disk around $\tilde{\LocAgent[]{}}^i$.
For notation convenience: $\Discat[i-]_{t} := \Discat[i]_{t} \backslash \Discat[1:i-1]_{t}\cap {\uniSet[,\batch]{{t}}}$ and $\tlDiscat[i-]_{t} := \tlDiscat[i]_{t} \backslash \Discat[1:i-1]_{t}\cap {\uniSet[,\batch]{{t}}}$\\
% \todo{add some definition of mutual info}

\safemac picks the agent at $\LocAgent[i]{t}$ greedily in the set \batch. Following the steps in \cref{lem: const-optimal-eps} we can bound simple agent-wise local regret as $\simReg{}$ or simply from \cref{eqn: const-gap-marginal-gain} by summing over all the $\batch \in \BatchColl{t}$, we get,
\begin{align*}
    \simReg{} = \sum_{\batch \in \BatchColl{t}} \sum_{i\in \batch} \delgain{\tilde{x}^i_{t}|\LocAgents^{1:i-1}_{t}}{&\density}{\uniSet[,\batch]{{t}}} - \delgain{\LocAgent[i]{t}|\LocAgents^{1:i-1}_{t}}{\density}{\uniSet[,\batch]{{t}}} \\
    &\leq \sum_{\batch \in \BatchColl{t}} \sum_{i\in \batch}  2 \sqrt{\betadensity[t]} \sum_{\PtInDomain \in \Discat[i-]_t} \sigdensity[t-1](\PtInDomain) / \factor \tag{From \cref{eqn: const-margin-gain-factor}} \\
&\leq  \sum_{\batch \in \BatchColl{t}} \sum_{i\in \batch} 2 \sqrt{\betadensity[t]} \DiskCoverageRatio \max_{\PtInDomain\in \Discat[i-]_{t}} \sigdensity[t-1](\PtInDomain) = \sumMaxwidth[]{t} \tag{From \cref{eqn: const-gap-marginal-gain}}
\end{align*}
\mypar{On bounding simple regret}
\begin{align*}
    \simReg{} &\leq \sum_{\batch \in \BatchColl{t}} \sum_{i\in \batch} 2 \sqrt{\betadensity[t]} \sum_{\PtInDomain \in \Discat[i-]_t} \sigdensity[t-1](\PtInDomain) / \factor \\
   \implies \simReg{2} &\leq \Big(\sum_{\batch \in \BatchColl{t}} \sum_{i\in \batch} 2 \sqrt{\betadensity[t]} \sum_{\PtInDomain \in \Discat[i-]_t} \sigdensity[t-1](\PtInDomain) / \factor \Big)^2 \\
   &\leq 4 \betadensity[t]  \sum_{\batch \in \BatchColl{t}} \sum_{i\in \batch} \sum_{\PtInDomain \in \Discat[i-]_t} \Big( \sigdensity[t-1](\PtInDomain) \Big)^2 / \factor = 4 \betadensity[t]  \sum_{i=1}^{\numOfAgents}  \sum_{\PtInDomain \in \Discat[i-]_t} \Big( \sigdensity[t-1](\PtInDomain) \Big)^2 / \factor \label{eqn: batchCauchyschwarz} \numberthis \\
    &\leq  \frac{8 \DiskCoverageRatio \numOfAgents \betadensity[t]}{\log(1+\numOfAgents\noisedensity)} \sum_{i=1}^{\numOfAgents} \frac{1}{2} \log ( 1 + \noisedensity \lambda_{i,t})  \numberthis \label{eqn: reg-square-bound-constrained}
    % \sumMaxwidth[]{t} = \sum_{\batch \in \BatchColl{t}} \sum_{i\in \batch} 2 \sqrt{\betadensity[t]} \max_{\PtInDomain\in \Discat[i-]_{t}} \sigdensity[t-1](\PtInDomain) = 2 \sqrt{\betadensity[t]} \sum_{i=1}^{\numOfAgents} \sigdensity[t-1](\LocAgent[g,i]{t}) \tag{$\LocAgent[g,i]{t} = \argmax\limits_{\PtInDomain\in \Discat[i-]_{t}} \sigdensity[t-1](\PtInDomain)$} \\
    % %
    % \sumMaxwidth[2]{t} &\leq 4 \betadensity[t] \numOfAgents \sum_{i=1}^{\numOfAgents} (\sigdensity[t-1](\LocAgent[g,i]{t}))^2  \tag{Using Cauchy–Schwarz inequality} \\
    % %
    % &= 4 \betadensity[t] \numOfAgents \sum_{i=1}^{\numOfAgents} \lambda_{i,t} \tag{$\sum_{i=1}^{\numOfAgents} (\sigdensity[t-1](\LocAgent[g,i]{t}))^2 = \trace(\densitykernel) = \sum_{i=1}^{\numOfAgents} \lambda_{i,t} $  }\\
    % &= 4 \betadensity[t] \numOfAgents \sum_{i=1}^{\numOfAgents} \sigma^2 \noisedensity \lambda_{i,t} \\
    % &\leq 4 \betadensity[t] \numOfAgents \sum_{i=1}^{\numOfAgents} \sigma^2 C_2 \log ( 1 + \noisedensity \lambda_{i,t}) 
    % \tag{Since, $s \leq C_2 \log(1+s) $ for $ s \in [0, \numOfAgents \noisedensity],$ where $ C_2 = \numOfAgents\noisedensity/\log(1+\numOfAgents\noisedensity) \geq 1$}\\
    % \tag{Here, $s = \noisedensity \lambda_{i,t} \leq \noisedensity \lambda_{\max} \leq \noisedensity \sum_i \lambda_{i,t} = \noisedensity \trace(\densitykernel) \leq \noisedensity \numOfAgents, \; (wlog \; \kernelfunc(v,v) \leq 1)$ }\\
    % &\leq  \frac{8 \numOfAgents^2 \betadensity[t]}{\log(1+\numOfAgents\noisedensity)} \sum_{i=1}^{\numOfAgents} \frac{1}{2} \log ( 1 + \noisedensity \lambda_{i,t})  \label{eqn: reg-square-bound-constrained} \numberthis
\end{align*}

\cref{eqn: batchCauchyschwarz} follows by Cauchy–Schwarz inequality and $\sum_{\batch \in \BatchColl{t}} \sum_{i\in \batch} |\Discat[i-]_t| \leq |\Domain|$. \cref{eqn: reg-square-bound-constrained} follows the steps in \cref{eqn: CS-ineq,eqn: max-ineq,eqn: w-sq-bound,eqn: log-ineq,eqn: max-single-bound_sq_reg}.

% Note that the above inequality is written for a particular time when the utility measurement is obtained by all the agents, across two 2 instance of density measurement, over time t (expansion time) the optimistic set could shrink. %the measurements are obtained in $\uniSet[]{t}$  
\mypar{On bounding cumulative regret with mutual information}
\begin{align*}
\Big( \sum_{t=1}^{\Tdensity} \simReg{} \Big)^2 &\leq \Tdensity \sum_{t=1}^{\Tdensity} \simReg{2} \tag{Using Cauchy–Schwarz inequality} \\
    &\leq \Tdensity \sum_{t=1}^{\Tdensity} \frac{8 \DiskCoverageRatio \numOfAgents \betadensity[t]}{\log(1+\numOfAgents\noisedensity)} \sum_{i=1}^{\numOfAgents} \frac{1}{2} \log ( 1 + \noisedensity \lambda_{i,t}) \tag{Using \cref{eqn: reg-square-bound-constrained}}\\
    &=\frac{8 \DiskCoverageRatio \numOfAgents\Tdensity  \betadensity[T]}{\log(1+\numOfAgents\noisedensity)} \sum_{t=1}^{\Tdensity}  \sum_{i=1}^{\numOfAgents} \frac{1}{2} \log ( 1 + \noisedensity \lambda_{i,t}) \tag{Since $\betadensity[t]$ is non-decreasing \& $\betadensity[T] := \betadensity[\Tdensity]$ } \\
    &= \frac{8 \DiskCoverageRatio \numOfAgents\Tdensity  \betadensity[T] I(Y_{\LocAgent[g,1:\numOfAgents]{1:\Tdensity}};\density)}{\log(1+\numOfAgents\noisedensity)}  \tag{Using \cref{eqn: macopt-information-gain}} \\
    &\leq \frac{8 \DiskCoverageRatio \numOfAgents\Tdensity  \betadensity[T] \gammadensity{\numOfAgents T}}{\log(1+\numOfAgents\noisedensity)} \tag{$\gammadensity{\numOfAgents T} = \sup_{\LocAgents^m_{1:\Tdensity} \subset \Domain} I(Y_{\LocAgents^m_{1:\Tdensity}};\density)$ }\\
    \implies \sum_{t=1}^{\Tdensity} \simReg{} &\leq \sqrt{\frac{8 \DiskCoverageRatio \numOfAgents\Tdensity  \betadensity[T] \gammadensity{\numOfAgents T}}{\log(1+\numOfAgents\noisedensity)}} \numberthis \label{eqn: width-CR-sublinear-bound}\\
% \optiLocReg(\Tdensity) &= \sum_{t=1}^{\Tdensity} \simReg{} \leq \sum_{t=1}^{\Tdensity} \sumMaxwidth{t} \tag{From definition of $\optiLocReg(\Tdensity)$}\\
   \implies \optiLocReg(\Tdensity) &\leq \sqrt{\frac{8 \DiskCoverageRatio \numOfAgents\Tdensity  \betadensity[T] \gammadensity{\numOfAgents T}}{\log(1+\numOfAgents\noisedensity)}} 
\end{align*}
\end{proof}
% \todo{May be remove the para below}
This lemma nicely connects the near optimal coverage in the reachable set i.e, $(1- \frac{1}{e})  \sum_{\batch \in \BatchColl{}} \Objfunc{\LocAgents^{\batch}_{\star}}{\density}{\Rbareps{}{\LocAgents^{\batch}_0}}$,  with the coverage in a possibly disjoint optimistic sets. (Note that the only requirement is that the optimistic set needs to always superset $\Rbar$.  % Also the equation above shows that it is sufficient to bound the agent wise regret. 

The agents observe the location only if all the agents can reach the max uncertain point under their disk i.e, $2 \sqrt{\betadensity[t]} \max_{\PtInDomain\in \Discat[i-]_{t}} \sigdensity[t-1](\PtInDomain)$. (Accordingly, information gain is defined, and $\Tdensity$ above is a counter when all the agents obtain density measurements). %If only a few of them and not all can reach the measurement point, then we re-expand and don't collect any measurements. If the agents can't reach the desired goal, it is guaranteed that the optimistic set will shrink.

\begin{lemma}[\safemac Near-Optimality]
\label{lem: const-optimal-eps}
For any $t \geq 1$, if $\sumMaxwidth[]{t} \leq \epsdensity$ at \safemac's recommendation $\LocAgents_t$ then with high probability,
\begin{align*}
    \sum_{\batch \in \BatchColl{t}} \Objfunc{\LocAgents_{t}^{\batch}}{\density}{\uniSet[,\batch]{t}} \geq (1- \frac{1}{e}) \sum_{\batch \in \BatchColl{}} \Objfunc{\LocAgents_{\star}^{\batch}}{\density}{\Rbareps{}{\LocAgents_{0}^{\batch}}} - \epsdensity,
\end{align*} 
where $\sumMaxwidth[]{t} = \sum_{\batch \in \BatchColl{t}} \sum_{i\in \batch}  2 \sqrt{\betadensity[t]} \DiskCoverageRatio \max_{\PtInDomain\in \Discat[i-]_{t}} \sigdensity[t-1](\PtInDomain)$.
\end{lemma}
\begin{proof}
\mypar{Given} \safemac recommends a location for the agent $i \in \batch$ greedily in the $\uniSet[,\batch]{t}$ set as per,
\begin{align*}
    \LocAgent[i]{t} = \argmax_{\PtInDomain} \sum_{\PtInDomain\in \Discat[i-]_{t}}  \mudensity[t-1](\PtInDomain) + \sqrt{\betadensity[t]} \sigdensity[t-1](\PtInDomain)  \numberthis \label{eqn: safe_pick_strategy}
\end{align*}
Let $\tilde{\LocAgent[]{}}^i_t = \argmax \delgain{\LocAgent[i]{t}|\LocAgents^{1:i-1}_t}{\density}{\uniSet[, \batch]{t}}$ and $\tlDiscat[i-]_{t} := \tlDiscat[i]_{t} \backslash \Discat[1:i-1]_{t}\cap {\uniSet[,\batch]{{t}}}$, where $\tlDiscat[i]_{t}$ is a disk around $\tilde{\LocAgent[]{}}^i$.  Based on this picking strategy, 
\begin{align*}
    \sum_{\PtInDomain\in \tlDiscat[i-]_t }\density(\PtInDomain) &\leq \sum_{\PtInDomain\in \tlDiscat[i-]_t } \big( \mudensity[t-1](\PtInDomain) + \sqrt{\betadensity[t]}  \sigdensity[t-1](\PtInDomain) \big) \tag{Follows due to upper confidence bound}\\
    &\leq \sum_{\PtInDomain\in \Discat[i-]_t } \big( \mudensity[t-1](\PtInDomain) + \sqrt{\betadensity[t]}  \sigdensity[t-1](\PtInDomain) \big) \tag{Since, \cref{eqn: safe_pick_strategy}, $\LocAgent[i]{t}$ is greedily picked}\\
    \sum_{\PtInDomain\in \tlDiscat[i-]_t}\density(\PtInDomain) &\leq \sum_{\PtInDomain\in \Discat[i-]_t } \big( \mudensity[t-1](\PtInDomain) + \sqrt{\betadensity[t]}  \sigdensity[t-1](\PtInDomain) \big) \numberthis \label{eqn: opti-xtilde_inequality}
\end{align*}
\mypar{On bounding simple regret}
With definition $\simReg{} = \sum_{\batch \in \BatchColl{t}} \sum_{i\in \batch} \delgain{\tilde{x}^i_{t}|\LocAgents^{1:i-1}_{t}}{\density}{\uniSet[,\batch]{{t}}} -  \delgain{\LocAgent[i]{t}|\LocAgents^{1:i-1}_{t}}{\density}{\uniSet[,\batch]{t}}$. \\
Consider,
\begin{align*}
\delgain{\tilde{x}^i_{t}|\LocAgents^{1:i-1}_{t}}{\density}{\uniSet[,\batch]{{t}}} - \,&\delgain{\LocAgent[i]{t}|\LocAgents^{1:i-1}_{t}}{\density}{\uniSet[,\batch]{t}} \\
% &= \Objfunc{\LocAgents^{1:i-1}_{t} \cup \{\tilde{x}^i_{t}\} }{\density}{\uniSet[,\batch]{{t}}} - \Objfunc{\LocAgents^{1:i-1}_{t} \cup \{\LocAgent[i]{t}\}}{\density}{\uniSet[,\batch]{{t}}} \tag{using defi. of $\delgain{.}{.}{.}$, $\Objfunc{\LocAgents^{1:i-1}_t}{\density}{\uniSet[,\batch]{t}}$ cancels out} \\ 
    &=  \Big( \sum_{\PtInDomain\in \tlDiscat[i-]_{t}} \density(\PtInDomain) - \sum_{\PtInDomain\in \Discat[i-]_{t} } \density(\PtInDomain)  \Big) / \factor 
    \tag{Note $\Discat[i-]_{t}$ and $\tlDiscat[i-]_{t}$}\\
    % \tag{using defi. of $\Objfunc{\LocAgents^{1:i-1}_{t}}{\density}{\uniSet[,\batch]{{t}}}$. Note $\Discat[i-]_{t}$ and $\tlDiscat[i-]_{t}$}\\
    &\leq \Big(   \sum_{\PtInDomain\in \Discat[i-]_{t}} \! \Big( \mudensity[t-1](\PtInDomain) + \sqrt{\betadensity[t]}  \sigdensity[t-1](\PtInDomain) \Big) - \! \sum_{\PtInDomain\in \Discat[i-]_{t}} \! \density(\PtInDomain)  \Big)/ \factor \tag{Using \cref{eqn: opti-xtilde_inequality}} \\
     &\leq  \Big( \!\!\!\! \sum_{\PtInDomain\in \Discat[i-]_{t}} \!\!\!\! \big( \mudensity[t-1](\PtInDomain) + \sqrt{\betadensity[t]}  \sigdensity[t-1](\PtInDomain) \big)  -  \big( \mudensity[t-1](\PtInDomain) - \sqrt{\betadensity[t]} \sigdensity[t-1](\PtInDomain) \big) \Big)/ \factor \\ \tag{Since, $\density(\PtInDomain) \geq \mudensity[t-1](\PtInDomain) - \sqrt{\betadensity[t]} \sigdensity[t-1](\PtInDomain) \ \forall \ \PtInDomain$ }\\
      &=  2 \sqrt{\betadensity[t]} \sum_{\PtInDomain\in \Discat[i-]_{t}} \sigdensity[t-1](\PtInDomain) / \factor \label{eqn: const-margin-gain-factor} \numberthis \\
      &\leq  2 \sqrt{\betadensity[t]} \DiskCoverageRatio \max_{\PtInDomain\in \Discat[i-]_{t}} \sigdensity[t-1](\PtInDomain) \numberthis \label{eqn: const-gap-marginal-gain}
\end{align*}
The last inequality follows since $\sum_{\PtInDomain\in \Discat[i-]_{t}} \sigdensity[t-1](\PtInDomain) \leq |\Discat[i-]| \max_{\PtInDomain\in \Discat[i-]_{t}} \sigdensity[t-1](\PtInDomain)$ and $\frac{|\Discat[i-]_t|}{\factor} \leq \frac{\max_i|\Discat[i]_t|}{\factor} = \DiskCoverageRatio$.% From \cref{eqn: sim-reg-const-bound}, we have
Now,
\begin{align*}
\simReg{} &= \sum_{\batch \in \BatchColl{t}} \sum_{i\in \batch}  \delgain{\tilde{x}^i_{t}|\LocAgents^{1:i-1}_{t}}{\density}{\uniSet[,\batch]{{t}}} - \delgain{\LocAgent[i]{t}|\LocAgents^{1:i-1}_{t}}{\density}{\uniSet[,\batch]{t}} \\ 
      &\leq \sum_{\batch \in \BatchColl{t}} \sum_{i\in \batch} 2 \sqrt{\betadensity[t]} \DiskCoverageRatio \max_{\PtInDomain\in \Discat[i-]_{t}} \sigdensity[t-1](\PtInDomain) \tag{from \cref{eqn: const-gap-marginal-gain}}\\
      &= \sumMaxwidth[]{t} \leq \epsdensity \\
    %   \implies (1- \frac{1}{e}) \Objfunc{\LocAgents_{\star}}{\density}{\Rbar} - \Objfunc{\LocAgents_{t}}{\density}{\uniSet[]{t}} &\leq \epsdensity \\
\end{align*}
From \cref{eqn: sim-reg-const-bound}, $(1- \frac{1}{e}) \sum_{\batch \in \BatchColl{}} \Objfunc{\LocAgents_{\star}^{\batch}}{\density}{\Rbareps{}{\LocAgents_{0}^{\batch}}} - \sum_{\batch \in \BatchColl{t}} \Objfunc{\LocAgents_{t}^{\batch}}{\density}{\uniSet[,\batch]{t}} = \simActReg \leq \simReg{}$ 
\begin{align*}
\implies
        \sum_{\batch \in \BatchColl{t}} \Objfunc{\LocAgents_{t}^{\batch}}{\density}{\uniSet[,\batch]{t}} &\geq (1- \frac{1}{e}) \sum_{\batch \in \BatchColl{}} \Objfunc{\LocAgents_{\star}^{\batch}}{\density}{\Rbareps{}{\LocAgents_{0}^{\batch}}} - \epsdensity 
\end{align*}
\end{proof}

\begin{proposition} \label{lem: const-time-bound}
Let $\delta \in (0,1)$ and $\betadensity[t]$ as in \citep{beta-chowdhury17a}, i.e., ${\beta^{\density}_t}^{1/2} = B_{\density} + 4 \sigma_{\density} \sqrt{\gamma^{\density}_{\numOfAgents t} + 1 + \ln(1/\delta)}$ and ${\tdensity^{\star}}^1$ is the smallest integer such that $\frac{{\tdensity^{\star}}^1}{\betadensity[{\tdensity^{\star}}^1] I(Y_{\LocAgent[g,1:\numOfAgents]{1:{\tdensity^{\star}}^1}};\density)} \geq \frac{8 \DiskCoverageRatio^2 \numOfAgents^2 }{ \log(1+\numOfAgents\noisedensity) \epsdensity^2}$, then with probability $1-\delta$ that there exists $\tdensity^1 < {\tdensity^{\star}}^1$ such that $\sumMaxwidth{\tdensity^1+1} \leq \epsdensity$, where $\sumMaxwidth{t} = \sum\limits_{\batch \in \BatchColl{t}} \sum\limits_{i \in \batch} 2 \sqrt{\betadensity[t]} \DiskCoverageRatio \max\limits_{\PtInDomain\in \Discat[i-]_{t}} \sigdensity[t-1](\PtInDomain) \leq \epsdensity.$ 
% \begin{align*}
%     \frac{{\tdensity^{\star}}^1}{\betadensity[{\tdensity^{\star}}^1] I(Y_{\LocAgent[g,1:\numOfAgents]{1:{\tdensity^{\star}}^1}};\density)} \leq \frac{8 \DiskCoverageRatio^2 \numOfAgents^2 }{ \log(1+\numOfAgents\noisedensity) \epsdensity^2}
% \end{align*}
\end{proposition}
\begin{proof} Similar to \cref{eqn: macopt-time-bound-w-sq-cum}, It is defined that,
\begin{align*}
 \sumMaxWidth[]{t} &\coloneqq 2 \sqrt{\betadensity[t]} \DiskCoverageRatio \sum\limits_{i =1}^{\numOfAgents}  \max\limits_{\PtInDomain\in \Discat[i-]_{t}} \sigdensity[t-1](\PtInDomain) \\
 \implies \sumMaxWidth[2]{t} &\leq  4 \betadensity[t] \DiskCoverageRatio^2 \numOfAgents \sum_{i=1}^{\numOfAgents} \Big( \sigdensity[t-1](\LocAgent[g,i]{t}) \Big)^2 \numberthis \label{eqn: safemac-time-bound-w-sq}\\
 \implies \Big( \sum_{t=1}^{\Tdensity} \sumMaxWidth[]{t} \Big)^2 \leq \Tdensity \sum_{t=1}^{\Tdensity} \sumMaxWidth[2]{t} &\leq (\DiskCoverageRatio \numOfAgents) \Tdensity \sum_{t=1}^{\Tdensity} 4 \betadensity[t] \DiskCoverageRatio \sum_{i=1}^{\numOfAgents} \Big( \sigdensity[t-1](\LocAgent[g,i]{t}) \Big)^2 \numberthis \label{eqn: const-w-sq-bound-cumm}\\
 &\leq (\DiskCoverageRatio \numOfAgents)  \frac{8 \DiskCoverageRatio  \numOfAgents \Tdensity \betadensity[\Tdensity] I(Y_{\LocAgent[g,1:\numOfAgents]{1:\Tdensity}};\density)}{\log(1+\numOfAgents \noisedensity)} \numberthis \label{eqn: safemac-time-bound-w-sq-cum}
\end{align*}
\cref{eqn: safemac-time-bound-w-sq} follows from Cauchy-Schwarz inequality and $\LocAgent[g,i]{t} = \argmax\limits_{\PtInDomain \in \Discat[i-]_t} \sigma_{\density_{t-1}}(\PtInDomain)$. The RHS of \cref{eqn: safemac-time-bound-w-sq} resembles \cref{eqn: w-sq-bound} in \cref{lem: reg-bound-info}, with an additional factor of $(\DiskCoverageRatio \numOfAgents)$. \cref{eqn: const-w-sq-bound-cumm} directly follows from Cauchy-Schwarz inequality and \cref{eqn: safemac-time-bound-w-sq}. Following the steps in \cref{lem: const-CR} will result in \cref{eqn: safemac-time-bound-w-sq-cum}. 

% we know that from \cref{lem: const-CR}, $\optiLocReg(\Tdensity)$ is sublinear in $\Tdensity$, i.e, $\optiLocReg(\Tdensity)/\Tdensity \to 0 ~\textit{as}~ \Tdensity \to \infty$.  We can quantify the number of times an agent go for the density measurements, before which the \macopt stopping criteria ($\sumMaxwidth{t} \leq \epsdensity$) must have been satisfied at least once.\\
%  since CR is sublinear in \Tdensity i.e, CR/\Tdensity -> 0 , as \Tdensity -> $\infty$,  \\
% CR = $\sum_{t=1}^{\Tdensity} \sum_i^{\numOfAgents} 2 \sqrt{\betadensity[t]} \max_{\PtInDomain\in \Discat[i]_t \backslash \Discat[1:i-1]_t} \sigma_{\density_{t -1}}(\PtInDomain) $
% when $CR/\Tdensity \leq \epsdensity$, we know that at least once  $\sum_i^{\numOfAgents} 2 \sqrt{\betadensity[t]} \max_{\PtInDomain\in \Discat[i]_t \backslash \Discat[1:i-1]_t} \sigma_{\density_{t -1}}(\PtInDomain) \leq \epsdensity$,
Since, it is given that 
\begin{align*}
\frac{{\tdensity^{\star}}^1}{\betadensity[{\tdensity^{\star}}^1] I(Y_{\LocAgent[g,1:\numOfAgents]{1:{\tdensity^{\star}}^1}};\density)} &\geq \frac{8 \DiskCoverageRatio^2 \numOfAgents^2 }{ \log(1+\numOfAgents\noisedensity) \epsdensity^2} \numberthis \label{eqn: t1-bound} \\
\implies (\DiskCoverageRatio \numOfAgents)^{1/2} \sqrt{\frac{8 \DiskCoverageRatio \numOfAgents \betadensity[{\tdensity^{\star}}^1] I(Y_{\LocAgent[g,1:\numOfAgents]{1:{\tdensity^{\star}}^1}};\density)}{{\tdensity^{\star}}^1 \log(1+\numOfAgents\noisedensity)}} &\leq \epsdensity \tag{Rearranging terms}\\
\frac{\sum_{t=1}^{{\tdensity^{\star}}^1} \sumMaxwidth{t}}{{\tdensity^{\star}}^1} \leq (\DiskCoverageRatio \numOfAgents)^{1/2} \sqrt{\frac{8  \DiskCoverageRatio \numOfAgents \betadensity[{\tdensity^{\star}}^1] I(Y_{\LocAgent[g,1:\numOfAgents]{1:{\tdensity^{\star}}^1}};\density)}{{\tdensity^{\star}}^1 \log(1+\numOfAgents\noisedensity)}}  &\leq  \epsdensity \tag{From \cref{eqn: width-CR-sublinear-bound} in \cref{lem: const-CR}}\\
\implies \min_{t \in [1, {\tdensity^{\star}}^1]}  \sumMaxwidth{t} &\leq   \epsdensity \tag{$\frac{{\tdensity^{\star}}^1 \min\limits_{t \in [1, {\tdensity^{\star}}^1]} \sumMaxwidth{t}}{{\tdensity^{\star}}^1} \leq \frac{\sum_{t=1}^{{\tdensity^{\star}}^1} \sumMaxwidth{t}}{{\tdensity^{\star}}^1}$}
\end{align*}
Hence there exists $\tdensity^1 < {\tdensity^{\star}}^1$, such that $\sumMaxwidth{\tdensity^1 +1} \leq \epsdensity$.
%%%%%%%%%%%VERY IMP DESCRIPTION COMMENT IT OUT %%%%%%%%%%%%%%%%
% Note the strict inequality in $\tdensity^1 < {\tdensity^{\star}}^1$, since in the worst case the density measurement is obtained  until ${\tdensity^{\star}}^1 - 1$ and the algorithm terminates at the recommended location for ${\tdensity^{\star}}^1$ without collecting the measurement since $\sumMaxwidth{t} \leq \epsdensity$. Note that ${\tdensity^{\star}}^1$ is the number of times $\sumMaxwidth{t}$ is added up and $\tdensity^1$ is incremented only when a density measurement is collected.
\end{proof}

For notation convenience we denote with $\optiLocReg(\deltime{n}) := \optiLocReg(\tdensity^{n-1} + \deltime{n}) - \optiLocReg(\tdensity^{n-1}) = \sum_{t = \tdensity^{n-1} + 1}^{\tdensity^{n-1} + \deltime{n}} \simReg{}$ and $I(Y_{\deltime{n}};\density) = I(Y_{\LocAgent[g,1:\numOfAgents]{\tdensity^{n-1}+1:\tdensity^{n-1} + \deltime{n}}};\density)$. 

\begin{lemma} \label{lem: delta-time-regret-bound} Let the coverage phase be terminated for the ${(n-1)}^{th}$ time at $\tdensity^{n-1}$, and $\deltime{n}$ be the maximum number of density measurements required to terminate coverage phase for the $n^{th}$ time. Let $\delta \in (0,1)$ and $\betadensity[t]$ as in \citep{beta-chowdhury17a}, i.e., ${\beta^{\density}_t}^{1/2} = B_{\density} + 4 \sigma_{\density} \sqrt{\gamma^{\density}_{\numOfAgents t} + 1 + \ln(1/\delta)}$, then with probability at least  $1-\delta$ the following inequality holds,
\begin{align*}
    \optiLocReg(\deltime{n}) &\leq \Big(\deltime{n} \sum_{t = \tdensity^{n-1} + 1}^{\tdensity^{n-1} + \deltime{n}} \sumMaxwidth[2]{t}\Big)^{1/2}  \leq (\DiskCoverageRatio \numOfAgents)^{1/2} \sqrt{ \frac{8 \deltime{n} \DiskCoverageRatio \numOfAgents \betadensity[\tdensity^{n-1} + \deltime{n}] I(Y_{\deltime{n}};\density) }{\log(1+\numOfAgents\noisedensity)}}
\end{align*}
\end{lemma}
\begin{proof}
% Using the definition of local agent wise regret as $ \optiLocReg(\Tdensity) = \sum_{t=1}^{\Tdensity} \simReg{} $, where $\simReg{} = \sum_{\batch \in \BatchColl{t}} \sum_{i\in \batch} \simLocReg{t}$ and $\simReg{} \leq \sumMaxwidth{t}$. We know, 
With definitions,
\begin{align*}
    \optiLocReg(\deltime{n}) &= \sum_{t = \tdensity^{n-1} + 1}^{\tdensity^{n-1} + \deltime{n}} \simReg{} \leq \sum_{t = \tdensity^{n-1} + 1}^{\tdensity^{n-1} + \deltime{n}} \sumMaxwidth{t} \\
   \implies \big( \optiLocReg(\deltime{n}) \big)^2 &\leq \sum_{t = \tdensity^{n-1} + 1}^{\tdensity^{n-1} + \deltime{n}} \sumMaxwidth{t} \leq \deltime{n}\sum_{t = \tdensity^{n-1} + 1}^{\tdensity^{n-1} + \deltime{n}} \sumMaxwidth[2]{t} \tag{Using, Cauchy-Schwarz inequality}
\end{align*}
Now, the RHS of the inequality can be simplified as,
\begin{align*}
    \deltime{n} \!\!\!\!\sum_{t = \tdensity^{n-1} + 1}^{\tdensity^{n-1} + \deltime{n}}  \!\!\!\! \sumMaxwidth[2]{t} &\leq ( \DiskCoverageRatio \numOfAgents) \deltime{n} \!\!\!\! \sum_{t = \tdensity^{n-1} + 1}^{\tdensity^{n-1} + \deltime{n}} \!\!\!\! \frac{8 \DiskCoverageRatio \numOfAgents \betadensity[t]}{\log(1+\numOfAgents\noisedensity)}  \sum_{i=1}^{\numOfAgents} \frac{1}{2} \log ( 1 + \noisedensity \lambda_{i,t}) \tag{using \cref{eqn: const-w-sq-bound-cumm}}\\
    &\leq   \frac{8 \deltime{n} \numOfAgents^2 \betadensity[\tdensity^{n-1} + \deltime{n}]} {\log(1+\numOfAgents\noisedensity)} \sum_{t = \tdensity^{n-1} + 1}^{\tdensity^{n-1} + \deltime{n}} \sum_{i=1}^{\numOfAgents} \frac{1}{2} \log ( 1 + \noisedensity \lambda_{i,t}) \tag{since, $\betadensity[t]$ is non-decreasing and using definition of mutual information we get,}\\
   \!\!\!\! \implies \!\!\!\! \optiLocReg(\deltime{n}) &\leq \!\!\!\!\!\! \sum_{t = \tdensity^{n-1} + 1}^{\tdensity^{n-1} + \deltime{n}} \!\!\!\! \sumMaxwidth{t} \leq \Big(\deltime{n} \!\!\!\!\! \sum_{t = \tdensity^{n-1} + 1}^{\tdensity^{n-1} + \deltime{n}} \!\!\!\!\! \sumMaxwidth[2]{t}\Big)^{1/2} \!\!\!\! \leq (\DiskCoverageRatio \numOfAgents)^{1/2} \!\!\!\! \sqrt{ \frac{8 \deltime{n} \DiskCoverageRatio \numOfAgents \betadensity[\tdensity^{n-1} + \deltime{n}] I(Y_{\deltime{n}};\density) }{\log(1+\numOfAgents\noisedensity)}} \label{eqn: t2-eps-bound} \numberthis
\end{align*}
\end{proof}
\begin{lemma} \label{lem: delta-time-bound}
Let $\delta \in (0,1)$ and $\betadensity[t]$ as in \citep{beta-chowdhury17a}, i.e., ${\beta^{\density}_t}^{1/2} = B_{\density} + 4 \sigma_{\density} \sqrt{\gamma^{\density}_{\numOfAgents t} + 1 + \ln(1/\delta)}$ and $\deltime{n}$ is the smallest integer after $\tdensity^{n-1}$ such that $\frac{\deltime{n}}{\betadensity[\tdensity^{n-1} + \deltime{n}] I(Y_{\deltime{n}};\density)} \geq \frac{8 \DiskCoverageRatio^2 \numOfAgents^2 }{ \log(1+\numOfAgents\noisedensity) \epsdensity^2}$, then we know with probability $1-\delta$ that there exists $\smdeltime{n} < \deltime{n}$ such that $\sumMaxwidth{\tdensity^{n-1} + \smdeltime{n}+1} \leq \epsdensity$, where $\sumMaxwidth{t} = \sum_{\batch \in \BatchColl{t}} \sum_{i \in \batch} 2 \sqrt{\betadensity[t]} \DiskCoverageRatio \max_{\PtInDomain\in \Discat[i-]_{t}} \sigdensity[t-1](\PtInDomain) \leq \epsdensity.$
\end{lemma}
\begin{proof}
Given,
\begin{align*}
    \frac{\deltime{n}}{\betadensity[\tdensity^{n-1} + \deltime{n}] I(Y_{\deltime{n}};\density)} &\geq \frac{8 \DiskCoverageRatio^2 \numOfAgents^2 }{ \log(1+\numOfAgents\noisedensity) \epsdensity^2} \\
    \implies (\DiskCoverageRatio \numOfAgents)^{1/2} \sqrt{ \frac{8 \DiskCoverageRatio \numOfAgents \betadensity[\tdensity^{n-1} + \deltime{n}] I(Y_{\deltime{n}};\density) }{\deltime{n}\log(1+\numOfAgents\noisedensity)}} &\leq \epsdensity \\
    \frac{\sum_{\tdensity^{n-1} + 1}^{\tdensity^{n-1} + \deltime{n}} \sumMaxwidth{t}}{\deltime{n}} \leq \epsdensity \tag{Using \cref{eqn: t2-eps-bound} in \cref{lem: delta-time-regret-bound}} \\
    \implies  \min_{t \in [\tdensity^{n-1} + 1,  \tdensity^{n-1} + \deltime{n}]}  \sumMaxwidth{t} \leq \epsdensity
\end{align*}
Hence there exists $\smdeltime{n} < \deltime{n}$, such that $\sumMaxwidth{\tdensity^{n-1} + \smdeltime{n}+1} \leq \epsdensity$.
\end{proof}
%Since we collect density measurement until $\tdensity^{n-1} + \deltime{n} - 1$, we know that 

\begin{lemma} \label{lem: additive-info-time-bound}
Let $\delta \in (0,1)$ and ${\beta^{\density}_t}^{1/2} = B_{\density} + 4 \sigma_{\density} \sqrt{\gamma^{\density}_{\numOfAgents t} + 1 + \ln(1/\delta)}$ and $\tdensity^{\star}$ is the smallest integer such that $\frac{\tdensity^{\star}}{\betadensity[\tdensity^{\star}] \gammadensity{\numOfAgents \tdensity^{\star}}} \geq \frac{8 \DiskCoverageRatio^2 \numOfAgents^2 }{ \log(1+\numOfAgents\noisedensity) \epsdensity^2}$, then for any $n \geq 1$, $\tdensity^{n-1} + \smdeltime{n} < \tdensity^{\star}$. 
\end{lemma}
\begin{proof}
\begin{align*}
    \tdensity^{n-1} + \smdeltime{n}
    &< \frac{8 \DiskCoverageRatio^2 \numOfAgents^2 \betadensity[\tdensity^{n-1}] I(Y_{\LocAgent[g,1:\numOfAgents]{1:\tdensity^{n-1}}};\density)}{ \log(1+\numOfAgents\noisedensity) \epsdensity^2} + \frac{8 \DiskCoverageRatio^2 \numOfAgents^2 \betadensity[\tdensity^{n-1} + \smdeltime{n}] I(Y_{\smdeltime{n}};\density)}{ \log(1+\numOfAgents\noisedensity) \epsdensity^2}  \tag{using \cref{eqn: t1-bound}, since $\tdensity^{1} < {\tdensity^{\star}}^{1}$} \\
    &< \frac{8 \DiskCoverageRatio^2 \numOfAgents^2 \betadensity[\tdensity^{n-1} + \smdeltime{n}]}{ \log(1+\numOfAgents\noisedensity) \epsdensity^2}( I(Y_{\LocAgent[g,1:\numOfAgents]{1:\tdensity^{n-1}}};\density) +  I(Y_{\smdeltime{n}};\density) \tag{Since, $\betadensity[t]$ is non decreasing function} \\
    &= \frac{8 \DiskCoverageRatio^2 \numOfAgents^2 \betadensity[\tdensity^{n-1} + \smdeltime{n}] I(Y_{\LocAgent[g,1:\numOfAgents]{1:\tdensity^{n-1} + \smdeltime{n}}};\density)}{ \log(1+\numOfAgents\noisedensity) \epsdensity^2} \tag{Since mutual info is additive} \\
    &< \frac{8 \DiskCoverageRatio^2 \numOfAgents^2 \betadensity[\tdensity^{n-1} + \smdeltime{n}] \gammadensity{\numOfAgents (\tdensity^{n-1} + \smdeltime{n})}}{ \log(1+\numOfAgents\noisedensity) \epsdensity^2} \numberthis \label{eqn: time-add-gamma-bound}
\end{align*}
Using \cref{eqn: time-add-gamma-bound} and since, $\tdensity^{\star} \geq  \frac{8 \DiskCoverageRatio^2 \numOfAgents^2 \betadensity[\tdensity^{\star}] \gammadensity{\numOfAgents \tdensity^{\star}}}{ \log(1+\numOfAgents\noisedensity) \epsdensity^2}$, we get $\tdensity^{n-1} + \smdeltime{n} < \tdensity^{\star}$. 
\end{proof}

\begin{lemma}
\label{lem: safemac-convergence-set-equality}
When \safemac converges, i.e, $\uncertaindisk \coloneqq \{ \optiSet[,i]{t} \backslash \pessiSet[,i]{t}) \cap  \Discat[i]_t , \forall i \in [\numOfAgents] \} = \emptyset$,  then the following inequality holds,
\begin{align*}
    \sum_{\batch \in \BatchColl{t}} \Objfunc{\LocAgents^{\batch}_{t}}{\density}{\pessiSet[,\batch]{t}} = \sum_{\batch \in \BatchColl{t}} \Objfunc{\LocAgents^{\batch}_{t}}{\density}{\uniSet[,\batch]{t}}
\end{align*}
\end{lemma}
\begin{proof}
Since, $\uncertaindisk  = \emptyset$,
\begin{align*}
    \{ (\optiSet[,i]{t} \backslash \pessiSet[,i]{t}) \cap  \Discat[i]_t , \forall i \in [\numOfAgents] \} &= \emptyset \\
    \implies (\optiSet[,i]{t} \cap \Discat[i]_t) &\subseteq (\pessiSet[,i]{t} \cap \Discat[i]_t) ~~\forall i \in [\numOfAgents]\\
    &= (\uniSet[,i]{t} \cap \Discat[i]_t) ~~\forall i \in [\numOfAgents]  \tag{Since $\uniSet[,i]{t} \coloneqq \pessiSet[,i]{t} \cup \optiSet[,i]{t}$ }\\
\end{align*}
Based on the last equality, it directly follows, 
$$\sum_{\batch \in \BatchColl{t}} \Objfunc{\LocAgents^{\batch}_{t}}{\density}{\pessiSet[,\batch]{t}} = \sum_{\batch \in \BatchColl{t}} \Objfunc{\LocAgents^{\batch}_{t}}{\density}{\uniSet[,\batch]{t}}. $$
\end{proof}

\subsection{Intermediate recommendation is near-optimal at \safemac's convergence}
\label{apx: recommendation-rule}

% \begin{restatable*}[Goldbach's conjecture]{theom}{goldbach}
% \label{thm:goldbach}
% Every even integer greater than 2 can be expressed as the sum of two primes.
% \end{restatable*}

% \goldbachqw

\begin{lemma} Let $\delta \in (0,1)$ and $\betadensity[t]$ as in \citep{beta-chowdhury17a}, i.e., ${\beta^{\density}_t}^{1/2} = B_{\density} + 4 \sigma_{\density} \sqrt{\gamma^{\density}_{\numOfAgents t} + 1 + \ln(1/\delta)}$ and $t^{\star}_\density$ be the smallest integer such that $\frac{t^{\star}_{\density}}{\beta_{t^{\star}_{\density}} \gamma_{\numOfAgents t^{\star}_{\density} }} \geq \frac{8 \DiskCoverageRatio^2 \numOfAgents^2 }{ \log(1+\numOfAgents \noisedensity) \epsdensity^2}$. Let $\beta^{\constrain}_t$ and $t^{\star}_\constrain$ be defined analogously.
Then, there exists $t < t^{\star}_\constrain + t^{\star}_{\density}$, such that with probability at least $1 - \delta$ 
\begin{align*}
   \sum_{\batch \in \BatchColl{T}} \Objfunc{\LocAgents_T^{\batch}}{\density}{\RbarO{\LocAgents_0^{\batch}}} \geq (1- \frac{1}{e})\sum_{\batch \in \BatchColl{}} \Objfunc{\LocAgents_{\star}^{\batch}}{\density}{\Rbareps{}{\LocAgents_0^{\batch}}} - \epsdensity \numberthis 
\end{align*}
where, 
\begin{align*}
\LocAgents_T = \argmax_{\LocAgents_T,\LocAgents^{l}_T, T \in [1, t]}  \Big\{ \sum_{\batch \in \BatchColl{T}^p} \Objfunc{\LocAgents^{\batch}_T}{\lbdensity[T-1]}{\pessiSet[,\batch]{T}}, \sum_{\batch \in \BatchColl{T}^p} \Objfunc{\LocAgents^{l, \batch}_T}{\lbdensity[T-1]}{\pessiSet[,\batch]{T}} \Big\}~ s.t. \LocAgents_T \in \pessiSet[]{T}~~\numberthis \label{eqn: pre-mature-recommendation} 
\end{align*}
and $\LocAgents^{l, \batch}_t$, i.e., the greedy solution w.r.t. the worst-case objective, $\Objfunc{\cdot}{\lbdensity[t-1]}{\pessiSet[, \batch]{t}} \, \forall \batch \in \BatchColl{t}^p$.
\end{lemma}
\begin{proof} We prove the lemma in two parts. First, we prove the near optimality of \safemac's solution $\LocAgents_{t}$ but evaluated using $\lbdensity[t-1]$ instead of $\density$. This will imply the near optimality at convergence of the $1^{st}$ term $(\sum_{\batch \in \BatchColl{T}^p} \Objfunc{\LocAgents^{\batch}_T}{\lbdensity[T-1]}{\pessiSet[,\batch]{T}})$ in the above recommendation rule. Secondly, due to the $\argmax$ operator, the near optimality of the $1^{st}$ term is sufficient to establish the optimality of the recommendation rule in \cref{eqn: pre-mature-recommendation}.
%given in the \cref{eqn: pre-mature-recommendation}.

\mypar{Notations} $\LocAgents_t = \cup_{\batch \in \BatchColl{t}} \LocAgents^\batch_t$, $\delgain{\cdot}{\density}{\Domain}$ as defined in \cref{eqn: marginal-gain}. %\todo{Delta notation defined at ..... }

\mypar{Given} From \cref{thm: safeMac}, for $t < t^{\star}_\constrain + t^{\star}_{\density}$ with probability at least $1 - \delta$ ,
\begin{align*}
   \sum_{\batch \in \BatchColl{t}} \Objfunc{\LocAgents_t^{\batch}}{\density}{\RbarO{\LocAgents_0^{\batch}}} \geq (1- \frac{1}{e})\sum_{\batch \in \BatchColl{}} \Objfunc{\LocAgents_{\star}^{\batch}}{\density}{\Rbareps{}{\LocAgents_0^{\batch}}} - \epsdensity \numberthis 
\end{align*}
and 
\begin{align*}
    \sum_{\batch \in \BatchColl{t}} \sum_{i\in \batch} 2 \sqrt{\betadensity[t]} \DiskCoverageRatio \max_{\PtInDomain\in \Discat[i-]_{t}} \sigdensity[t-1](\PtInDomain) \leq \epsdensity
\end{align*}

 %Since, the coverage at the recommended solution will be at least better than the coverage evaluated using the lower bound to the solution suggested by \safemac. This shows that the recommendation solution is also near optimal. 

\mypar{Near-optimality of \safemac's $\LocAgents_{t}$ evaluated using $\lbdensity[t-1]$} 
\begin{align*}
\delgain{\tilde{x}^i_{t}|\LocAgents^{1:i-1}_{t}}{\density}{\uniSet[,\batch]{{t}}} - &\delgain{\LocAgent[i]{t}|\LocAgents^{1:i-1}_{t}}{\lbdensity[t-1]}{\uniSet[,\batch]{t}} \\
&=  \Big( \sum_{\PtInDomain\in \tlDiscat[i-]_{t}} \density(\PtInDomain) - \sum_{\PtInDomain\in \Discat[i-]_{t} } \lbdensity[t-1](\PtInDomain)  \Big) / \factor \tag{Note $\Discat[i-]_{t}$ and $\tlDiscat[i-]_{t}$}\\
&\leq  \Big( \!\!\!\! \sum_{\PtInDomain\in \Discat[i-]_{t}} \!\!\!\! \big( \mudensity[t-1](\PtInDomain) + \sqrt{\betadensity[t]}  \sigdensity[t-1](\PtInDomain) \big)  -  \big( \mudensity[t-1](\PtInDomain) - \sqrt{\betadensity[t]} \sigdensity[t-1](\PtInDomain) \big) \Big)/ \factor \\ \tag{Using \cref{eqn: opti-xtilde_inequality} and definition of $\lbdensity[t-1](\PtInDomain)$}\\
&=  2 \sqrt{\betadensity[t]} \sum_{\PtInDomain\in \Discat[i-]_{t}} \sigdensity[t-1](\PtInDomain) / \factor \\
      &\leq  2 \sqrt{\betadensity[t]} \DiskCoverageRatio \max_{\PtInDomain\in \Discat[i-]_{t}} \sigdensity[t-1](\PtInDomain) \numberthis
\end{align*}
\begin{align*}
    \sum_{\batch \in \BatchColl{t}} \sum_{i\in \batch}  \delgain{\tilde{x}^i_{t}|\LocAgents^{1:i-1}_{t}}{\density}{\uniSet[,\batch]{{t}}} - \delgain{\LocAgent[i]{t}|\LocAgents^{1:i-1}_{t}}{\lbdensity[t-1]}{\uniSet[,\batch]{t}} &\leq  \sum_{\batch \in \BatchColl{t}} \sum_{i\in \batch}  2 \sqrt{\betadensity[t]} \DiskCoverageRatio \max_{\PtInDomain\in \Discat[i-]_{t}} \sigdensity[t-1](\PtInDomain) \\
    & \leq \epsdensity
\end{align*}
Using the following two statements,
\begin{itemize}
    \item $(1- \frac{1}{e}) \Objfunc{\LocAgents_{\star}}{\density}{\uniSet[, \batch]{t}} 
    \leq \sum_{i\in \batch}  \delgain{\tilde{\LocAgent[]{}}^i|\LocAgents^{1:i-1}_t}{\density}{\uniSet[, \batch]{t}} $ from \cref{eqn: sim-reg-opti-bound}
    \item $\bigcup_{i \in \batch} \Rbareps{}{\{ \LocAgent[i]{0} \}} \subseteq \uniSet[, \batch]{t} $ 
$ \implies \sum_{\batch \in \BatchColl{}} \Objfunc{\LocAgents_{\star}}{\density}{\Rbareps{}{\LocAgents_{0}}} \leq \sum_{\batch \in \BatchColl{t} }\Objfunc{\LocAgents_{\star}}{\density}{\uniSet[, \batch]{t}} $
\end{itemize}
we get, 
\begin{align*}
    \implies    \sum_{\batch \in \BatchColl{t}} \Objfunc{\LocAgents_t^{\batch}}{\lbdensity[t-1]}{\uniSet[,\batch]{t}} \geq (1- \frac{1}{e})\sum_{\batch \in \BatchColl{}} \Objfunc{\LocAgents_{\star}^{\batch}}{\density}{\Rbareps{}{\LocAgents_0^{\batch}}} - \epsdensity \label{eqn: converge-near-optimal-lower-bound} \numberthis
\end{align*}

\mypar{Near-optimality of recommendation as per \cref{eqn: pre-mature-recommendation}}\\
Let's consider the following recommendation rule, 
\begin{align*}
\LocAgents_T = \argmax_{\LocAgents_T,  T \in [1, t]}  \Big\{ \sum_{\batch \in \BatchColl{T}^p} \Objfunc{\LocAgents^{\batch}_T}{\lbdensity[T-1]}{\pessiSet[,\batch]{T}} \Big\}~ s.t. \LocAgents_T \in \pessiSet[]{T}~~\numberthis \label{eqn: partial-recommendation} 
\end{align*}
At convergence, $\pessiSet[,i]{t} \cap \Discat[i]_t = \uniSet[,i]{t} \cap \Discat[i]_t \implies (\optiSet[]{} \backslash \pessiSet[,i]{t} ) \cap \Discat[i]_t = \emptyset$, using this \safemac recommendation $\LocAgents_t$ can be written as,
\begin{align*}
\sum_{\batch \in \BatchColl{t}} \Objfunc{\LocAgents_t^{\batch}}{\lbdensity[t-1]}{\uniSet[,\batch]{t}} = \sum_{i \in [\numOfAgents] } \delgain{ \LocAgent[i]{t} | \LocAgents^{1:i-1}_t }{\lbdensity[t-1]}{\pessiSet[,i]{t}} = \sum_{\batch \in \BatchColl{t}^p} \Objfunc{\LocAgents_t^{\batch}}{\lbdensity[t-1]}{\pessiSet[,\batch]{t}} 
\end{align*}
\begin{align*}
\sum_{\batch \in \BatchColl{T}^p} \Objfunc{\LocAgents^{\batch}_T}{\lbdensity[T-1]}{\pessiSet[,\batch]{T}} &\geq \sum_{\batch \in \BatchColl{t}^p} \Objfunc{\LocAgents_t^{\batch}}{\lbdensity[t-1]}{\pessiSet[,\batch]{t}} \tag{since, $\LocAgents_T^{\batch}=\argmax\limits_{\LocAgents_T^{\batch}, T \in [1, t]}$ $\sum_{\batch \in \BatchColl{T}^{p}} \Objfunc{\LocAgents_T^{\batch}}{\lbdensity[T-1]}{\pessiSet[,\batch]{T}}$ } \\
\implies \sum_{\batch \in \BatchColl{T}^p} \Objfunc{\LocAgents^{\batch}_T}{\lbdensity[T-1]}{\pessiSet[,\batch]{T}}
&\geq \sum_{\batch \in \BatchColl{t}} \Objfunc{\LocAgents_t^{\batch}}{\lbdensity[t-1]}{\uniSet[,\batch]{t}} \tag{Combining the above 2 equations} \\
\implies \sum_{\batch \in \BatchColl{T}^p} \Objfunc{\LocAgents^{\batch}_T}{\lbdensity[T-1]}{\pessiSet[,\batch]{T}}
&\geq (1- \frac{1}{e})\sum_{\batch \in \BatchColl{}} \Objfunc{\LocAgents_{\star}^{\batch}}{\density}{\Rbareps{}{\LocAgents_0^{\batch}}} - \epsdensity \tag{using \cref{eqn: converge-near-optimal-lower-bound}}
\end{align*}

Hence, the recommendation of \cref{eqn: partial-recommendation} evaluated with lower bound is near optimal (at convergence $\LocAgents_{T} \in \pessiSet[]{T}$). Further, due to $\argmax$ operator \cref{eqn: partial-recommendation} also implies near-optimality of recommendation rule in \cref{eqn: pre-mature-recommendation} evaluated with the lower bound. So now using $X_T$ chosen as per \cref{eqn: pre-mature-recommendation} and at convergence, $\forall i, ( \pessiSet[,i]{t} \cap \Discat[i]_t ) \subseteq (\RbarO{ \{ \LocAgent[i]{0} \}} \cap \Discat[i]_t)$, we get, 
\begin{align*}
\sum_{\batch \in \BatchColl{T}^p} \Objfunc{\LocAgents^{\batch}_T}{\density}{\pessiSet[,\batch]{T}} &\geq (1- \frac{1}{e})\sum_{\batch \in \BatchColl{}} \Objfunc{\LocAgents_{\star}^{\batch}}{\density}{\Rbareps{}{\LocAgents_0^{\batch}}} - \epsdensity \tag{$\lbdensity[t-1](\PtInDomain) \leq \density(\PtInDomain) \forall \PtInDomain$}\\
 \sum_{\batch \in \BatchColl{T}^p} \Objfunc{\LocAgents^{\batch}_T}{\density}{\RbarO{ \LocAgents_{0} }}
&\geq (1- \frac{1}{e})\sum_{\batch \in \BatchColl{}} \Objfunc{\LocAgents_{\star}^{\batch}}{\density}{\Rbareps{}{\LocAgents_0^{\batch}}} - \epsdensity 
\end{align*}
Hence Proved.
\end{proof}

\newpage
\section{Multi-agent \goose version}
\label{Apx: ma-goose}
In this section, we present our lemma for the multi-agent version of goose. In the cooperative setting, each agent deploy \goose for safe exploration and shares its observations with the other agents. We first derive a sample complexity bound under the cooperative system. Later, we introduce our key \cref{lem: ma-goose}, which guarantees the safety of all agents as well as complete exploration (with respect to each agent) in finite time.

\begin{lemma} \label{lem: info-sharing}
Let $\delta \in (0,1)$ and let $(\betaconst[t])^{1/2} = B_{\constrain} + 4 \sigma_{\constrain} \sqrt{\gamma^{\constrain}_{\numOfAgents t} + 1 + \ln(1/\delta)}$. Then the following holds with probability at least $1-\delta$, 
\begin{align*}
    \sum_{t} \widthconst{2}{t} \leq C_1 \betaconst[t] I(Y_{\numOfAgents T}; \constrain) \leq C_1 \betaconst[t] \gammaconst{\numOfAgents T},
\end{align*}
where $C_1 = 8/\log(1 + \noiseconst)$, $\widthconst{}{t} = \ubconst[t-1](\LocAgent[i]{t}) - \lbconst[t-1](\LocAgent[i]{t})$, and $\LocAgent[i]{t}$ is the location visited by some agent $i$ at time $t$. $I(Y_{\numOfAgents T}; \constrain)$ is the information gain and $\gammaconst{\numOfAgents T}$ is the information capacity.
\end{lemma}
\begin{proof} Using $\widthconst{}{t} \leq 2 \sqrt{\betaconst[t] }  \sigconst[t-1](\LocAgent[i]{t})$,
\begin{align*}
    \widthconst{2}{t} &\leq 4 \betaconst[t]  ({\sigconst[t-1]}(\PtInDomain))^2 \leq 4 \betaconst[t]  \sigma^2_{\constrain} \noiseconst ({\sigconst[t-1]}(\LocAgent[i]{t}))^2 \leq 4 \betaconst[t]  \sigma^2_{\constrain} C_2 \log ( 1 + \noiseconst({\sigconst[t-1]}(\LocAgent[i]{t}))^2) \label{eqn: cooper-sigma-log} \numberthis \\
    % &\leq C_1 \betaconst[t] \frac{1}{2} \log ( 1 + \noiseconst({\sigconst[t-1]}(\LocAgent[i]{t}))^2) \tag{$C_1 = 8 \sigma^2_{\constrain} C_2$}\\
    &\leq C_1 \betaconst[t] \frac{1}{2} \log ( 1 + \noiseconst\sum_{i=1}^{\numOfAgents} ({\sigconst[t-1]}(\LocAgent[i]{t}))^2) = C_1 \betaconst[t] \frac{1}{2} \log ( 1 + \noiseconst\sum_{i}^{\numOfAgents} \lambda_{i,t}) \label{eqn: cooper-trace} \numberthis \\
    &\leq C_1 \betaconst[t] \sum_{i=1}^{\numOfAgents} \frac{1}{2} \log ( 1 + \noiseconst\lambda_{i,t}) = C_1 \betaconst[t] I(Y_{\numOfAgents T}; \constrain) \leq C_1 \betaconst[t] \gammaconst{\numOfAgents T} \label{eqn: cooper-mutual-info} \numberthis 
\end{align*}
Last inequality in \cref{eqn: cooper-sigma-log} follows since, $s \leq C_2 \log(1+s) $ for $ s \in [0, \noiseconst],$ where $ C_2 = \noiseconst/\log(1+\noiseconst) \geq 1$, where $s = \noiseconst{\sigconst[t-1]}(\PtInDomain)^2 \leq \noiseconst\kernelfunc^{\constrain}(\PtInDomain,\PtInDomain) \leq \noiseconst, (wlog \; \kernelfunc^{\constrain}(v,v) \leq 1)$. Inequality of \cref{eqn: cooper-trace} follows since, $({\sigconst[t-1]}(\LocAgent[i]{t}))^2 \leq \sum_{i=1}^{\numOfAgents} ({\sigconst[t-1]}(\LocAgent[i]{t}))^2$ and equality using $\sum_{i=1}^{\numOfAgents} (\sigdensity[t-1](\LocAgent[i]{t}))^2 = \trace(\constkernel) = \sum_{i=1}^{\numOfAgents} \lambda_{i,t}$. \cref{eqn: cooper-mutual-info} follows since 
$\log(1+x_1 +x_2) \leq \log(1+x_1) + \log(1+x_2), \textit{for}~x_1, x_2 \geq 0$. Lastly, $I(;\constrain)$ is defined analogous to $I(;\density)$ (as in \cref{eqn: macopt-information-gain}) and $\gammaconst{\numOfAgents T} = \sup_{{A} \subset V; |A| = \numOfAgents T} I(Y_{A};\constrain)$.
\end{proof}

Similar to \lem~8 of \citet{turchetta2019safe}, 
Let us denote $\T^{\PtInDomain}_t = \{ \tau_1, . . . , \tau_j \}$ the set of steps where the constraint $\constrain$ is evaluated at $\PtInDomain$ by step t. 
\begin{lemma} \label{lem: width-bound}
For any $t \geq 1$ and for any $ \PtInDomain \in \Domain$, it holds that $w_t(\PtInDomain) \leq \sqrt{\frac{C_1 \betaconst[t] \gammaconst{\numOfAgents t}}{|\T^{\PtInDomain}_t|}}$, with $C_1 = 8/\log(1+\noiseconst)$.
\end{lemma}
\begin{proof} 
\begin{align*}
    |\T^{\PtInDomain}_t| w^{2}_t(\PtInDomain) &\leq \sum_{\tau \in \T^{\PtInDomain}_t} w^{2}_{\tau}(\PtInDomain) \numberthis \label{eqn: width-shrink}\\
    &\leq \sum_{\tau \in \T^{\PtInDomain}_t} 4 \betaconst[\tau] ({\sigconst[t-1]}(\LocAgent[i]{t}))^2 \\
    &\leq \sum_{\tau \in t} 4 \betaconst[\tau] ({\sigconst[t-1]}(\LocAgent[i]{t}))^2 \\
    &\leq C_1 \betaconst[t] \gammaconst{\numOfAgents t} 
\end{align*}
\cref{eqn: width-shrink}, follows due to intersection of confidence interval arguments, Lemma 1 of \citet{turchetta2019safe} and the inequality follows due to \cref{lem: info-sharing}. 
\end{proof}
% Hint: $T_t$ is the time after $t$ steps when the safe region doesn't expand and the width of the constraint confidence intervals is below $\eps$.\\
Let us denote with $T_t$, the smallest positive integer such that $\frac{T_t}{\betaconst[t+T_t] \gammaconst{\numOfAgents, t + T_t}} \geq \frac{C_1}{\epsconst^2}$, with $C_1 = 8/\log(1 + \noiseconst)$ and with $t^{\star}$ the smallest positive integer such that $t^{\star} \geq |\RbarO{\LocAgents_0}| T_{t^{\star}}$.

\begin{lemma} \label{lem: width-shrink}
For any $t \leq t^{\star}$, for any $\PtInDomain \in \Domain$ such that $|\T^{\PtInDomain}_t| \geq T_{t^\star}$, it holds that $w_t(\PtInDomain) \leq \epsconst$.
\end{lemma}
\begin{proof} Since $T_t$ is an increasing function of $t$ \citep{safe-bo-sui15}, we have  $|\T^{\PtInDomain}_t| \geq T_{t^\star} \geq T_{t}$. Therefore using \cref{lem: width-bound} and the definition of $T_t$, we get,
\begin{align*}
    w_t(\PtInDomain) \leq \sqrt{\frac{C_1 \betaconst[t] \gammaconst{\numOfAgents t}}{T_t}} \leq \sqrt{\frac{C_1 \betaconst[t] \gammaconst{\numOfAgents t} \epsconst^2 }{C_1 \gammaconst{\numOfAgents, t + T_t} \betaconst[t + T_t]}} \leq \sqrt{\frac{ \betaconst[t] \gammaconst{\numOfAgents t} }{ \gammaconst{\numOfAgents, t + T_t} \betaconst[t + T_t]}} \epsconst\leq \epsconst.
\end{align*}
The last inequality follows from the fact that both $\betaconst[t]$ and $\gammaconst{t}$ are non-decreasing function of $t$.
\end{proof}

Regarding the convergence of the pessimistic and the optimistic sets, Lemma 10-18 of \citet{turchetta2019safe} can be proved analogously for each agent $i$. We skip re-writing them and directly cite them in the following lemma.

% \begin{align*}
%     \frac{t^{\star}}{\betaconst[{t^{\star}}] \gammaconst{\numOfAgents t^{\star}}} \geq \frac{|\RbarO| T_{t^{\star}}}{\betaconst[{t^{\star}}] \gammaconst{\numOfAgents t^{\star}}}  \geq \frac{|\RbarO| T_{t^{\star}}}{\betaconst[t^{\star} + T_{t^{\star}}] \gammaconst{\numOfAgents, t^{\star} + T_{t^{\star}}}} \geq \frac{C_1 |\RbarO| }{\epsconst^2}
% \end{align*}

\begin{lemma} \label{lem: ma-goose}
Assume that $\constrain(\cdot)$ is $\LipConst$-Lipschitz continuous w.r.t d(.,.) with $\|q\|_k \leq B_{\constrain}$, $\LocAgents_{0} \neq \emptyset$, $\constrain(\LocAgent[i]{0}) \geq 0$ for all $i \in [\numOfAgents]$. Let $(\betaconst[t])^{1/2} = B_{\constrain} + 4 \sigma_{\constrain} \sqrt{\gamma^{\constrain}_{\numOfAgents t} + 1 + \ln(1/\delta)}$, then, for any heuristic $h_t : \Domain \to \R$, with probability at least $1-\delta$, we have $\constrain(\LocAgent[]{}) \geq 0$, for any $\LocAgent[]{}$ along the state trajectory pursued by any agent in \safemac. Moreover, let $\gammaconst{\numOfAgents t}$ denote the information capacity associated with the kernel $\kernelfunc^{\constrain}$ and let $\tconst^{\star}$ be the smallest integer such that $\frac{\tconst^{\star}}{\beta_{\tconst^{\star}} \gamma_{\numOfAgents \tconst^{\star}}} \geq \frac{C_1 |\RbarO{\LocAgents_0}|}{\epsconst^2}$, with $C_1 = 8/\log(1+ \noiseconst)$, then there exists $t\leq \tconst^{\star}$ such that, with probability at least $1-\delta$, $\Rbareps{}{\{\LocAgent[i]{0}\}} \subseteq \optiSet[,i]{t}\subseteq \pessiSet[,i]{t} \subseteq \RbarO{\{\LocAgent[i]{0}\}}$ for all $i \in [\numOfAgents]$.
\end{lemma}
\begin{proof} In \safemac, each agent has a record of its optimistic and pessimistic set.The lemma is similar to $\numOfAgents$ instances of Theorem 1 of \citet{turchetta2019safe}; each instance corresponds to per agent case. Safety of each agent $i$ is a direct consequence of Theorem 2 of \citet{turchetta2016safemdp}. Finite time bound while agents are sharing information is consequence of \cref{lem: info-sharing}-\ref{lem: width-shrink}. The convergence of the pessimistic and optimistic approximation of the safe sets for each agent is a direct consequence of Lemmas 16-18 of \citet{turchetta2019safe}. 
\end{proof}

For a detailed discussion, we refer the reader to Appendix D Completeness of \citet{turchetta2019safe}.
\newpage
\section{Experiments}
\label{apx: experiments}

% \newpage

% \input{appendix/G1-scaling experiments}
% \newpage

\mypar{Implementation details} We implemented all our algorithms with BoTorch \citep{balandat2020botorch} and GPyTorch \citep{gardner2018gpytorch} frameworks, built on top of Pytorch \citep{paszke2019pytorch}. The code for both the algorithms will be made public along with the competitive baselines. We limit the maximum number of rounds to $300$, and with the selected hyperparameters and the given environments, it terminates before that. This roughly takes \emph{10 min} of training for \safemac on a single core CPU. The code is written for running a single instance of the experiment. In practice, we launch nearly 1000 such instances simultaneously on the cluster in parallel to get results about different environments, noise realizations and initializations.

% de base, license, library use, ran on which cluster,r=5
\mypar{Gorilla Environment} The gorilla environment  (\cref{fig: gorilla-env}) is defined in a grid of $34 \times 34$, with each grid cell being a square of length $0.1$. The $\numOfAgents=3$ agents perform the coverage task, with each having a sensing region defined as a set of locations agents that can travel in 5 steps in the underlying transition graph (Precisely, $\Discat[i] = \Roperator{reach}{5}{\{\LocAgent[i]{}\}}, \cref{eqn: n-reach}$). We considered 10 gorilla environments each differ in the initial location of the agents. The nest density is obtained by fitting a smooth rate function \citep{mojmir-cox} over Gorilla nest site locations which were provided by the Wildlife Conservation Society Takamanda-Mone Landscape project (WCS-TMLP) \citet{gorilla-kagwene}. As a proxy for bad weather, we use the cloud coverage data over the Kagwene Gorilla Sanctuary from OpenWeather \citep{OpenWeather}. The density and the constraint function used are available in our code base. The code for fitting a rate function is available here (https://github.com/Mojusko/sensepy) under the MIT license. 
We used a lengthscale of 1 for the density and of 2 for the constraint function. The noise variance is set to $10^{-3}$ and $7 \times 10^{-3}$ for density and the constraint respectively. However, the performance in the experiments is not sensitive to the hyperparameters and is easily reproducible with other sensible parameters as well.

\mypar{Obstacles Environment} The obstacle environment (\cref{fig: obstacle-env}) is defined on a grid of $30 \times 30$, with each grid cell being a square of length $0.1$. The sensing region and number of agents are defined similar to the Gorilla environment. 
The obstacle is completely defined by the location of its top right corner and the bottom left corner. The obstacle environment is generated by combining a set of such obstacles. The density is directly sampled from a GP with the parameters same as synthetic data. We produced ten instances of environments, each having a different set of obstacles and GP sample and initialization. We used a lengthscale of 2 for both density and the constraint function. The noise variance is set to $10^{-3}$. Similar to earlier environments, performance is not sensitive to hyperparameters.

\mypar{Experiment results} 

\emph{Unconstrained case} \cref{fig: gorilla_macopt2D-regret} and \cref{fig: gp_macopt2D-regret} plots the simple regret $\simReg{}$ for each round t, precisely, defined as $\sum_{i=1}^{\numOfAgents} \delgain{\tilde{\LocAgent[]{}}|\LocAgents^{1:i-1}}{\density}{\Domain} - \delgain{\LocAgent[i]{t}|\LocAgents^{1:i-1}}{\density}{\Domain}$. This quantity upper bounds the actual regret and provides intuition for the convergence rate. We see in the plots that the simple regret goes to zero for \macopt, but gets stuck for the \ucb algorithm. Due to this, we also observe that \macopt can achieve higher coverage value as compared to \ucb in \cref{fig: gp_macopt2D}. 

\emph{Constrained case} \cref{fig: maps} and \cref{fig: safemac-gp} compares coverage of area attained by \safemac, \passivemac and the two stage algorithm. Precisely the intermediate locations are recommended as per \cref{eqn: pre-mature-recommendation}. We see that \safemac finds a comparable solution to two stage more efficiently without exploring the whole environment, where as \passivemac gets stuck in the local optimum. 
% Plot for sigma below the discs, optimistic or upper bound of the regret
% \emph{Regret plots}
% sigma that shows the shrinkage

\begin{figure*}[t]
% 	\hspace{-3.00mm}
% 	\centering
%     \begin{subfigure}[t]{0.52\columnwidth}
%   	\centering
%   	\includegraphics[scale=0.9]{images/macopt2D-regret.pdf}
% %   	\setlength{\abovecaptionskip}{0pt}
% %   	\setlength{\belowcaptionskip}{0pt}
%     \caption{}
%     \label{fig: macopt2D-regret}
%     \end{subfigure}
%     \hspace{-6.00mm}
\hspace{-13.00mm}
	\centering
    \begin{subfigure}[t]{0.3\columnwidth}
  	\centering
  	\includegraphics[scale=0.7]{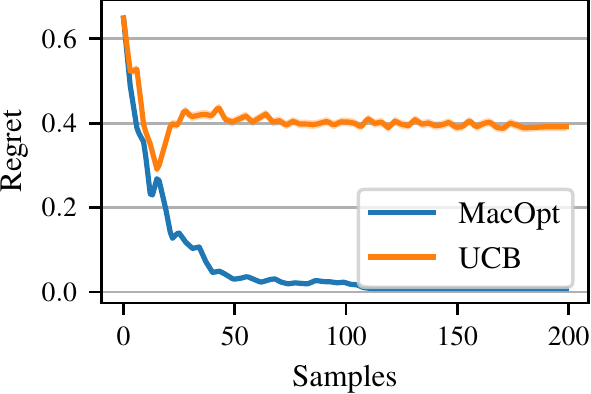}
    \caption{Gorilla environment}
    \label{fig: gorilla_macopt2D-regret}
    \end{subfigure}
    % \hspace{-6.00mm}
    ~
% 	\hspace{-4.00mm}
	\centering
    \begin{subfigure}[t]{0.3\columnwidth}
  	\centering
  	\includegraphics[scale=0.7]{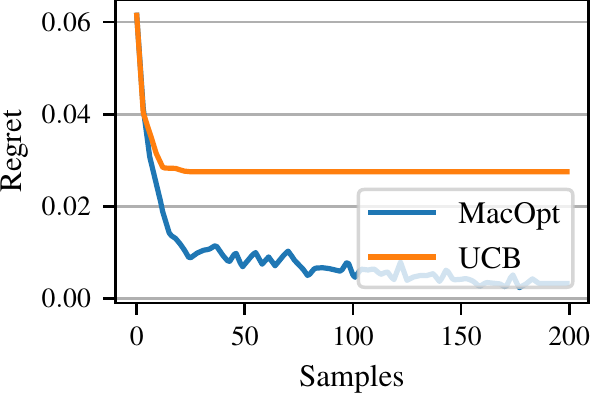}
    \caption{\GP environment}
    \label{fig: gp_macopt2D-regret}
    \end{subfigure}
    % \hspace{-6.00mm}
~
    \begin{subfigure}[t]{0.3\columnwidth}
  	\centering
  	\includegraphics[scale=0.7]{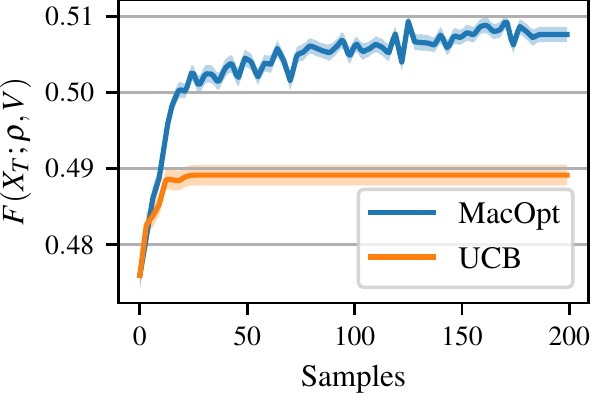}
    \caption{\GP environment}
    \label{fig: gp_macopt2D}
    \end{subfigure}
    \hspace{-6.00mm}
\caption{
Compares \macopt and \ucb on the Gorilla (a) and the \GP (b,c) environment. a,b) Compares simple regret $\simReg{}$ \cref{def: sim-act-reg-loc-reg} in the unconstrained case (domain $\Domain$). c) Plots total coverage achieved by both the algorithms.  
% The contours show the density sampled from the \GP in \cref{fig: obstacle-env} and the distribution of Gorilla nests in \cref{fig: gorilla-env}. In \cref{fig: obstacle-env} obstacles are marked by black blocks. In \cref{fig: gorilla-env}, the black dashed line shows constraints. Light grey lines show the contour of constraint distribution (plotted only near to threshold). \cref{fig: gorilla_macopt2D} compares the performance of \macopt with \ucb
}
% \vspace{-4 mm}
\end{figure*}
% \begin{figure}
%     \centering
%     \includegraphics[scale=1]{images/gp_macopt2D-regret.pdf}
%     \caption{Regret}
%     \label{fig:my_label}
% \end{figure}
\begin{figure*}[t!]
% 	\hspace{-3.00mm}
% 	\centering
    \begin{subfigure}[t]{0.5\columnwidth}
  	\centering
  	\includegraphics[scale=0.85]{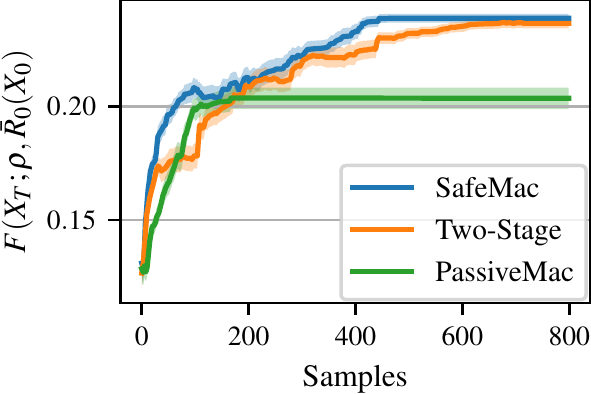}
    \caption{Obstacle environment}
    \label{fig: maps}
    \end{subfigure}
    % \hspace{-6.00mm}
    ~
    \begin{subfigure}[t]{0.5\columnwidth}
  	\centering
  	\includegraphics[scale=0.85]{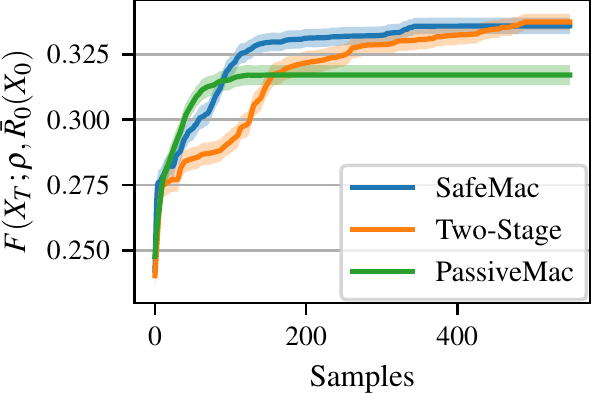}
    \caption{GP environement}
    \label{fig: safemac-gp}
    \end{subfigure}
    % \hspace{-6.00mm}
    \caption{Comparison of \safemac with \passivemac and Two-Stage in Obstacle and the \GP environment during optimization}
    \vspace{-4mm}
\end{figure*}
\vspace{-2mm}
\subsection{Scaling in terms of agents and domain size}
\label{apx: scaling}

In this section, we evaluate scalability in terms of number of agents and the domain size. \safemac evaluates a greedy solution $\numOfAgents$ times (one for each agent) at each iteration, it is linear in the number of agents. Moreover, the greedy solution is linear in the number of cells (domain size). To demonstrate this we run the experiment on the Gorilla nest density for $\numOfAgents =$ 3, 6, 10 and 15 agents each for the domains of size 30$\times$30, 40$\times$40, 50$\times$50 and 60$\times$ 60. 
We see that with more agents in larger domain the same results hold that is \safemac finds a comparable solution to two stage more efficiently without exploring the whole environment, where as \passivemac gets stuck in the local optimum. 
% We observe that same results hold that is \safemac achieves better solution as compared to \passivemac while being more sample efficient as compared to Two-stage algorithm.  

\begin{figure*}[ht]
% 	\hspace{-3.00mm}
% 	\centering
    \begin{subfigure}[t]{0.23\columnwidth}
  	\centering
  	\includegraphics[scale=0.55]{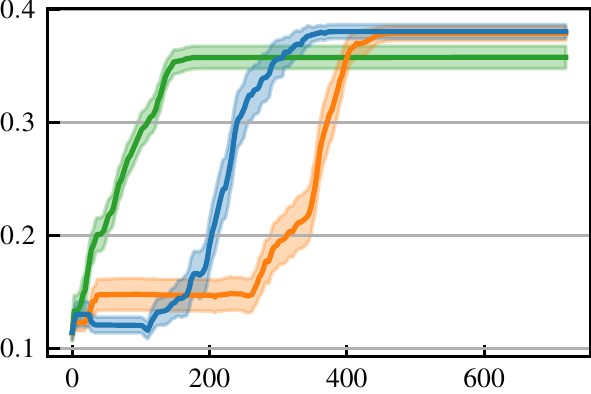}
    \caption{3 agents in 30 $\times$ 30}
    \label{fig: 3-30X30}
    \end{subfigure}
    % \hspace{-6.00mm}
    ~
    \begin{subfigure}[t]{0.23\columnwidth}
  	\centering
  	\includegraphics[scale=0.55]{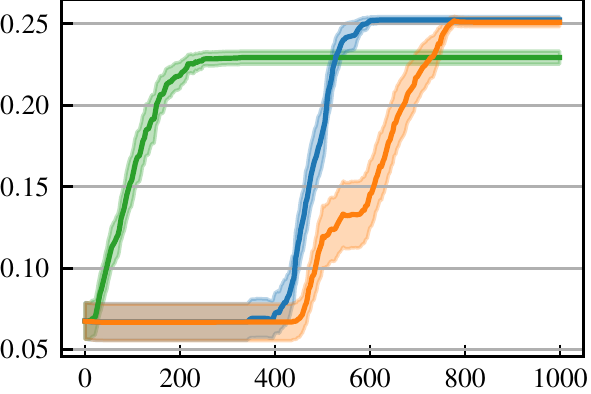}
    \caption{3 agents in 40 $\times$ 40 }
    \label{fig: 3-40X40}
    \end{subfigure}
    % \hspace{-6.00mm}
    ~
    \begin{subfigure}[t]{0.23\columnwidth}
  	\centering
  	\includegraphics[scale=0.55]{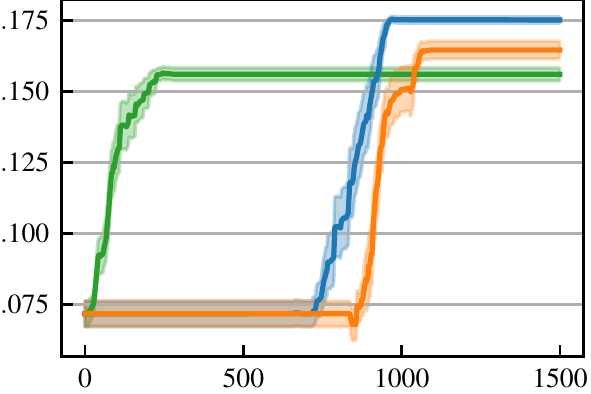}
    \caption{3 agents in 50 $\times$ 50 }
    \label{fig: 3-50X50}
    \end{subfigure}
        ~
    \begin{subfigure}[t]{0.23\columnwidth}
  	\centering
  	\includegraphics[scale=0.55]{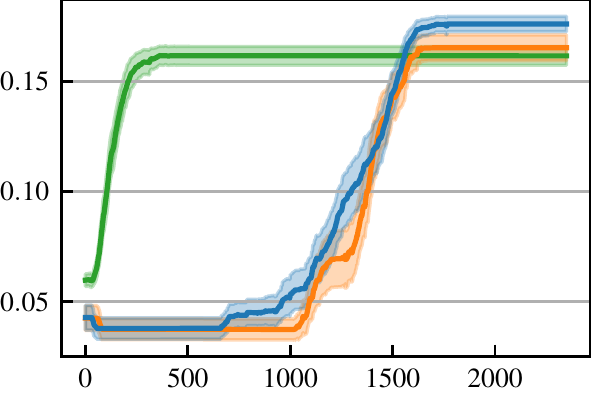}
    \caption{3 agents in 60 $\times$ 60 }
    \label{fig: 3-60X60}
    \end{subfigure}
    
    %%%%%%%%%%%%%%%%%%%%%%%%%New line
    
       \begin{subfigure}[t]{0.23\columnwidth}
  	\centering
  	\includegraphics[scale=0.55]{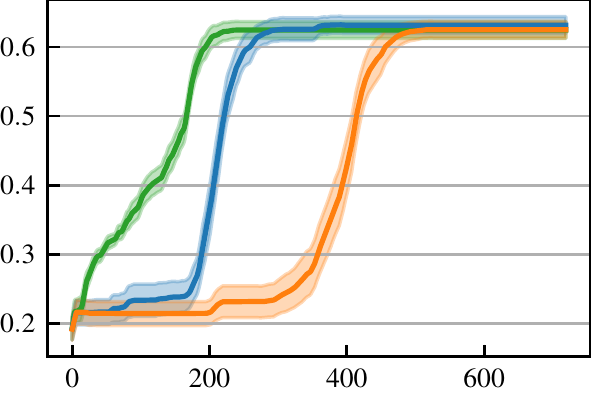}
    \caption{6 agents in 30 $\times$ 30}
    \label{fig: 6-30X30}
    \end{subfigure}
    % \hspace{-6.00mm}
    ~
    \begin{subfigure}[t]{0.23\columnwidth}
  	\centering
  	\includegraphics[scale=0.55]{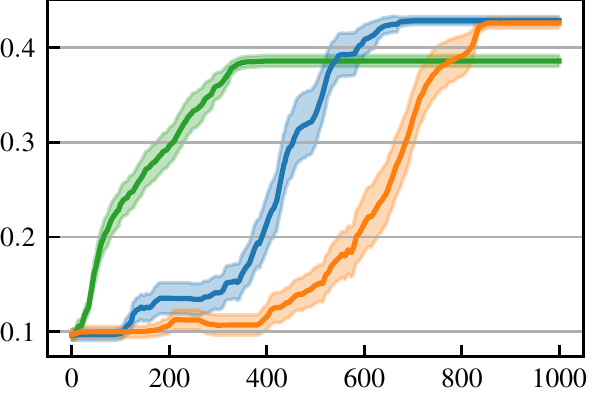}
    \caption{6 agents in 40 $\times$ 40}
    \label{fig: 6-40X40}
    \end{subfigure}
    % \hspace{-6.00mm}
    ~
    \begin{subfigure}[t]{0.23\columnwidth}
  	\centering
  	\includegraphics[scale=0.55]{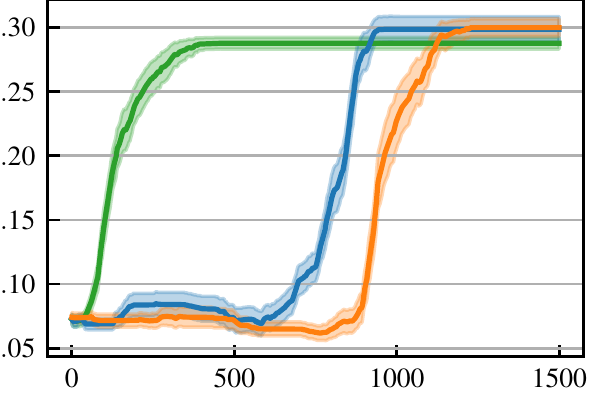}
    \caption{6 agents in 50 $\times$ 50}
    \label{fig: 6-50 x 50}
    \end{subfigure}
        ~
    \begin{subfigure}[t]{0.23\columnwidth}
  	\centering
  	\includegraphics[scale=0.55]{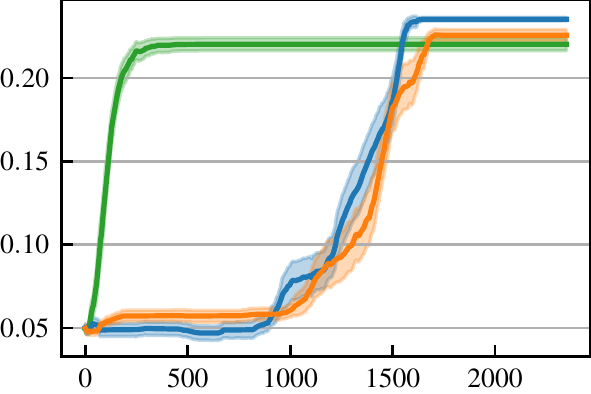}
    \caption{6 agents in 60 $\times$ 60}
    \label{fig: 6-40 x 40}
    \end{subfigure}
    
        %%%%%%%%%%%%%%%%%%%%%%%%%New line
    
       \begin{subfigure}[t]{0.23\columnwidth}
  	\centering
  	\includegraphics[scale=0.55]{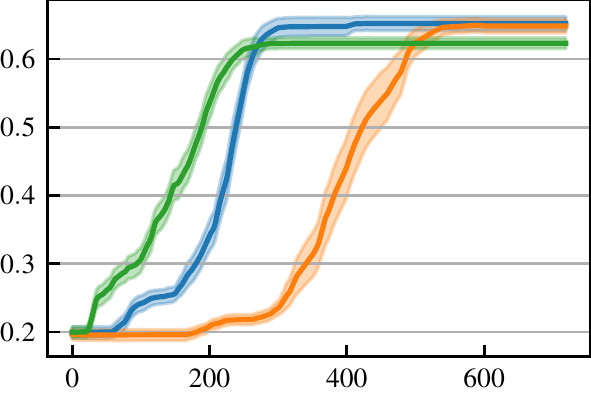}
    \caption{10 agents in 30 $\times$ 30}
    \label{fig: 10-30X30}
    \end{subfigure}
    % \hspace{-6.00mm}
    ~
    \begin{subfigure}[t]{0.23\columnwidth}
  	\centering
  	\includegraphics[scale=0.55]{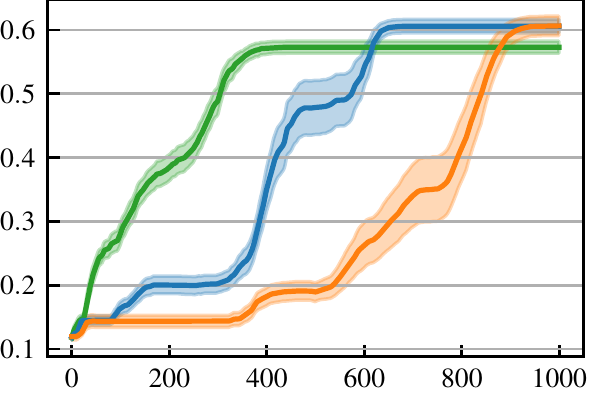}
    \caption{10 agents in 40 $\times$ 40}
    \label{fig: 10-40X40}
    \end{subfigure}
    % \hspace{-6.00mm}
    ~
    \begin{subfigure}[t]{0.23\columnwidth}
  	\centering
  	\includegraphics[scale=0.55]{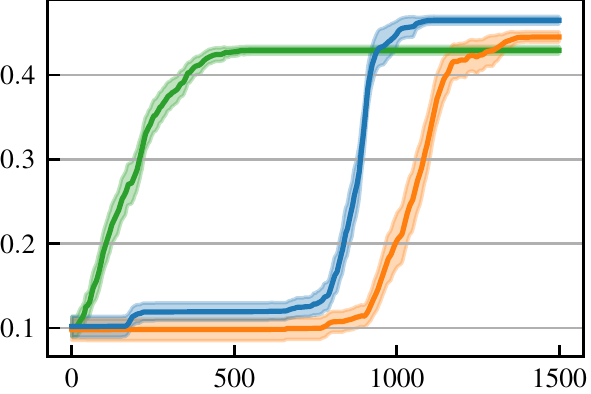}
    \caption{10 agents in 50 $\times$ 50}
    \label{fig: 10-50X50}
    \end{subfigure}
        ~
    \begin{subfigure}[t]{0.23\columnwidth}
  	\centering
  	\includegraphics[scale=0.55]{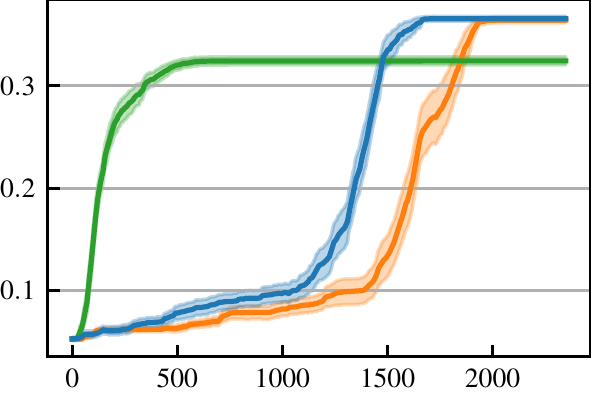}
    \caption{10 agents in 60 $\times$ 60}
    \label{fig: 10-60X60}
    \end{subfigure}
    
        %%%%%%%%%%%%%%%%%%%%%%%%%New line
    
       \begin{subfigure}[t]{0.23\columnwidth}
  	\centering
  	\includegraphics[scale=0.55]{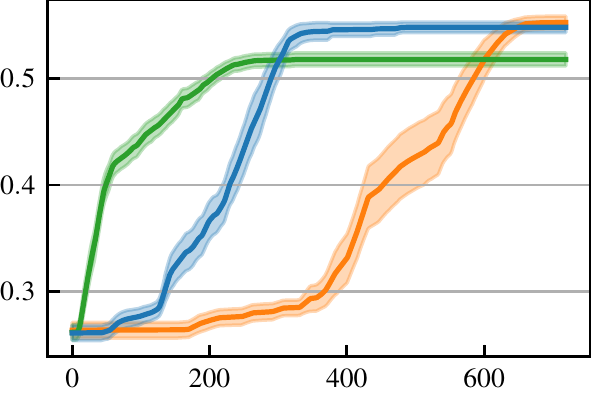}
    \caption{15 agents in 30 $\times$ 30}
    \label{fig: 15-30X30}
    \end{subfigure}
    % \hspace{-6.00mm}
    ~
    \begin{subfigure}[t]{0.23\columnwidth}
  	\centering
  	\includegraphics[scale=0.55]{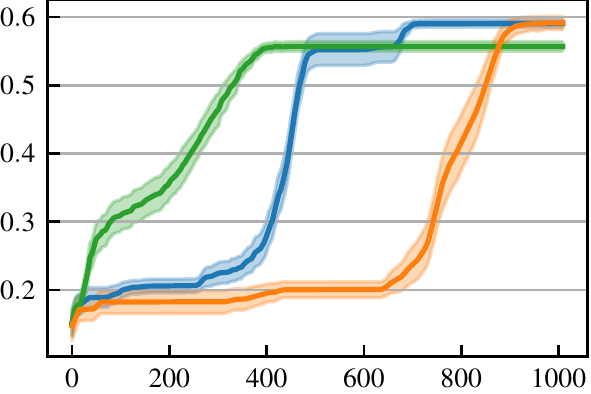}
    \caption{15 agents in 40 $\times$ 40}
    \label{fig: 15-40X40}
    \end{subfigure}
    % \hspace{-6.00mm}
    ~
    \begin{subfigure}[t]{0.23\columnwidth}
  	\centering
  	\includegraphics[scale=0.55]{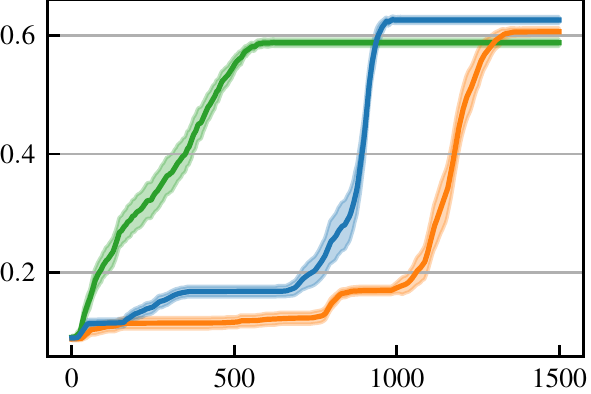}
    \caption{15 agents in 50 $\times$ 50}
    \label{fig: 15-50X50}
    \end{subfigure}
        ~
    \begin{subfigure}[t]{0.23\columnwidth}
  	\centering
  	\includegraphics[scale=0.55]{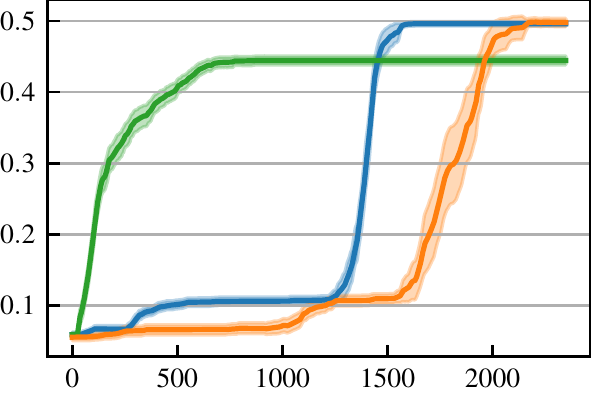}
    \caption{15 agents in 60 $\times$ 60}
    \label{fig: 15-60X60}
    \end{subfigure}
    \caption{Comparison of \safemac with \passivemac and Two-Stage in the Gorilla nest environment. Legend: \textcolor{blue}{Blue is \safemac}, \textcolor{green}{Green is \passivemac} and \textcolor{orange}{Orange is Two-Stage} algorithm. The experiment is performed for 3, 6, 10 and 15 agents (increased row-wise) each for the domains of size 30$\times$30, 40$\times$40, 50$\times$50 and 60$\times$ 60 (increased column-wise). }
\end{figure*}

\end{document}